\documentclass{article}

\usepackage{microtype}
\usepackage{graphicx}
\usepackage{subfigure}
\usepackage{booktabs} 

\usepackage{hyperref}

\usepackage[accepted]{icml2021}

\usepackage{nicefrac}
\usepackage{bbm}
\usepackage{enumitem}
\usepackage{wrapfig}
\usepackage{amsthm}

\usepackage{amsthm}
\usepackage{graphicx}
\usepackage{wrapfig}

\usepackage{amssymb}

\usepackage{color}
\usepackage{mathtools}

\usepackage{multirow}

\usepackage{algorithm}
\usepackage{algorithmic}

\usepackage[most]{tcolorbox}

\usepackage[scaled=0.84]{helvet}

\newcommand{\denselist}{\itemsep -1pt\partopsep 0pt}

\newcommand{\E}{\mathop{\mathbb{E}}} 
\newcommand{\Var}{\mathop{\mathbb{V}}} 

\newcommand{\gzrep}{g^\mathrm{rep}}
\newcommand{\gzdrep}{g^\mathrm{STL}}
\newcommand{\grep}{g^\mathrm{rep}_\alpha}
\newcommand{\gdrep}{g^\mathrm{drep}_\alpha}

\newcommand{\ga}{g_\alpha}
\newcommand{\gaj}{g_{\alpha j}}

\newcommand{\R}{\lambda}

\newcommand{\SNR}{\mathrm{SNR}}
\newcommand{\KL}{\mathrm{KL}}
\newcommand{\sw}{\sigma_{w}}
\newcommand{\sv}{\sigma_{v}}

\newcommand{\Cw}{\Sigma_{w}}
\newcommand{\Cv}{\Sigma_{v}}
\newcommand{\Sw}{S_{w}}

\newcommand{\T}{\mathcal{T}_w(\epsilon)}
\newcommand{\Ti}{\mathcal{T}_{w_i}(\epsilon_i)}
\newcommand{\Tj}{\mathcal{T}_{w_j}(\epsilon_j)}

\newcommand{\f}{f(\R, \alpha)}
\newcommand{\g}{g(\R, \alpha)}

\usepackage{thmtools} 
\usepackage{thm-restate}

\declaretheorem[name=Theorem]{thm}

\newtheorem{lemma}[thm]{Lemma}
\newtheorem{proposition}[thm]{Proposition}

\icmltitlerunning{On the Difficulty of Unbiased Alpha Divergence Minimization}

\begin{document}

\twocolumn[
\icmltitle{On the Difficulty of Unbiased Alpha Divergence Minimization}



\icmlsetsymbol{equal}{*}

\begin{icmlauthorlist}
\icmlauthor{Tomas Geffner}{umass}
\icmlauthor{Justin Domke}{umass}
\end{icmlauthorlist}

\icmlaffiliation{umass}{College of Information and Computer Science, University of Massachusetts, Amherst, MA, USA}

\icmlcorrespondingauthor{Tomas Geffner}{tgeffner@cs.umass.edu}
\icmlcorrespondingauthor{Justin Domke}{domke@cs.umass.edu}

\icmlkeywords{Machine Learning, ICML}

\vskip 0.3in
]



\printAffiliationsAndNotice{}  

\begin{abstract}

Several approximate inference algorithms have been proposed to minimize an alpha-divergence between an approximating distribution and a target distribution. Many of these algorithms introduce bias, the magnitude of which becomes problematic in high dimensions. Other algorithms are unbiased. These often seem to suffer from high variance, but little is rigorously known. In this work we study unbiased methods for alpha-divergence minimization through the Signal-to-Noise Ratio (SNR) of the gradient estimator. We study several representative scenarios where strong analytical results are possible, such as fully-factorized or Gaussian distributions. We find that when alpha is not zero, the SNR worsens exponentially in the dimensionality of the problem. This casts doubt on the practicality of these methods. We empirically confirm these theoretical results.

\end{abstract}



\section{Introduction}

Variational inference (VI) typically minimizes the KL-divergence from an approximating distribution $q_w$ to a target distribution $p$ \citep{jordan1999introduction, blei2017variational, zhang2017advances}. While computationally convenient, the use of this objective may lead to distributions $q_w$ with undesirable statistical properties (e.g. variance underestimation \citep{minka2005divergence}). To avoid this, recent methods instead attempt to minimize an alpha-divergence \citep{amari1985differential}. This is a class of divergences indexed by a parameter $\alpha$, that reduces to the typical KL-divergence when $\alpha \rightarrow 0$. The $\alpha$ parameter determines the divergence's properties. For instance, for $\alpha \gg 0$, the divergence penalizes distributions $q_w$ that place no mass in regions where $p$ does, penalizing variance underestimation. For many use-cases, approximating distributions minimizing these divergences would be more useful.


Existing alpha-divergence minimization algorithms can be classified into two broad groups: biased methods \citep{renyiVI, regli2018alpha} and unbiased methods \citep{kuleshov2017neural, dieng2017chivi}. While some positive empirical results have been obtained using biased methods, it has been recently observed that, in problems with high dimensionality, these often fail to minimize the target alpha-divergence \cite{biasedalphafail}.\footnote{In fact, it was observed that, in high dimensions, biased methods often just minimize the typical VI objective, the KL divergence from $q_w$ to $p$, regardless of the target alpha-divergence chosen.} Thus, in this paper we turn our attention to unbiased methods. These attempt to minimize an alpha-divergence by running stochastic optimization algorithms with unbiased estimators of the divergence's gradient.

The major open question for unbiased methods concerns the difficulty of the optimization problem. Too much variability in the gradient estimator would force a very small step-size and thus a huge number of optimization steps. While some positive empirical results have been observed, authors have cautioned the the gradient estimators used can be temperamental \citep{kuleshov2017neural, dieng2017chivi}. It is currently unclear when these methods will succeed, and how their performance depends on $\alpha$ or the dimensionality of the problem.


We address this question. Informally, our main conclusion is that methods based on unbiased gradient estimates of an alpha-divergence will often require catastrophically large amounts of computation to scale to high dimensional models for any $\alpha \neq 0$. This is not always due to high variance gradients, but to an extremely low Signal-to-Noise ratio (SNR).

In Section \ref{sec:motiv} we present two gradient estimators used by unbiased methods, one obtained using reparameterization, and a novel one obtained by applying ``double'' reparameterization. An empirical evaluation shows that, even in simple scenarios, these methods do not seem scale to problems of moderately high dimension ($d \approx 100$), nor to moderately large values of $\alpha$ ($\alpha \approx 0.4$). Curiously, optimization fails even in cases where the gradient estimator has provably low variance. Instead, we propose that this failure is best explained by the estimator's (SNR), which is known to be related to optimization convergence (Section~\ref{sec:SNRSGD}).

The main contribution of this paper is a theoretical analysis of the gradient estimators' SNR, given in Section \ref{sec:NSR}. We analyze two representative scenarios: The first is when the target and approximating family are arbitrary fully-factorized distributions. The second is when the target and approximating family are both full-rank Gaussians. We give exact results and bounds for the SNR in these cases. We show that for \textit{any} $\alpha \neq 0$, under mild assumptions, the SNR decreases \textit{exponentially} in the dimensionality of the problem. Thus, even in these seemingly ``easy'' scenarios, unbiased methods will not scale to high dimensional problems. Finally, in Section \ref{sec:exps} we empirically confirm that the same phenomena seems to occur in real problems.

Our results are pessimistic. Ideally, one might hope to guarantee a good SNR under some favorable assumptions about the target. For example, one might hope that the SNR could be guaranteed to be reasonably large if the log-posterior obeyed some common regularity conditions, e.g. that it were fully-factorized, concave, strongly concave, Lipschitz smooth, Gaussian, or even a fully-factorized Gaussian. Our results show that, for general alpha-divergences, no such guarantee is possible.

We do not rule out the possibility of a good SNR guarantee under some other assumptions about the target. However, these assumptions would have to be \textit{stronger} than any of those listed above, and also \textit{prohibit} the cases we typically think of as easy, e.g. fully-factorized or Gaussian distributions. This suggests that a general-purpose algorithm for optimizing an alpha-divergence based on currently available unbiased gradient estimators may be unachievable.


\section{Preliminaries}

\textbf{Variational Inference} (VI) is an approximate inference algorithm that finds $w$ such that the approximating distribution $q_w(z)$ is close to some target distribution $p(z)$. This is usually done by minimizing $\KL(q_w||p)$. In most cases the gradient of this objective with respect to $w$ cannot be computed exactly. However, unbiased estimates are often available. Two popular alternatives are reparameterization \citep{doublystochastic_titsias, vaes_welling, rezende2014stochastic} and the ``sticking the landing'' (STL) estimator \citep{stickingthelanding}. Both require a mapping $\mathcal{T}_w$ that transforms a base density $q_0$ into $q_w$. Then, the estimators are computed as
\begin{equation} \begin{array}{rcl}
\gzrep(p, q_w, \epsilon) & = & \nabla_w \log \frac{q_w(\T)}{p(\T)}\\ \\
\gzdrep(p, q_w, \epsilon) & = & \left. \nabla_w \log \frac{q_v(\T)}{p(\T)} \right|_{v=w}, \label{eq:estz}
\end{array} \end{equation}
where $\epsilon \sim q_0(\epsilon)$. While both estimators have shown good empirical performance, it has been observed that $\gzdrep$ often leads to better results \citep{stickingthelanding}. Also, it has the desirable property of being deterministically zero at the optimum $p = q_w$, which is not true for the reparameterization estimator.

While minimizing $\KL(q_w||p)$ is computationally convenient, it may lead to distributions $q_w$ with undesirable properties. For instance, the resulting $q_w$ tends to underestimate the variance of $p$ \citep{minka2005divergence}. This is problematic, for instance, if $q_w$ will be used as a proposal distribution to estimate the expectation $\E_p f(z)$ with importance sampling. An under-dispersed distribution $q_w$ leads to importance weights with high variance \citep{mcbook}.
For reasons like this, recent work has developed methods to minimize other divergence measures \citep{powerEP, hernandez2016black, renyiVI, kuleshov2017neural, dieng2017chivi, regli2018alpha, wang2018variational, fdivVI, MarkovianSC}.


\textbf{Alpha-divergences} may be used as an objective for VI. The alpha-divergence between distributions $p$ and $q_w$ is given by
\small
\begin{equation}
D_\alpha(p||q_w) = \frac{1}{\alpha(\alpha - 1)} \E_{q_w}\left[ \left(\frac{p(z)}{q_w(z)}\right)^\alpha - 1 \right],\quad \alpha \in \mathbb{R} \setminus \{0, 1\}. \label{eq:obj_a}
\end{equation}
\normalsize
For $\alpha \gg 0$ the divergence will penalize distributions $q_w$ that place no mass in regions where $p$ does, penalizing under-dispersion. For instance, minimizing $D_\alpha(p||q_w)$ for $\alpha = 2$ is equivalent to minimizing the variance of the importance weights $p(z)/q_w(z)$. Also, alpha-divergences recover several well known divergences for different values of $\alpha$ \citep{amari2010families}: the $\chi^2$-divergence for $\alpha = 2$, and the Hellinger distance for $\alpha = 0.5$. For $\alpha = 0$ it is defined as the limit $\alpha \rightarrow 0$, which result in $\KL(q_w||p)$, the divergence typically used for VI. Algorithms to minimize alpha-divergences can be classified into two groups:

\textbf{Biased methods} \citep{powerEP, RWS, hernandez2016black, renyiVI, regli2018alpha}. These either minimize local surrogates and/or use biased gradient estimators. Therefore, the distributions $q_w$ returned by these methods are not minimizers of $D_\alpha(p||q_w)$. \citet{biasedalphafail} present an extensive empirical evaluation of methods based on biased gradients showing that, in high dimensions, these methods often fail to minimize the target alpha-divergence; they return suboptimal distributions $q_w$ that heavily underestimate the variance of $p$. This is problematic, since one of the goals of using alpha-divergences is to avoid under-dispersion.

\textbf{Unbiased methods} \citep{kuleshov2017neural, dieng2017chivi}. These methods were developed for the $\chi^2$-divergence. However, as mentioned by the authors, the same approach can be used for $\alpha \neq 2$. These methods attempt to minimize $D_\alpha(p||q_w)$ exactly by running SGD with an unbiased estimator of $\nabla_w D_\alpha(q_w||p)$. Under an appropriate choice for the step-size this is guaranteed to converge \citep{bottou_SGDreview}. As mentioned previously, these methods often present scalability issues, and it is unclear when they may be successfully applied.\footnote{While in their simulations \cite{dieng2017chivi} use a biased algorithm, their whole formulation was carried out for unbiased divergence minimization.}

The \textbf{signal-to-noise ratio} (SNR) can be used as a measure of an estimator's quality. We define the SNR of a random vector $X$ as
\small
\begin{equation}
\SNR[X] = \frac{\left \Vert \mathbb{E} X  \right\Vert^2}{\mathbb{E} \Vert X \Vert^2} = \frac{\left \Vert \mathbb{E} X \right\Vert^2}{\left \Vert \mathbb{E} X \right\Vert^2 + \mathbb{V} \Vert X \Vert}, \label{eq:NSR}
\end{equation}
\normalsize
which is always less than or equal to one, with equality holding if and only if the variance of the random variable is zero. If some of the expectations do not exist, the SNR is not defined. We will use this quantity to evaluate the quality of an estimator. This has been previously done in the context of importance weighted auto-encoders \citep{rainforth2018tighter}.


\section{Gradient Estimators}

\label{sec:motiv}

\begin{figure*}[t]
  \centering
  \includegraphics[scale=0.24, trim = {0 2.3cm 0 0}, clip]{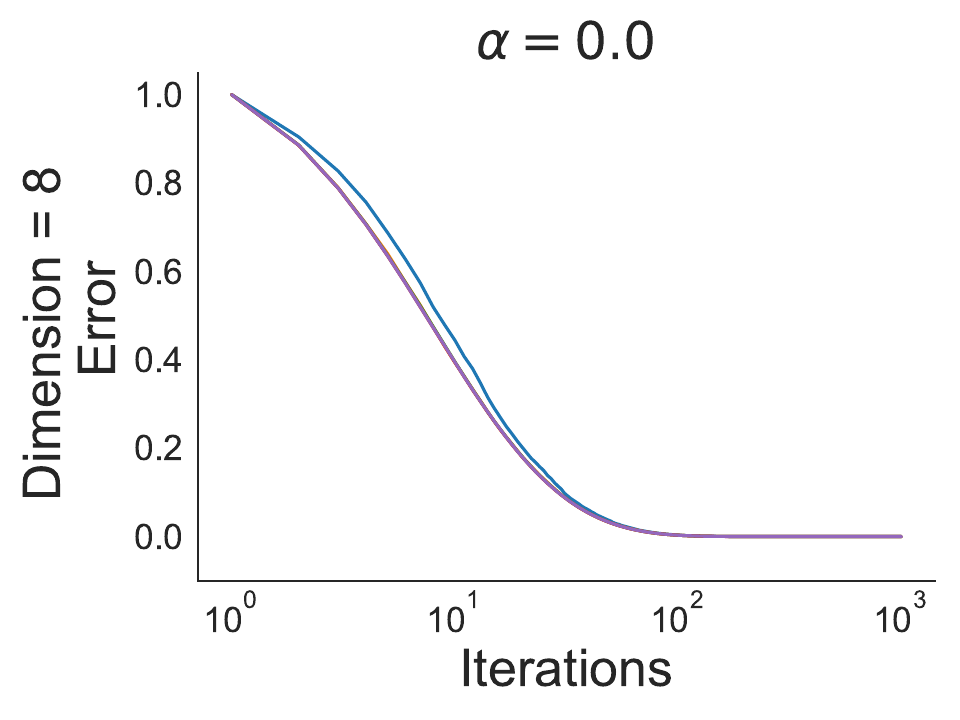}\hfill
  \includegraphics[scale=0.24, trim = {2.31cm 2.3cm 0 0}, clip]{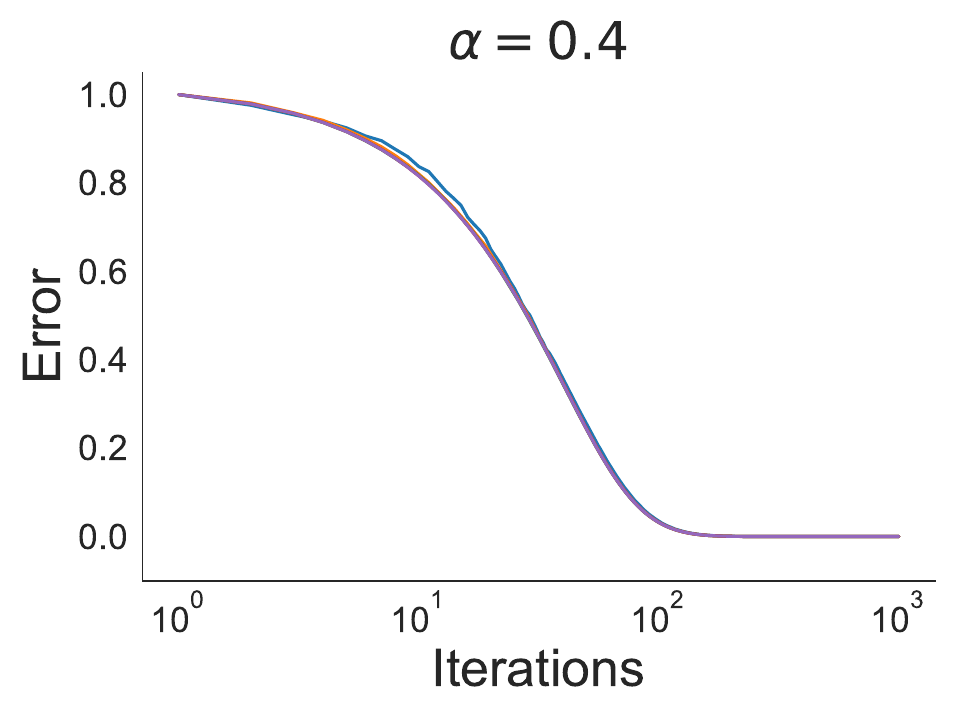}\hfill
  \includegraphics[scale=0.24, trim = {2.31cm 2.3cm 0 0}, clip]{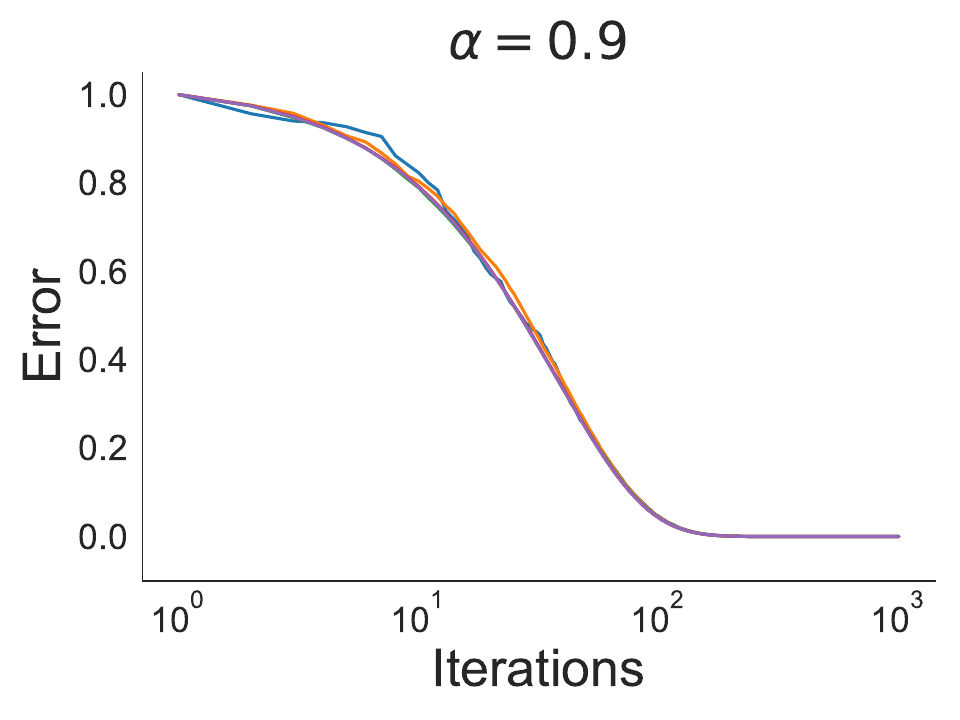}\hfill
  \includegraphics[scale=0.24, trim = {2.31cm 2.3cm 0 0}, clip]{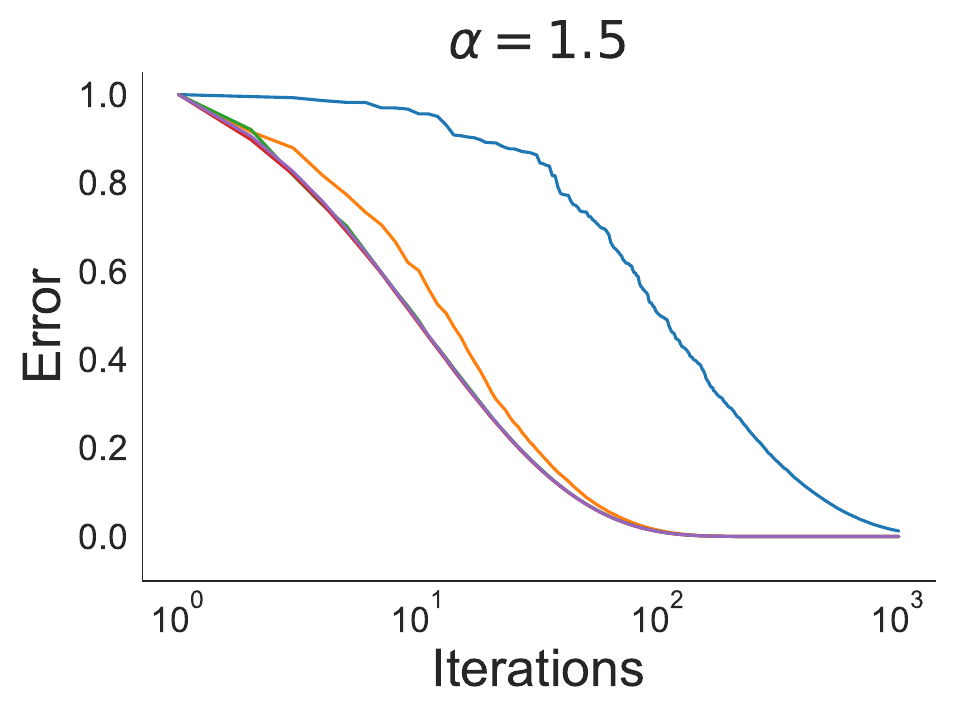}

  \vspace{0.2cm}

  \includegraphics[scale=0.24, trim = {0 2.3cm 0 1.05cm}, clip]{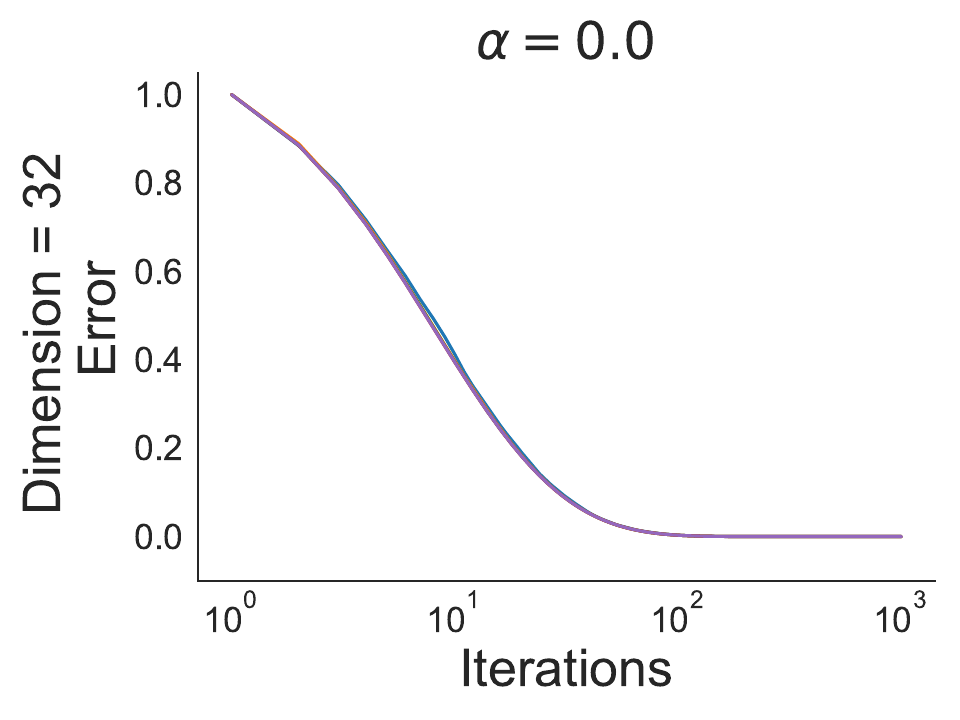}\hfill
  \includegraphics[scale=0.24, trim = {2.31cm 2.3cm 0 1.05cm}, clip]{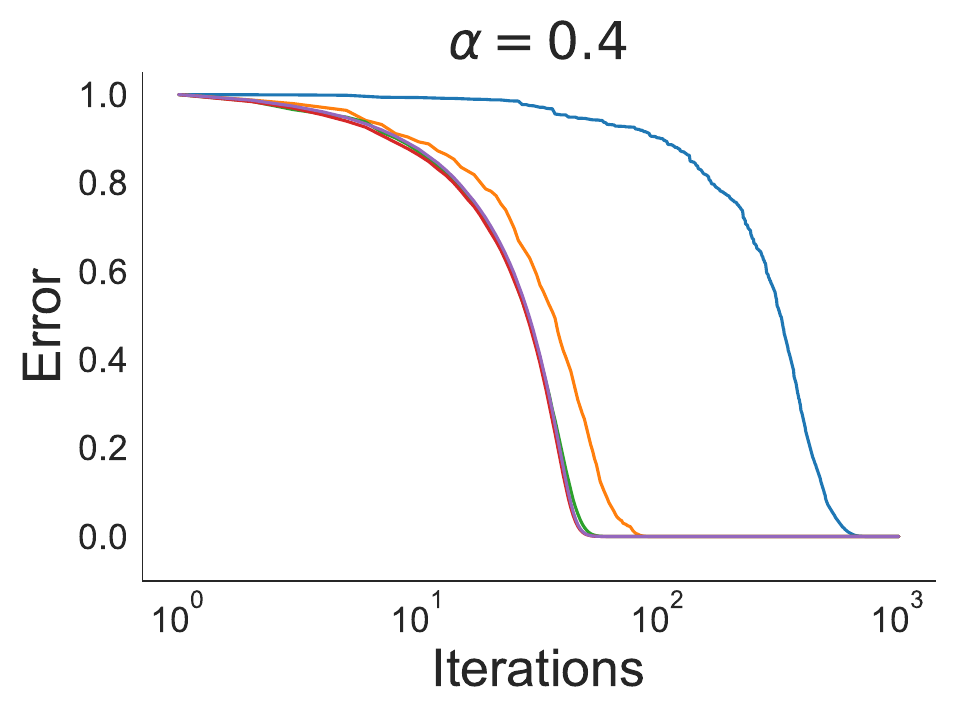}\hfill
  \includegraphics[scale=0.24, trim = {2.31cm 2.3cm 0 1.05cm}, clip]{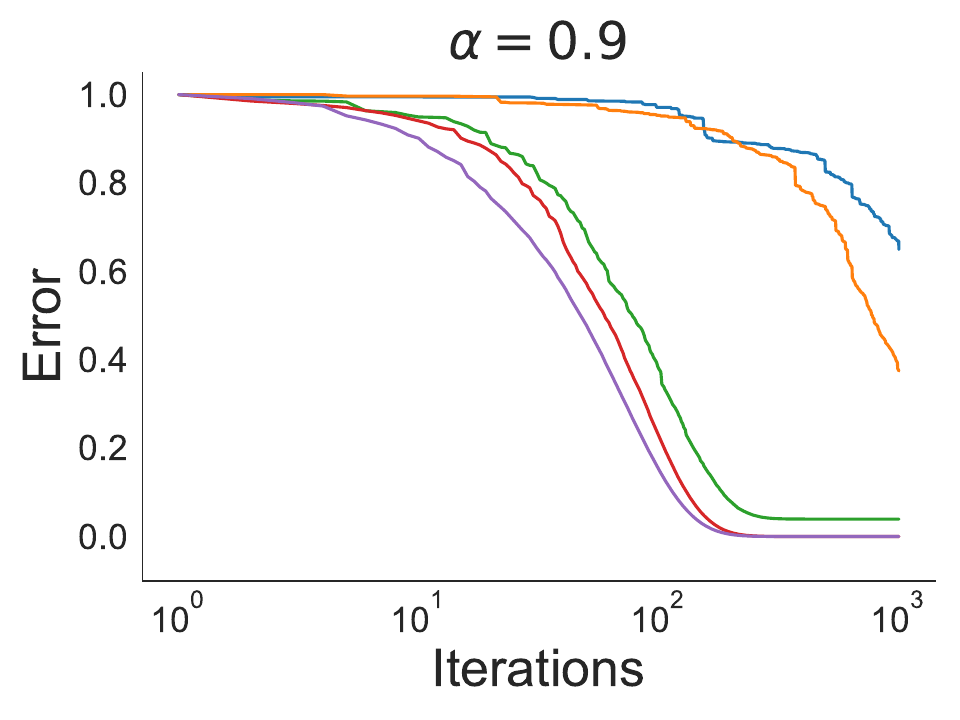}\hfill
  \includegraphics[scale=0.24, trim = {2.31cm 2.3cm 0 1.05cm}, clip]{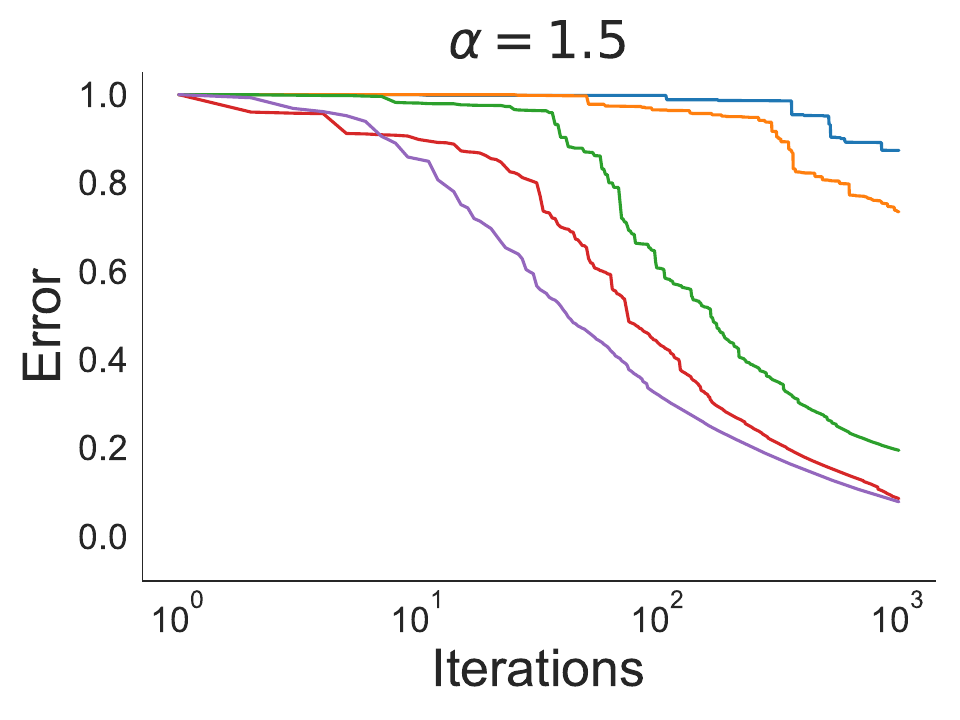}

  \vspace{0.2cm}

  \includegraphics[scale=0.24, trim = {0 0 0 1.1cm}, clip]{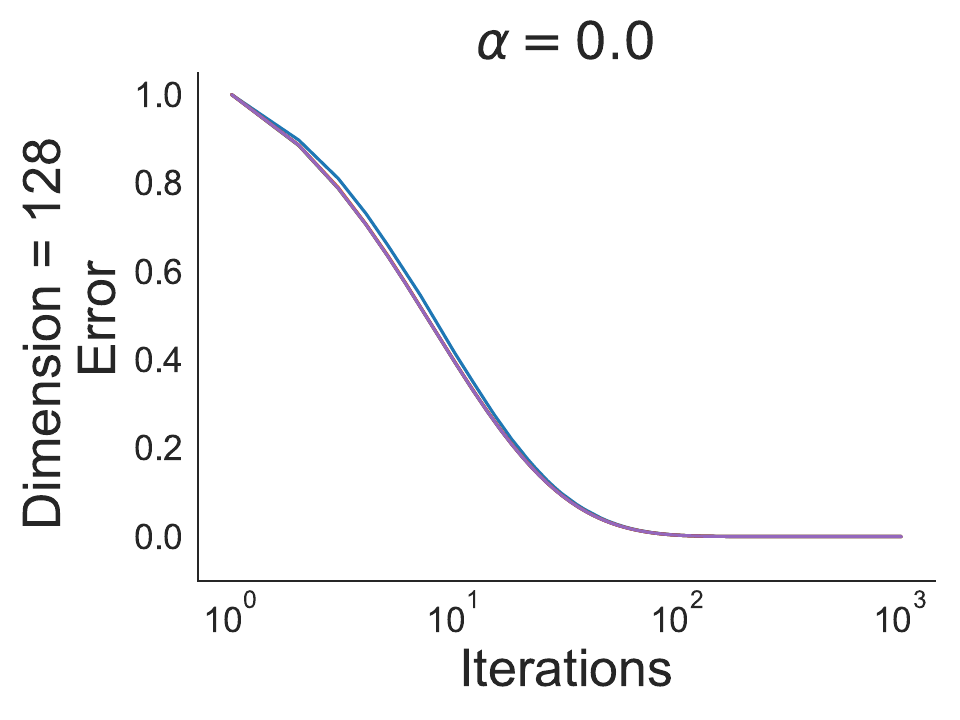}\hfill
  \includegraphics[scale=0.24, trim = {2.31cm 0 0 1.1cm}, clip]{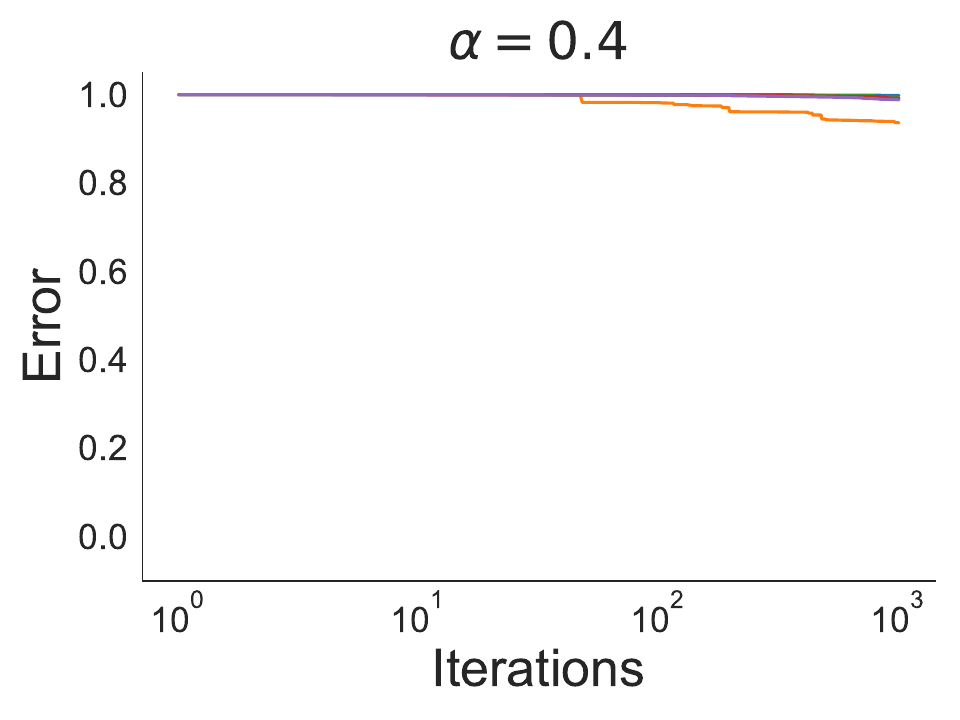}\hfill
  \includegraphics[scale=0.24, trim = {2.31cm 0 0 1.1cm}, clip]{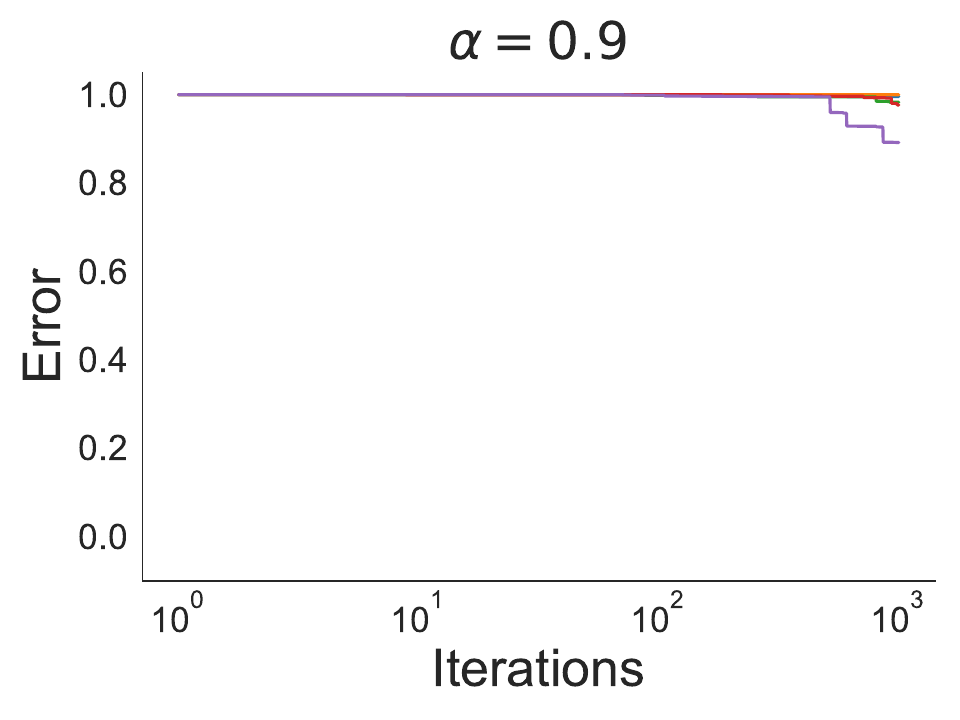}\hfill
  \includegraphics[scale=0.24, trim = {2.31cm 0 0 1.1cm}, clip]{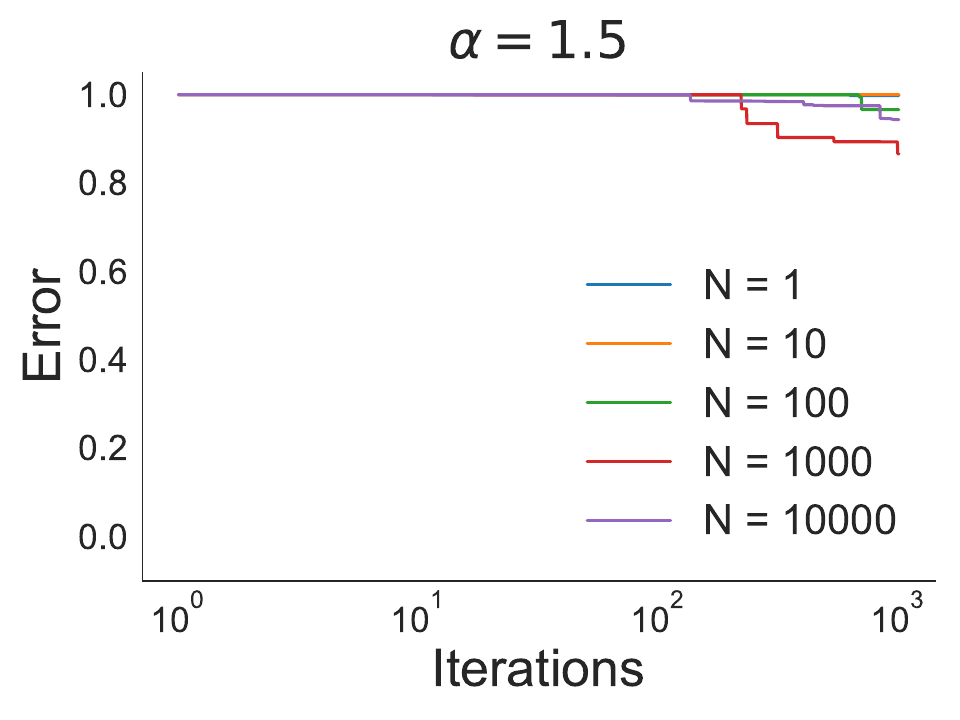}
  
  \caption{\textbf{Running SGD with unbiased gradient estimates becomes inefficient for moderately large dimensions and $\mathbf{\alpha}$.} Each plot shows the optimization results for each pair $(d, \alpha)$ for all values of $N$ considered. Plots show ``Error vs. Iteration'', where error is the normalized distance between the scale parameters of $p$ and $q_w$, computed as $(1/ d) \sum_i (\sigma_{qi} - 1)^2$.}
  \label{fig:noopt_gauss}
\end{figure*}

We present two unbiased estimators for $\nabla_w D_\alpha(p||q_w)$. The first one, obtained via reparameterization, is given by \citep{dieng2017chivi}
\small
\begin{equation}
    \grep(p, q_w, \epsilon) =
    \begin{cases}
    \displaystyle \frac{1}{\alpha^2 - \alpha} \nabla_w \left( \frac{p(\mathcal{T}_w(\epsilon))}{q_w(\mathcal{T}_w(\epsilon))} \right)^\alpha &
    \mbox{ if } \alpha \not\in \{0, 1\} \vspace{.25cm}\\
    \displaystyle \nabla_w \log \frac{q_w(\T)}{p(\T)} & \mbox{ if } 
    \alpha \rightarrow 0,
    \end{cases}
    \label{eq:grep}
\end{equation}
\normalsize
where $\epsilon \sim q_0(\epsilon)$. For $\alpha \rightarrow 0$ this estimator recovers $\gzrep$ from eq.~\ref{eq:estz}. As presented, this estimator is not defined for $\alpha \rightarrow 1$. A second unbiased estimator, obtained by applying reparameterization twice, is given by
\small
\begin{equation}
    \gdrep(p, q_w, \epsilon) =
    \begin{cases}
    \displaystyle \left . -\frac{1}{\alpha} \nabla_w \left( \frac{p(\mathcal{T}_w(\epsilon))}{q_v(\mathcal{T}_w(\epsilon))} \right)^\alpha \right|_{v=w} &
    \mbox{ if } \alpha \neq 0 \vspace{.25cm}\\
    \displaystyle \left . \nabla_w \log \frac{q_v(\T)}{p(\T)}  \right|_{v=w} & \mbox{ if } 
    \alpha \rightarrow 0.
    \end{cases}
    \label{eq:gdrep}
\end{equation}
\normalsize
This estimator is novel. We show its derivation in Appendix \ref{app:estimators}. For $\alpha\rightarrow 0$ it recovers $\gzdrep$ from eq.~\ref{eq:estz}. For $\alpha \neq 0$ it is derived by applying the reparameterization trick twice. This is similar to the derivation for the ``doubly-reparameterized'' gradient estimator \citep{doublyrep} for importance weighted auto-encoders \citep{IWVAE}. A third unbiased estimator may be obtained via the score function method \citep{williams_reinforce}. In this work we do not focus on this one since it has been consistently observed that reparameterization estimators outperform their score function counterparts.

In practice, we observed that $\gdrep$ often works better than $\grep$ (Appendix \ref{app:compare} shows an empirical comparison). This may not be surprising since (i) $\gdrep$ is a natural extension of $\gzdrep$, which often works better than $\gzrep$ when $\alpha \rightarrow 0$; (ii) $\gdrep$ has the property of being deterministically zero at the optimum $p = q_w$, which is not true for $\grep$; (iii) The use of double reparameterization has led to significant improvements over ``plain'' reparameterization for multi-samples objectives \citep{doublyrep}.

\textbf{Empirical evaluation.} We now present empirical results that motivate this work. These demonstrate two important phenomena. First, for larger $\alpha$, optimization scales poorly to high dimensions. Understanding this is the central goal of this paper. Second, this may happen even when the gradient estimator's variance is very small. This explains why we study SNR rather than ``raw'' variance.

We set $p$ to be a standard Gaussian in $d$ dimensions and $q_w$ to be a mean-zero fully-factorized Gaussian. The parameters of $q_w$ are $w = \sigma \in \R^d$, representing the standard deviation of each dimension of $q_w$. We initialize $\sigma_i = 2$, and optimize $D_\alpha(p||q_w)$ from eq.~\ref{eq:obj_a}. We do so by running SGD with the gradient estimator $\gdrep$ for 1000 steps. (Appendix \ref{app:compare} shows results using $\grep$, which are worse.)

We perform this optimization for three different dimensionalities, $d \in \{8, 32, 128\}$, for $\alpha \in \{0, 0.4, 0.9, 1.5\}$, and for gradient estimators obtained averaging $N$ samples, for $N \in \{1,10,10^2,10^3,10^4\}$. For each triplet $(d, \alpha, N)$ we tuned the step-size; we ran simulations for all step-sizes in the set $\{10^i\}_{i=-7}^7$ and selected the one that lead to the best final performance. All results are averages over 15 simulations.



Fig.~\ref{fig:noopt_gauss} shows the results. Optimization succeeds when the dimensionality $d$ is small or $\alpha$ is small. Indeed, for $\alpha \rightarrow 0$, optimization converges in approximately $30$ steps, regardless of the dimensionality and the number of samples used to estimate each gradient. However, when $\alpha$ is larger, increasing $d$ seems to cause major difficulties. For instance, for $d = 32$ and $\alpha = 1.5$, optimization does not meaningfully converge within 1000 steps, even using $10^4$ samples to estimate gradients. Furthermore, for $d = 128$ and $\alpha \neq 0$ optimization barely makes any progress regardless of the number of samples used to estimate gradients.

Optimization results using Adam \cite{adam} are shown  in Fig.~\ref{fig:noopt_gauss_adam} (Appendix~\ref{app:adam}). The same effect is observed; optimization converges properly when the dimensionality or $\alpha$ are low, but fails in high dimensions for larger values of $\alpha$. (We include a brief discussion of Adam's performance in Section~\ref{sec:discussion}.)

For $\alpha = 2$, previous work has attributed the scaling issues of unbiased methods to the use of gradient estimates with high variance \citep{kuleshov2017neural, dieng2017chivi}. While correct in spirit, care is needed. Some of the failures in Fig.~\ref{fig:noopt_gauss} happen despite {\em low}-variance gradients, because the true gradient is even smaller. Take $d = 128$, and let $\sigma_i = \sigma$ for all $i$, so that $q_w$ is isotropic. Fig.~\ref{fig:varNSRmotiv} (left) shows the variance of the gradient estimator for different values of $\sigma$ and $\alpha$. We see that increasing $\alpha$ sometimes decreases gradient variance. Instead, we attribute optimization's scaling issues to the estimator's SNR. As shown in Fig.~\ref{fig:varNSRmotiv} (center) we see that this decreases \textit{very} rapidly with higher $\alpha$.\footnote{In Appendix \ref{app:lowvar} we give an upper bound for the variance, and show that the variance of $\gdrep$ for $\alpha = 0.4$ becomes ``small'' for high dimensional problems. Surprisingly, the variance of each component of the estimator decreases as the dimension increases.}

\begin{figure*}[t]
  \centering
  \includegraphics[scale=0.25]{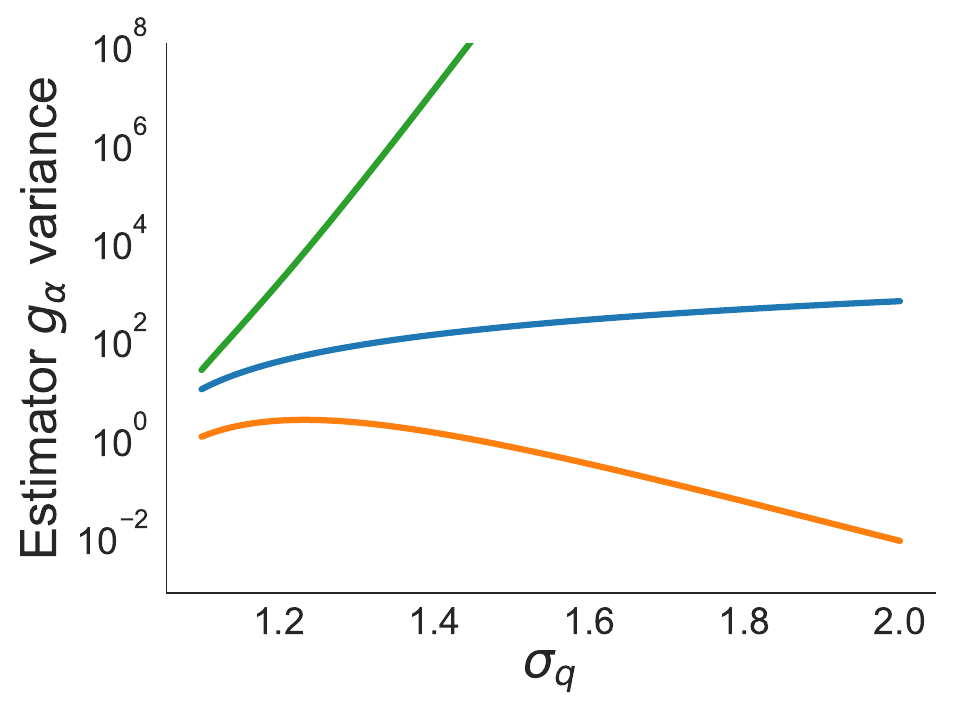}
  \includegraphics[scale=0.25]{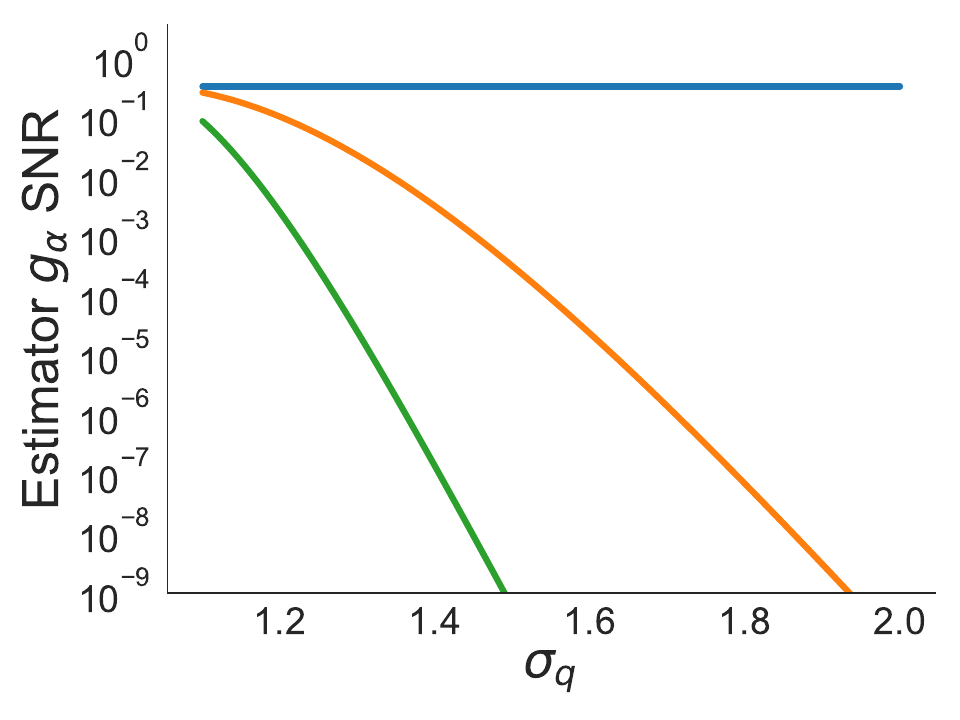}
  \includegraphics[scale=0.25]{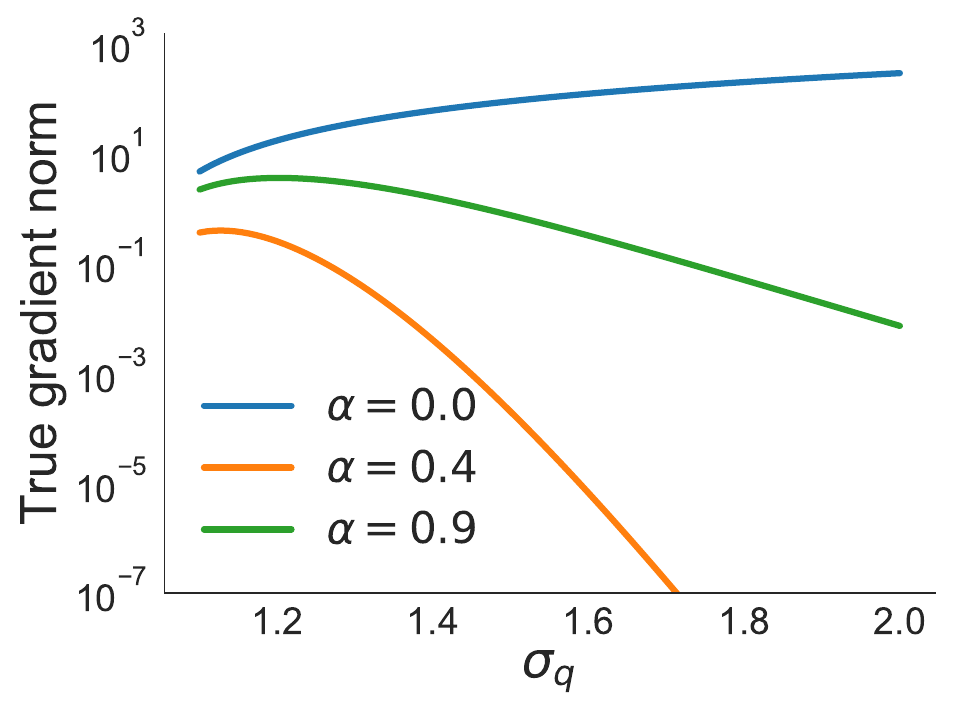}
  \caption{\textbf{Optimization difficulty is explained not by high variance but by low SNR.} $p$ is a standard Gaussian in $d=128$ dimensions and $q_w$ is an isotropic Gaussian with standard deviation $\sigma_q$. Left: Variance of gradient estimator. Center: SNR. Right: Squared norm of the true gradient, $\E[g]$.}
  \label{fig:varNSRmotiv}
\end{figure*}

\section{SNR Analysis} \label{sec:NSR}

In this section we present a detailed analysis of the SNR of gradient estimators for two general and representative scenarios. First, we consider the case where $p$ and $q_w$ are arbitrary fully-factorized distributions. Second, we consider the case where $p$ and $q_w$ are Gaussians with an arbitrary full-rank covariance matrices. This case is particularly relevant, since Gaussians are a good approximation of a huge range of posteriors (Bayesian central limit theorem). We show that, in both cases, for $\alpha \neq 0$ the gradient estimator's SNR becomes very small for problems with high dimensionality $d$. In fact, we present examples for which the SNR decreases \textit{exponentially} in $d$. In contrast to this, we show that, for $\alpha \rightarrow 0$ (typical VI), the SNR of the estimator decreases at most as $1/d$, and does not depend on $d$ if both $p$ and $q_w$ are factorized. Intuitively, a low SNR means that the level of noise present in the estimator is considerably larger than the ``learning signal'', which difficults optimization.

We proceed in a similar way for all scenarios considered. We first present a rigorous result, and then give a simple and intuitive interpretation and examples.

\subsection{Fully-Factorized Distributions}

We begin by studying the case where both $p$ and $q_w$ are arbitrary fully-factorized distributions. Of course, if we knew that $p$ is fully factorized, it would make sense to perform inference on each component separately. The point of examining this case is the insight it gives us into how gradient estimators behave as dimensionality changes. For simplicity, we assume that there is one parameter $w_i$ for each coordinate $z_i$. This assumption is used to simplify notation and can be removed, as long as each component is determined by disjoint sets of parameters.

\begin{restatable}{thm}{fullyfactgeneral}
\label{thm:NSR-gff}
Let $p(z)=\prod_{i=1}^d p_i(z_i)$\normalsize, $q_w(z)=\prod_{i=1}^d q_{w_i}(z_i)$\normalsize, and $\ga \in \{\grep, \gdrep\}$\normalsize. 

If $p_j \neq q_{w_j}$ and $g_\alpha$ has finite variance, the SNR of the $j$-th component of the estimator, $\gaj$, is given by
\small
\begin{equation}
\hspace{-1.7cm}\SNR[\gaj(p, q_w, \epsilon)] = \SNR\left[ \ga(p_j, q_{w_j}, \epsilon_j) \right]\label{eq:NSR-gff-z}
\end{equation}
for $\alpha \rightarrow 0$, and by
\begin{multline}
\SNR[\gaj(p, q_w, \epsilon)] \label{eq:NSR-gff-nz} = \\
\SNR\left[ \ga(p_j, q_{w_j}, \epsilon_j) \right] \prod^d_{\substack{i=1 \\ i\neq j}} \SNR\left[ \tilde D_\alpha(p_i, q_{w_i}, \epsilon_i) \right]
\end{multline}
\normalsize
for $\alpha \neq 0$, where 
\small$\tilde D_\alpha(p_i, q_{w_i}, \epsilon_i) = \left( \frac{p_i(\mathcal{T}_{w_i}(\epsilon_i))}{q_{w_i}(\mathcal{T}_{w_i}(\epsilon_i))} \right)^\alpha$\normalsize\
is an unbiased estimator of \small$\alpha (\alpha - 1) D_\alpha(p_i||q_{w_i}) + 1$\normalsize.

If $p_j = q_{w_j}$, the SNR is $0$ for $\grep$ and is undefined for $\gdrep$ (because $\gdrep$ is deterministically zero).
\end{restatable}




To clarify the Theorem's notation, $\gaj(p,q_w,\epsilon)$ is the $j$-th component of the estimator $\ga$ for the vector \small$\nabla_w D_\alpha(p\Vert q_w)$\normalsize. On the other hand, $\ga(p_j,q_{w_j},\epsilon_j)$ is the estimator for the scalar quantity \small$\nabla_{w_j} D_\alpha(p_j \Vert q_{w_j})$\normalsize, the derivative (with respect to $w_j$) of the divergence between the one dimensional distributions $p_j$ and $q_{w_j}$.

What does the theorem say? If $\alpha \rightarrow 0$ (eq.~\ref{eq:NSR-gff-z}), the SNR of each component of the estimator consists on a single term, which is the same as if inference were performed on each dimension of $p$ separately. That is, the SNR of the estimator's $j$-th component only depends on $p_j$ and $q_{w_j}$, and is not affected by the dimensionality of the problem $d$ in any way. In contrast, if $\alpha \neq 0$ (eq.~\ref{eq:NSR-gff-nz}), there are $d-1$ additional terms. These determine how the SNR scales with dimensionality. Since these terms can be expressed as the SNR of an estimator for {\small $D_\alpha(p_i\Vert q_{w_i})$} (for each $i\neq j$, up to scaling constants), each of them is at most one, with equality only if $p_i = q_{w_i}$. Thus, for $\alpha \neq 0$, discrepancies in several dimensions of $p$ and $q_w$ accumulate (as products of terms strictly smaller than one), leading to a large detrimental effect on the estimator's SNR. Intuitively, the larger the dimensionality of the problem $d$, the worse this effect becomes. 

An example that clearly illustrates this curse of dimensionality is given by the case where $p$ and $q$ are isotropic distributions. Suppose each component of $p$ and $q_w$ are the same, that is, $p_i = p_1$ and $q_{w_i} = q_{w_1}$ for all $i$. Following eq.~\ref{eq:NSR-gff-nz}, if $\alpha \neq 0$, the SNR of the $j$-th component of the gradient estimator is given by

\small
\begin{equation}\SNR\left[ \ga(p_1, q_{w_1}, \epsilon_1) \right] \left ( \SNR\left[ \tilde D_\alpha(p_1, q_{w_1}, \epsilon_1) \right]\right)^{d-1},\end{equation}
\normalsize
which worsens \textit{exponentially} in $d$. In contrast, if $\alpha \rightarrow 0$, the SNR does not depend on $d$ at all (eq.~\ref{eq:NSR-gff-z}).


\subsubsection{Fully-Factorized Gaussians}

As a second example of fully-factorized distributions, we consider the case where $p$ and $q_w$ are $d$-dimensional diagonal Gaussians with mean zero. The parameters are the standard deviations of each component of $q_w$, i.e. $w = \{\sigma_{q1}, \hdots, \sigma_{qd}\}$. In this case we can compute each term in eq.~\ref{eq:NSR-gff-nz} in closed form.

\begin{restatable}{corollary}{fullyfactgauss}
\label{thm:NSR-ffg}
Let $p$ and $q$ be two fully-factorized $d$-dimensional Gaussian distributions with mean zero and variances $\sigma_{pi}^2$ and $\sigma_{qi}^2$. Let $\R_i = \nicefrac{\sigma_{qi}^2}{\sigma_{pi}^2}$ and $\ga = \gdrep$.

If $\R_j \neq 1$ and $1 + 2\alpha(\R_i - 1) > 0$ for all $i$,
	\small\begin{multline}
	\SNR[\gaj(p, q_w, \epsilon)] = \\ \underbrace{\frac{1 + 2\alpha(\R_j - 1)}{3} f(\R_j, \alpha)^3}_{\scriptsize \SNR\left[ \ga(p_j, q_{w_j}, \epsilon_j) \right]\small}
	\,\,\prod^d_{\substack{i=1 \\ i\neq j}}
	\underbrace{f(\R_i, \alpha) \vphantom{\frac{1 + 2\alpha \R_j - 2\alpha}{3} f(\R_j, \alpha)^3}}_{\scriptsize \SNR\left[ \tilde D_\alpha(p_i, q_{w_i}, \epsilon_i) \right] \small}\hspace{-0.7cm} \label{eq:NSR-ffg}
	\end{multline}\normalsize
	where \small
	\begin{equation}
	f(\R, \alpha) = \frac{1}{\sqrt{1 + \alpha^2 \frac{(\R - 1)^2}{1 + 2\alpha(\R - 1)}}}. \label{eq:fc2}
	\end{equation}
	\normalsize
Otherwise, the SNR is not defined. If $\R_j = 1$ this is because the estimator $\gaj$ is 0 deterministically. If $1 + 2\alpha(\R_i - 1) \leq 0$ for any $i$, this is because the estimator has infinite variance.
\end{restatable}

\begin{figure*}[t]
	\centering
	\includegraphics[scale=0.27]{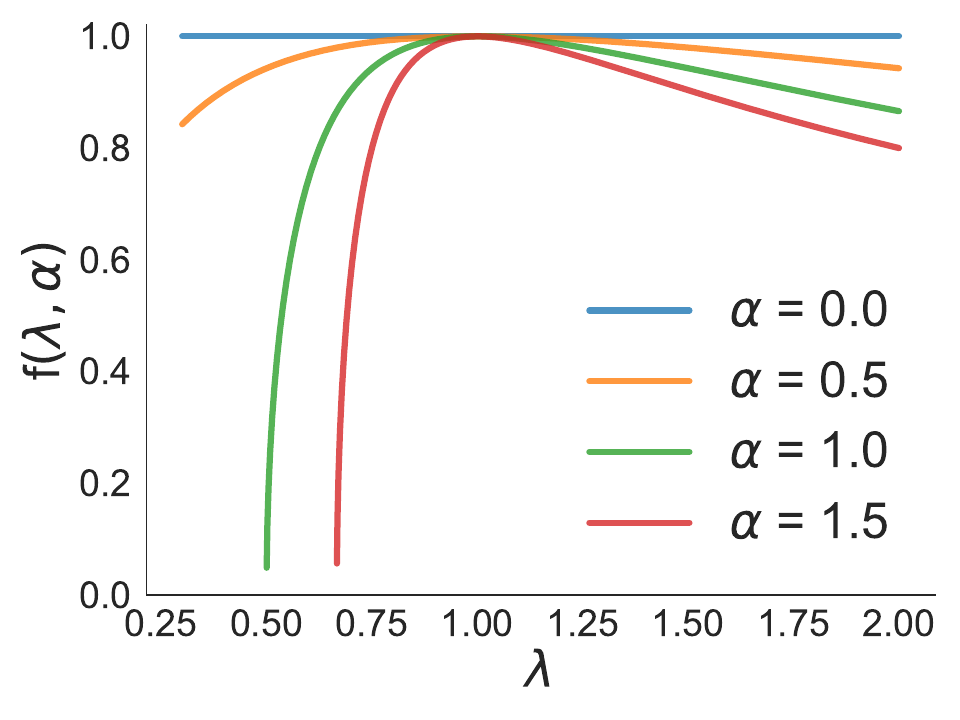}
	\includegraphics[scale=0.27]{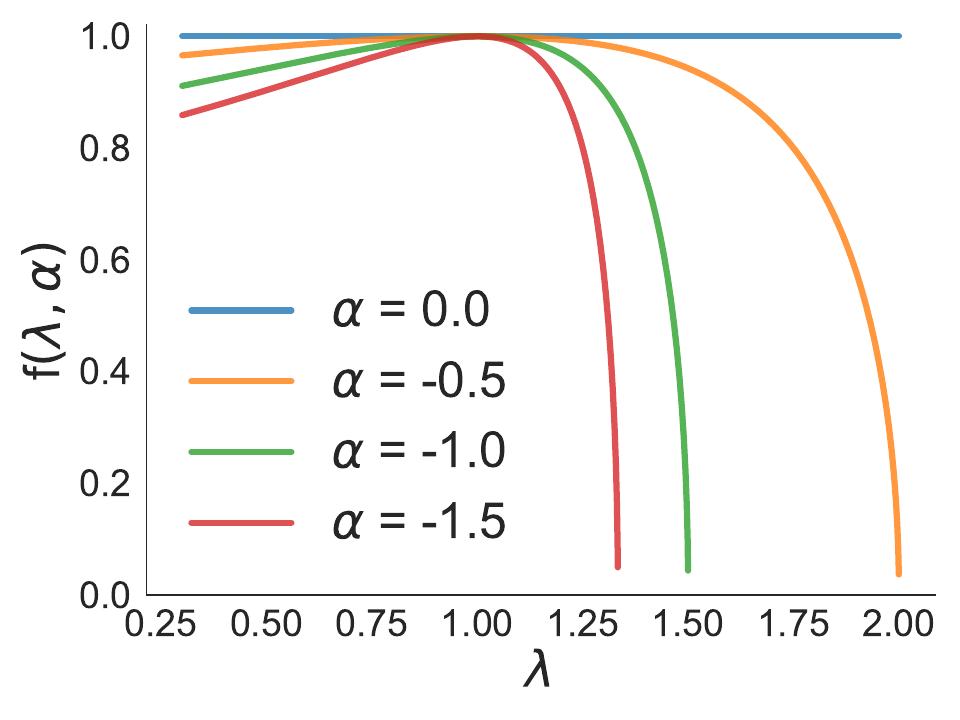}
	\includegraphics[scale=0.27]{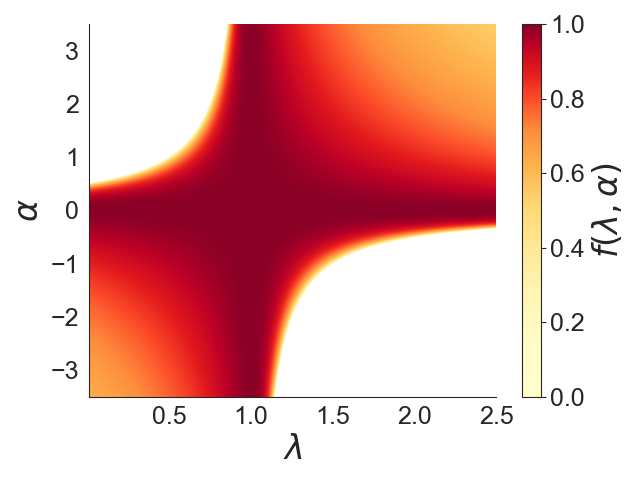}
	\caption{$f(\R, \alpha)$. White (rightmost plot) indicates regions where $1 + 2\alpha(\R - 1) > 0$ is not satisfied (in this region the estimator has infinite variance, and thus the SNR is not defined).}
	\label{fig:f}
\end{figure*}

Corollary \ref{thm:NSR-ffg} gives conditions under which the SNR of the gradient estimator is well-defined (i.e. estimator has finite variance), and gives an expression for the SNR in such cases. In order to understand this expression, it is important to understand the behavior of the function $f(\R, \alpha)$. Fig.~\ref{fig:f} shows a visualization. It can be observed that (i) $f(\R, \alpha)$ achieves its maximum value of 1 if and only if $\R = 1$ or $\alpha = 0$; and (ii) $f(\R, \alpha)$ decreases as $\alpha$ moves aways from $0$ or $\R$ moves away from $1$. We present a formal characterization of this function in Lemma \ref{lem:f-props} (Appendix~\ref{app:fproof}).

Again, the behavior of the SNR for problems with high dimensionalities is determined by the $d$-term product $\prod_i f(\R_i, \alpha)$ in eq.~\ref{eq:NSR-ffg}. We see that, if $\alpha \rightarrow 0$, each term in this product is just one (because $f(\R, 0) = 1$), and thus the SNR is just $1/3$. On the other hand, if $\alpha \neq 0$, each of the terms is at most one (with equality only if the corresponding $\R_i = 1$). Therefore, if $\alpha \neq 0$, discrepancies in several dimensions of $p$ and $q_w$ accumulate, leading to a large detrimental effect on the SNR of \textit{every} component of the estimator. In addition, in this case we can exactly quantify this deterioration in terms of $\alpha$ and $\R_i$: The SNR worsens for $\alpha$ values with large absolute value and when discrepancies between components is large (i.e. $\R_i$ far from 1), since for these cases $\f$ is significantly less than one (see Fig.~\ref{fig:f}).

In addition, there are entirely non-pathological cases for which the estimator has \textit{infinite} variance. This occurs whenever the condition $1 + 2\alpha(\R_i - 1) > 0$ is not satisfied for some $i$. For instance, this happens if we set $\alpha = 1$, $\sigma_{pi} = 1$ and $\sigma_{qi} = 0.7$. More generally, the condition is equivalent to $\alpha \sigma_q^2 > (\alpha - \nicefrac{1}{2})\sigma_p^2$. This is always satisfied for $0 \leq \alpha < \nicefrac{1}{2}$. If $\alpha \geq \nicefrac{1}{2}$, this means that the variance of $q_w$ cannot be much smaller than that of $p$. If $\alpha < 0$, this means that the variance of $q_w$ cannot be much larger than that of $p$.

\textbf{Example.} Consider the case where $\alpha = 0.4$, $p$ is a standard Gaussian with dimension $d = 128$, and $q_w$ is a mean zero factorized Gaussian with $\sigma_{qi} = 2$ for all $i$. Eq.~\ref{eq:NSR-ffg} yields $\SNR[\gaj] \approx 1.2 \times 10^{-10}$. This means that the variance of the estimator is approximately $8\times 10^9$ times larger than the actual signal. In contrast, for $\alpha \rightarrow 0$, the SNR is just $1/3$. Obtaining an estimator with a similar SNR for $\alpha = 0.4$ would require averaging $N \approx 4\times 10^9$ independent samples (this quantity grows exponentially if problems of larger dimensionality are considered).

\subsection{Gaussians with arbitrary Covariances}

We now move away from factorized distributions and consider the case in which both $p$ and $q_w$ are $d$-dimensional Gaussians with mean zero and arbitrary full-rank covariances $\Sigma_p$ and $\Sigma_q$. The set of parameters is given by $w = S$, where $S$ is a matrix such that $S S^\top = \Sigma_q$, and reparameterization is given by $\mathcal{T}_w(\epsilon) = S \epsilon$, where $\epsilon \sim \mathcal{N}(0, I)$.

\begin{restatable}{thm}{fullrankbound}
\label{thm:frg-NSR}
Let $p(z) = \mathcal{N}(z | 0, \Sigma_p)$ and $q(z) = \mathcal{N}(z | 0, \Sigma_q)$. Let $\lambda_1,\dots, \lambda_d$ be the eigenvalues of $\Sigma_p^{-1} \Sigma_q$ and $\ga = \gdrep$. 

If $\Sigma_p \neq \Sigma_q$ and $1 + 2\alpha(\lambda_i - 1) > 0$ for all $i$ we get
\small
\begin{equation}
\hspace{-4cm}\SNR[\ga(p, q_w, \epsilon)] = \frac{1}{d + 2} \label{eq:frg-z}
\end{equation}
\normalsize
for $\alpha \rightarrow 0$,
\small
\begin{equation}
\SNR[\ga(p, q_w, \epsilon)] \leq \left(\frac{1 + \alpha(\lambda_\mathrm{min} - 1)}{1 + 2\alpha(\lambda_\mathrm{max} - 1)}\right)^2 \prod_{i=1}^d f(\lambda_i, \alpha) \label{eq:NSR-frgp}
\end{equation}
\normalsize
for $\alpha > 0$, and 
\small
\begin{equation}
\SNR[\ga(p, q_w, \epsilon)] \leq \left(\frac{1 + \alpha(\lambda_\mathrm{max} - 1)}{1 + 2\alpha(\lambda_\mathrm{min} - 1)}\right)^2 \prod_{i=1}^d f(\lambda_i, \alpha)\label{eq:NSR-frgn}
\end{equation}
\normalsize
for $\alpha < 0$, where $\lambda_\mathrm{max} = \max_i \lambda_i$, $\lambda_\mathrm{min} = \min_i \lambda_i$ (these are both positive), and $f(\R, \alpha)$ is defined in eq.~\ref{eq:fc2}.

Otherwise, the SNR is not defined. If $\Sigma_p = \Sigma_q$, this is because the estimator is zero deterministically. If $1 + 2\alpha(\lambda_i - 1) \leq 0$ for any $i$, this is because the estimator has infinite variance.
\end{restatable}

The results in Theorem \ref{thm:frg-NSR} can be interpreted similarly to those in Corollary \ref{thm:NSR-ffg}. If $\alpha \rightarrow 0$, the SNR is just $1/(d+2)$, independent of $p$ and $q_w$. If $\alpha \neq 0$, the SNR's upper bound contains the product of $d$ terms, all at most 1, with equality only if the corresponding $\lambda_i = 1$. As with factorized distributions, for $\alpha \neq 0$, discrepancies between several dimensions of $p$ and $q_w$ accumulate, leading to a small SNR. As with fully-factorized Gaussians, this deterioration worsens for $\alpha$ values with large magnitude and for $\lambda_i$ far from one. The condition that must be satisfied to get an estimator with finite variance is similar to the one for factorized Gaussians. The only difference is that, in this case, $\lambda_i$ represents an eigenvalue of $\Sigma_p^{-1} \Sigma_q$, instead of the ratio $\nicefrac{\sigma_{qi}^2}{\sigma_{pi}^2}$.

\textbf{Example.} Consider the case where $p$ and $q_w$ are $d$ dimensional isotropic Gaussians with covariances $\sigma_p^2 I$ and $\sigma_q^2 I$, with $\sigma_p \neq \sigma_q$. If $\alpha \neq 0$, the estimator's SNR is upper bounded by $\propto f(\lambda, \alpha)^d$, where $f(\lambda, \alpha)$ is strictly less than 1. This upper bound goes to zero \textit{exponentially} as a function of $d$. In contrast, for $\alpha \rightarrow 0$, the SNR decreases as $1/d$.

It is worth mentioning that the bounds in Theorem~\ref{thm:frg-NSR} are obtained as a relaxation of an exact but much more technical result, shown in Section~\ref{sec:discussion} (Theorem~\ref{thm:thmafrg}). While this latter result is fully precise (it gives an exact expression for the SNR) it is hard to interpret, so we do not include it here.

\subsection{Effect of SNR on Optimization} \label{sec:SNRSGD}

We presented general and representative scenarios for which the gradient estimator's SNR becomes extremely small as the dimensionality of the problem increases. How does this affect optimization convergence? Under some regularity assumptions, an SGD convergence guarantee assuming a bound on the SNR is known. See, for instance, Theorem 4.8 by \cite{bottou_SGDreview}. Their eq. 4.9 is equivalent to an SNR bound. They show that (under some assumptions) SGD requires a number of iterations that is $\mathcal{O}(1/\SNR)$ to converge. Thus, an exponentially small SNR translates to an exponentially large number of SGD iterations. Intuitively, this is because a small SNR leads to a small step-size, which in turn leads to a large number of SGD iterations.

In addition, there are papers that analyze this from a more empirical perspective. For instance, \cite{shalev2017failures} relate gradient estimators with extremely low SNRs to complete failures of gradient-based optimization methods. They mention that when the SNR approaches small values, the noise can completely mask the signal, and thus gradients are not sufficiently informative for optimization to succeed.


\section{Experiments and Results} \label{sec:exps}

Results in this work show that, for the scenarios considered, if $\alpha \neq 0$ unbiased estimates of \small$\nabla_w D_\alpha(p||q_w)$\normalsize\ suffer from a low SNR, which worsens fast with the dimensionality of the problem. Thus, methods based on these estimates will not scale to high dimensional problems. In this section we empirically show similar severe scalability issues for Bayesian logistic regression models.

We use two datasets: \textit{Iris} and \textit{Australian}, which have dimensionalities $4$ and $14$, respectively. For both datasets we used a subset of $100$ samples. For \textit{Iris} this reduced to keeping only data-points from two classes (out of the original three), while for \textit{Australian} we subsampled $100$ data-points. We use a diagonal Gaussian as variational distribution $q_w$, initialized to have mean zero and covariance identity.

When $p$ is a posterior, we cannot directly estimate gradients of the alpha-divergence since $p(z|x)$ is intractable. However, if we define the ``$\alpha$-ELBO''
\begin{equation}
\mathcal{L}_\alpha(w) = \frac{1}{\alpha(1 - \alpha)} \E_{q_w(z)}\left[ \left(\frac{p(x, z)}{q_w(z)}\right)^\alpha - 1\right], \label{eq:obj_L}
\end{equation}
then it's easy to show that maximizing $\mathcal{L}_\alpha$ is equivalent to minimizing the alpha-divergence, since
\begin{equation}
f_\alpha(p(x)) = \mathcal{L}_\alpha(w) + p(x)^\alpha D_\alpha(p(z|x)||q(z)) \label{eq:elbo_decomp}
\end{equation}
for $f_\alpha(x) = \frac{1}{\alpha(1-\alpha)}(x^\alpha - 1)$. Thus a gradient of $\mathcal{L}_\alpha$ is equal to a gradient of $D_\alpha$ up to a sign change and a multiplication by the constant factor of $p(x)^\alpha$. Eq. \ref{eq:elbo_decomp} gives a lower-bound on $p(x)$ for $\alpha < 1$ and an upper-bound for $\alpha > 1$ (corresponding to the cases where $f_\alpha$ is increasing and decreasing, respectively).

We optimize $\mathcal{L}_\alpha$ by running SGD with unbiased gradient estimates for $1000$ steps. We do this for $\alpha \in \{0, 0.1, 0.2, 0.3\}$ and for $N \in \{1, 10, 10^2, 10^3, 10^4\}$ (number of samples used to estimate the gradient at each step). For each pair $(\alpha, N)$ we tuned the step-size; we ran simulations for all step-sizes in the set $\{10^i\}_{i=-7}^{7}$ and selected the top-performing one. All results shown are averages over 15 simulations.

\begin{figure*}[ht]
  \centering
  \includegraphics[scale=0.24, trim = {0 2cm 0 0}, clip]{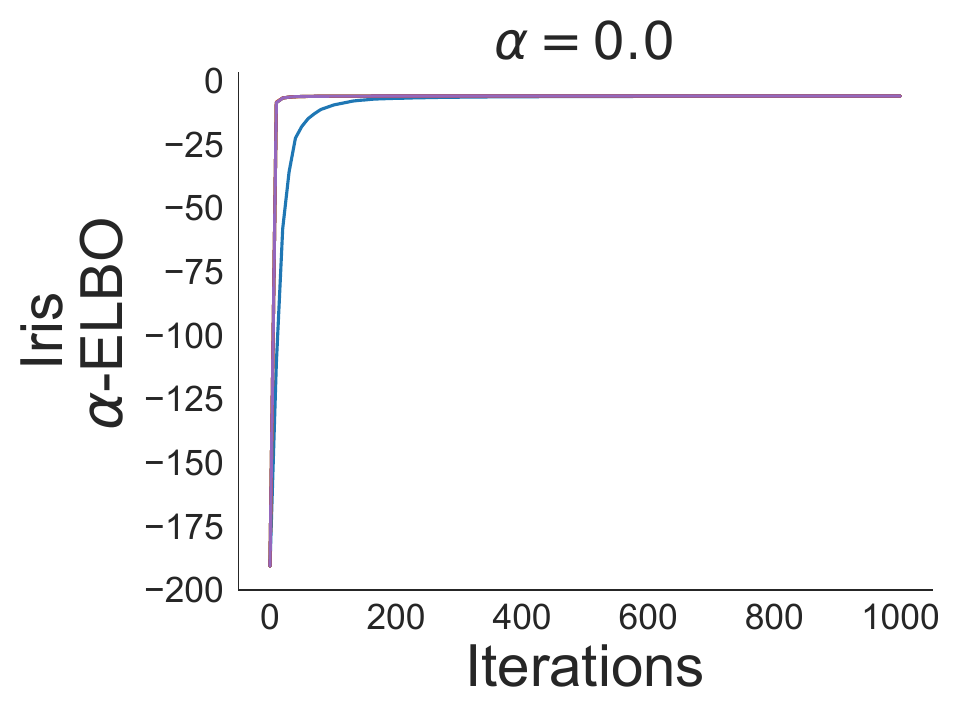}\hfill
  \includegraphics[scale=0.24, trim = {0 2cm 0 0}, clip]{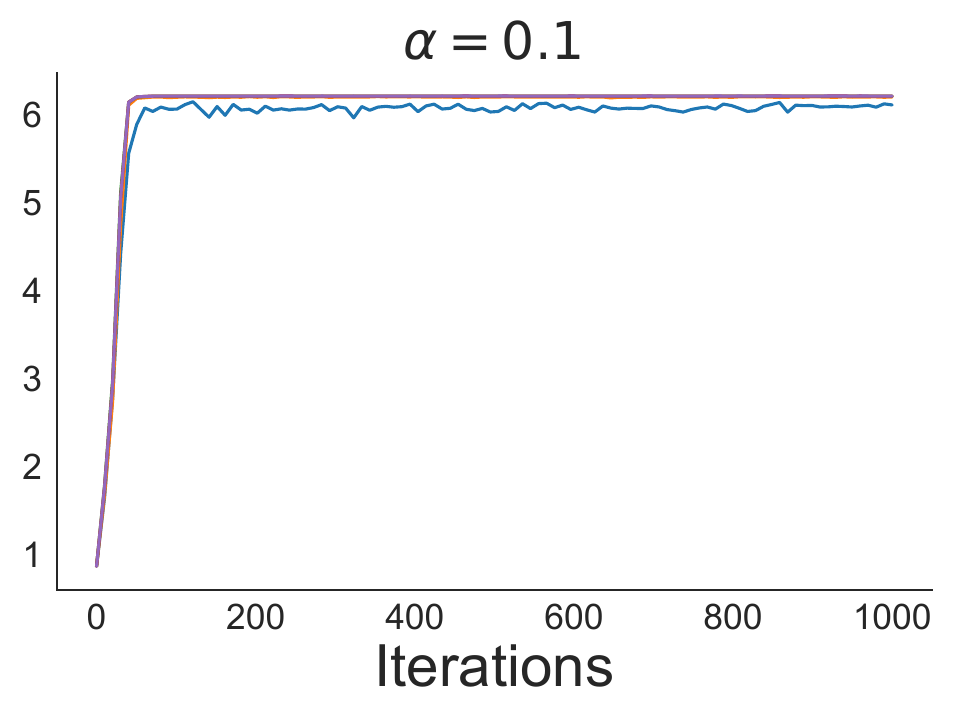}\hfill
  \includegraphics[scale=0.24, trim = {0 2cm 0 0}, clip]{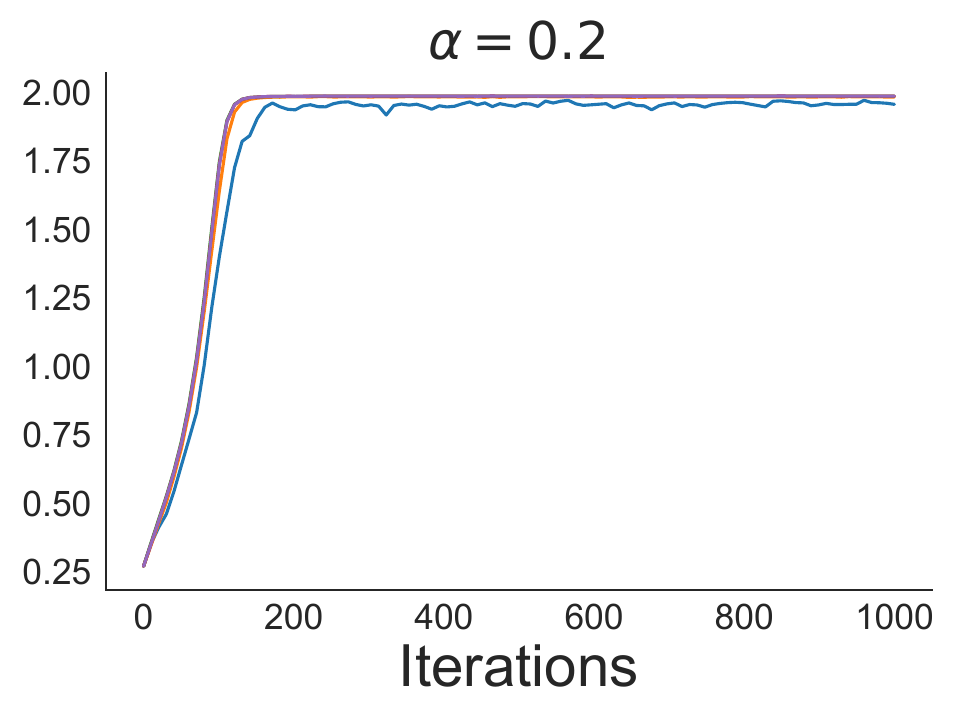}\hfill
  \includegraphics[scale=0.24, trim = {0 2cm 0 0}, clip]{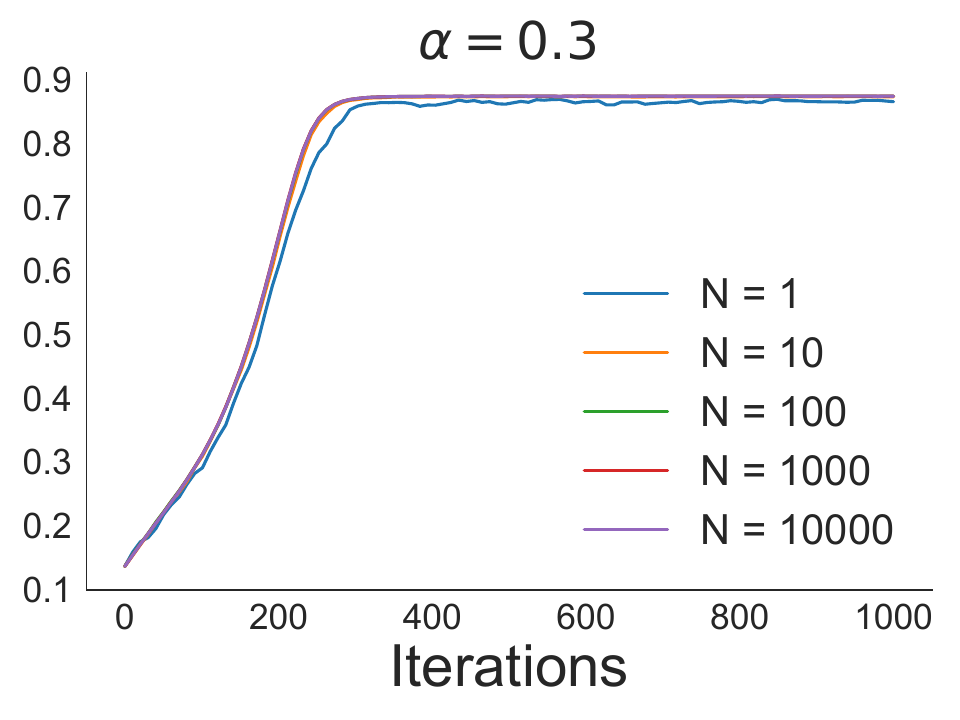}

  \vspace{0.1cm}

  \includegraphics[scale=0.24, trim = {0 0 0 0}, clip]{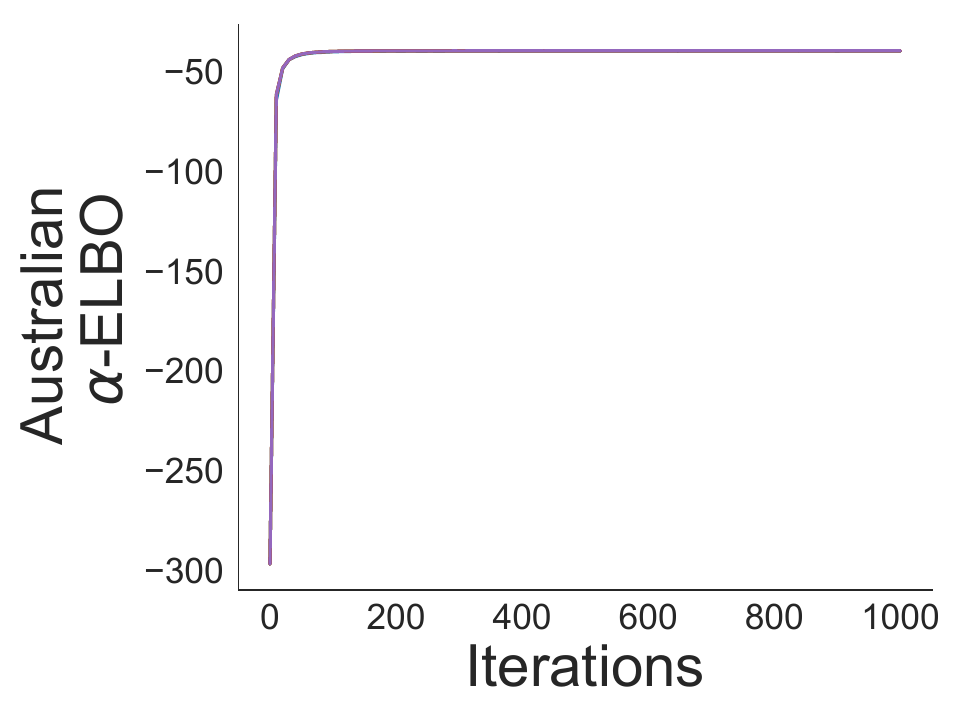}\hfill
  \includegraphics[scale=0.24, trim = {0 0 0 0}, clip]{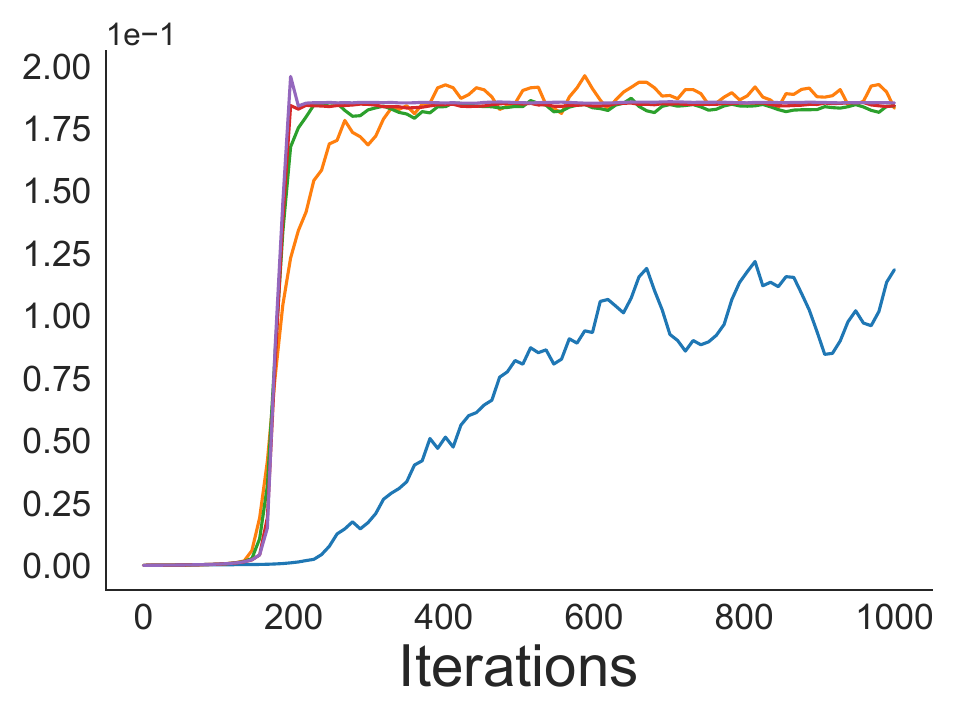}\hfill
  \includegraphics[scale=0.24, trim = {0 0 0 0}, clip]{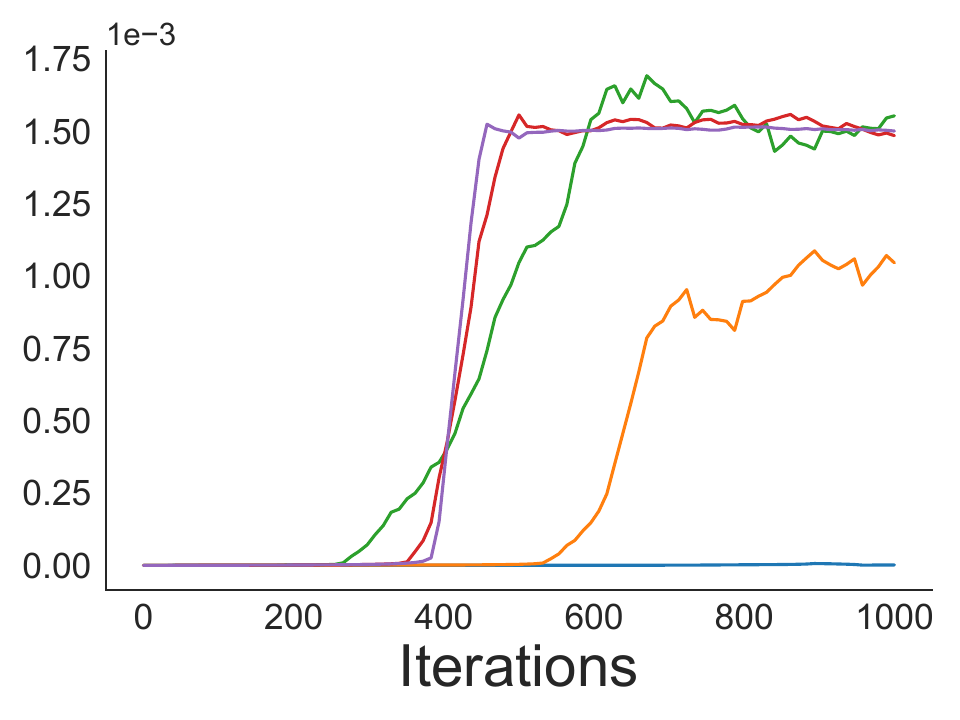}\hfill
  \includegraphics[scale=0.24, trim = {0 0 0 0}, clip]{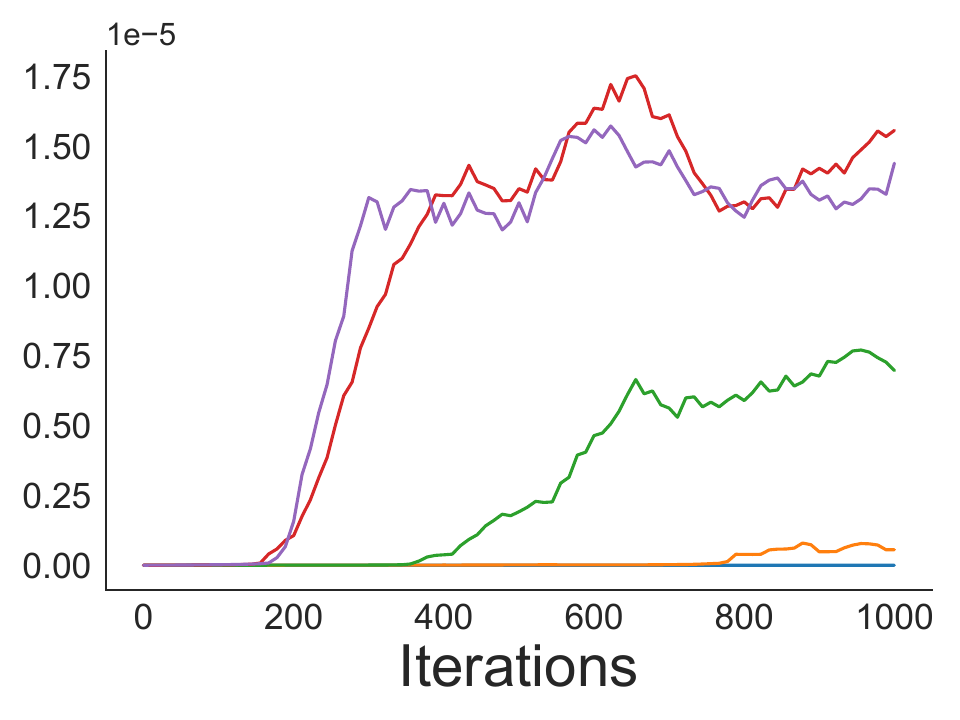}
  \caption{Optimization results for each $\alpha$ for all values of $N$ considered. The loss at each step (eq. \ref{eq:obj_L}) is estimated using $2.5 \times 10^5$ samples for both datasets.}
  \label{fig:exps_lr_opt}
\end{figure*}

Optimization results are shown in Fig.~\ref{fig:exps_lr_opt}. For the smaller dataset, \textit{Iris} ($d = 4$), optimization converges properly for all values of $\alpha$ considered regardless of the number of samples $N$ used to estimate the gradient. The situation is different for the \textit{Australian} dataset ($d = 14$). In this case, optimization converges properly for $\alpha \rightarrow 0$, but as $\alpha$ is increased a much larger number of samples $N$ is required to retain convergence ($N \geq 1000$ for $\alpha = 0.3$). This shows that, even for a simple logistic regression model of low dimension ($d = 14$), alpha-divergence minimization methods based on unbiased gradient estimates scale \textit{very} poorly with the dimensionality of the problem when $\alpha \neq 0$. We also ran simulations with larger datasets ($d \approx 40$), for which optimization barely made any progress at all regardless of the number of samples $N$ used. Similar results are obtained using the Adam optimizer (shown in Fig.~\ref{fig:exps_lr_opt_adam}, Appendix~\ref{app:adam}).

Fig.~\ref{fig:NSRLR} shows the SNR of the estimator for different values of $\alpha$ along a single optimization path, obtained by minimizing $\mathcal{L}_\alpha$ for $\alpha \rightarrow 0$ (equivalent to maximizing the ELBO). For the smaller dataset, \textit{Iris} ($d = 4$), all values of $\alpha$ considered lead to comparable SNRs. In contrast, for the \textit{Australian} dataset ($d = 14$), the SNR descreases rapidly as $\alpha$ is increased.

\begin{figure}[ht]
  \centering
  \includegraphics[scale=0.24, trim = {0 0 0 0}, clip]{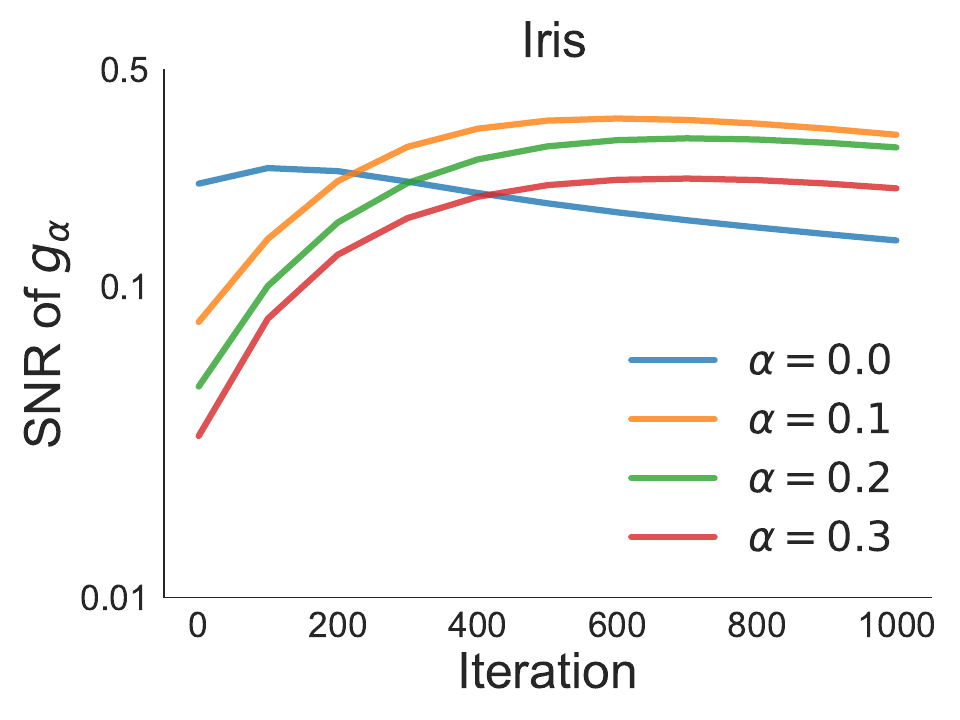}
  \includegraphics[scale=0.24, trim = {1.2cm 0 0 0}, clip]{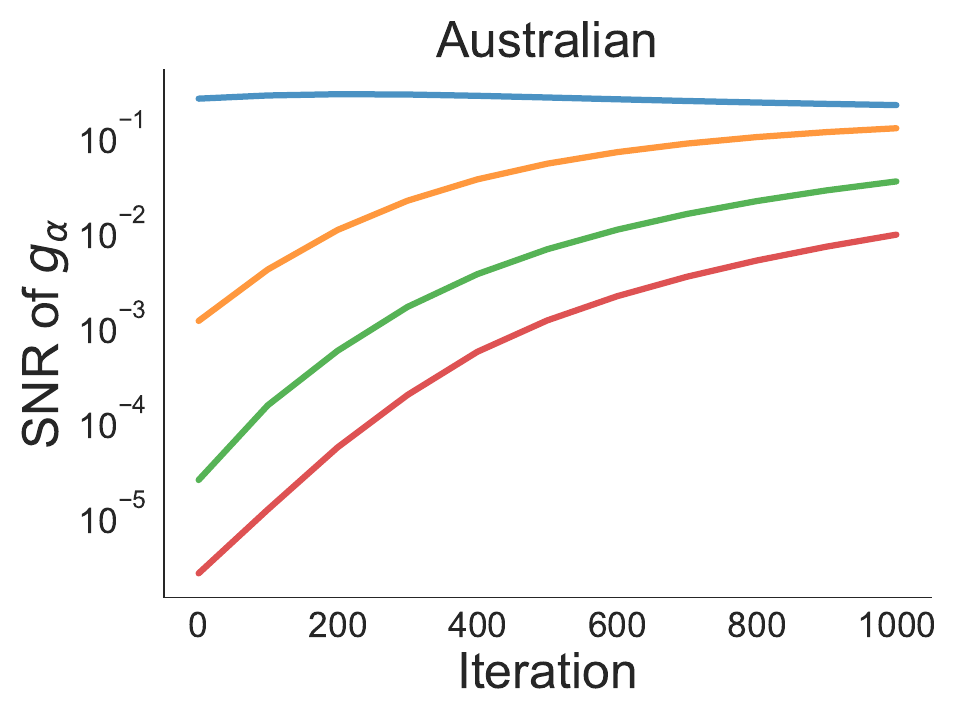}
  \caption{Direct test of SNR for different values of $\alpha$ along a single shared optimization path.}
  \label{fig:NSRLR}
\end{figure}

\section{Discussion} \label{sec:discussion}

We study unbiased methods for alpha-divergence minimization with the goal of understanding when these methods may be successfully applied. We present a detailed analysis of the SNR of unbiased gradient estimates for different scenarios. Our results are pessimistic. Suppose the variational family is any fully-factorized family, or the set of full-rank Gaussians. We show that in the favorable case where the posterior is inside the variational family, the SNR degrades catastrophically in the dimensionality. Optimization theory suggests that an exponential amount of computation time would be needed to optimize the objectives.

Interestingly, results in this work rule out some potential ``intuitive'' solutions. For instance, one might think that, for $\alpha > 0$, using an over-dispersed distribution $q_w$ could mitigate variance issues. However, Theorem \ref{thm:NSR-gff} shows this is not the case. While an over-dispersed distribution might help avoid estimators with infinite variance, it would not avoid the exponential deterioration of the SNR in terms of the dimensionality of the problem.

In addition, one could also consider using an adaptive gradient optimization method, such as Adam \cite{adam}. As it can be observed in Figs.~\ref{fig:noopt_gauss_adam} and \ref{fig:exps_lr_opt_adam} in Appendix~\ref{app:adam}, this does not solve the issue. This is not really surprising. Intuitively, this is because the extremely low SNR is a property of the gradient estimator, and thus will difficult optimization regardless of the gradient-based optimizer used. Indeed, Section 2.1 of the original Adam paper \cite{adam} states\footnote{v9 on arxiv.} ``With a smaller SNR the effective step-size $\Delta t$ will be closer to zero. This is a desirable property, since a smaller SNR means that there is greater uncertainty about whether the direction of $m$ corresponds to the direction of the true gradient.'' In addition, the theoretical results by \citet{shalev2017failures} are independent of the optimization algorithm used, and some of their negative empirical results were obtained using Adam.

Why is the behavior for $\alpha \neq 0$ so different to $\alpha \to 0$? Our understanding is that that the problematic terms vanish in the limit of $\alpha \to 0$. For example, eq.~\ref{eq:NSR-gff-nz} becomes eq.~\ref{eq:NSR-gff-z}, due to the fact that $\tilde{D}_\alpha \to 1$ as $\alpha \to 0$. Similarly, in eq.~\ref{eq:NSR-ffg}, we get that $f(\lambda, \alpha) \to 1$ as $\alpha \to 0$, meaning that only $\mathrm{SNR}[g_\alpha]$ remains. For full-rank Gaussians, it is probably easiest to understand via the exact result in Thm.~\ref{thm:thmafrg} (see below). There, if $\alpha \to 0$, we have that $f(\lambda, \alpha) \to 1$, $U \to I$, $V \to I$, meaning the overall SNR becomes $1 / (d + 2)$ (exactly eq.~\ref{eq:frg-z}) which has no problematic exponential dependence on dimensionality.

Given the failure of unbiased methods, one could consider using some biased alternative. However, it has been recently observed that, in high dimensions, these methods return suboptimal solutions that fail to minimize the target alpha-divergence \cite{biasedalphafail}. An analysis analogous to the one presented in this work is needed to understand this failure. We believe that such an analysis might be related to the curse of dimensionality for self-normalized importance sampling, which conjectures that to get meaningful results the number of samples used from the proposal distribution should be exponential in the dimensionality of the problem \citep{bengtsson2008curse, bugallo2017adaptive}.

\textbf{Exact result for Gaussians.} Theorem~\ref{thm:frg-NSR} bounds the gradient estimator's SNR for Gaussians with arbitrary covariances. While it admits a nice interpretation, it is not tight. As mentioned previously, an exact result is possible. We include it here. While this result is fully precise, it is harder to interpret than the bounds from Theorem~\ref{thm:frg-NSR}. It may be possible to find tighter bounds that are still ``simple'', or to find an intuitive interpretation of the exact result. We believe a step in these directions may further increase our understanding of these methods. 

\begin{restatable}{thm}{fullrankexact}
\label{thm:thmafrg}
Take the setting of Theorem \ref{thm:frg-NSR} with $\Sigma_p \neq \Sigma_q$ and $1 + 2\alpha (\lambda_i - 1) > 0$ for all $i$. Let $S$ be a matrix such that $\Sigma_q = SS^\top$ and let $\alpha \neq 0$. Then,
\small
\begin{equation*}
\SNR[\ga(p, q_w, \epsilon)] = \frac{\Vert B U^{-1} \Vert^2_F \,\times\, \prod_{i=1}^d f(\lambda_i, \alpha)}{\mathrm{tr}(V^{-1}) \mathrm{tr}(B V^{-1}B^\top) + 2 \Vert B V^{-1} \Vert_F^2},
\end{equation*}
\normalsize
where  $B = (\Sigma_p^{-1} - \Sigma_q^{-1}) S$\normalsize, $U = (1-\alpha) I + \alpha S^\top \Sigma_p^{-1} S$\normalsize, and $V = (1-2\alpha) I + 2\alpha S^\top \Sigma_p^{-1} S$\normalsize.
\end{restatable}

\section*{Acknowledgements}

We would like to thank Javier Burroni and Juan Cristi for useful feedback and suggestions.

\bibliography{control}
\bibliographystyle{icml2021}

\appendix
\newpage
\onecolumn

\appendix

\section{Comparison of Estimators} \label{app:compare}

In this section we present a simple empirical comparison between the estimators $\gdrep$ and $\grep$. We consider the same setting as in Section \ref{sec:motiv}, where $p$ is a $d$-dimensional standard Gaussian, and $q_w$ is set to be a mean-zero factorized Gaussian, with scales initialized to 2. We optimize $D_\alpha(p||q_w)$ from eq. \ref{eq:obj_a}. We do so by running SGD with the estimators $\gdrep$ and $\grep$ for 1000 steps.

We perform this optimization for four different dimensionalities, $d \in \{2, 8, 16, 32\}$, for $\alpha \in \{0, 0.4, 0.9, 1.5\}$, and for gradient estimators obtained averaging $N$ samples, for $N \in \{1,10,10^2\}$. For each estimator and triplet $(d, \alpha, N)$ we tune the step-size; we ran simulations for all step-sizes in the set $\{10^i\}_{i=-7}^7$ and select the one that leads to the best final performance. All results shown are average over 15 simulations.

Fig. \ref{fig:compare_gauss} shows the results. It can be observed that the results obtained using estimator $\gdrep$ are strictly better that the results with $\grep$. In some cases we can even observe that $\gdrep$ using $N = 1$ performs better than $\grep$ using $N = 100$.

\begin{figure}[h]
  \centering
  \includegraphics[scale=0.22, trim = {0 2.3cm 0 0}, clip]{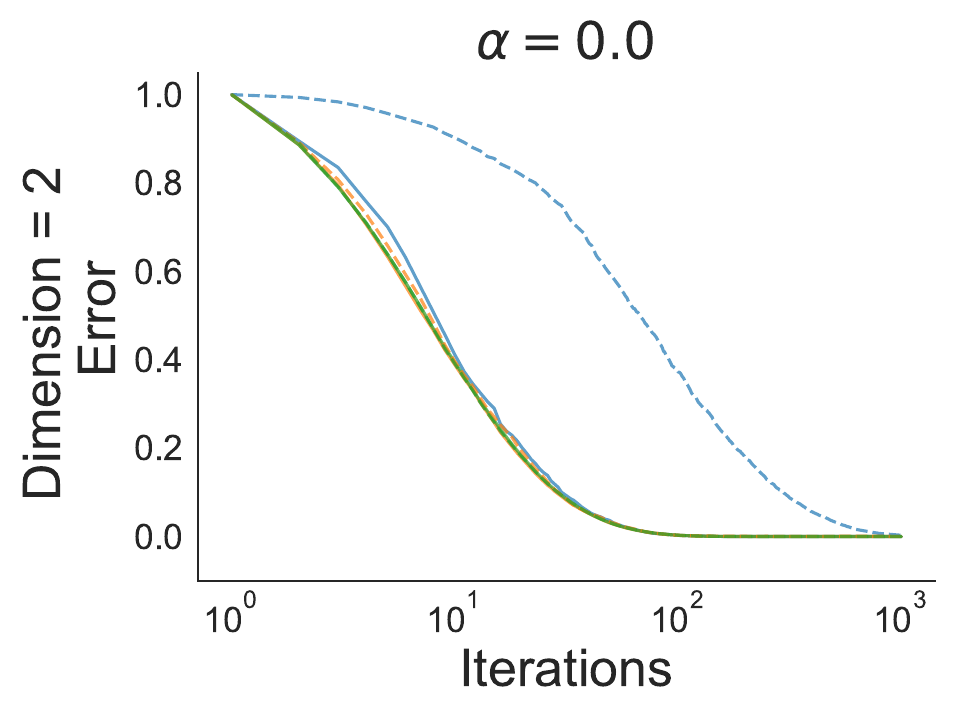}\hfill
  \includegraphics[scale=0.22, trim = {2.31cm 2.3cm 0 0}, clip]{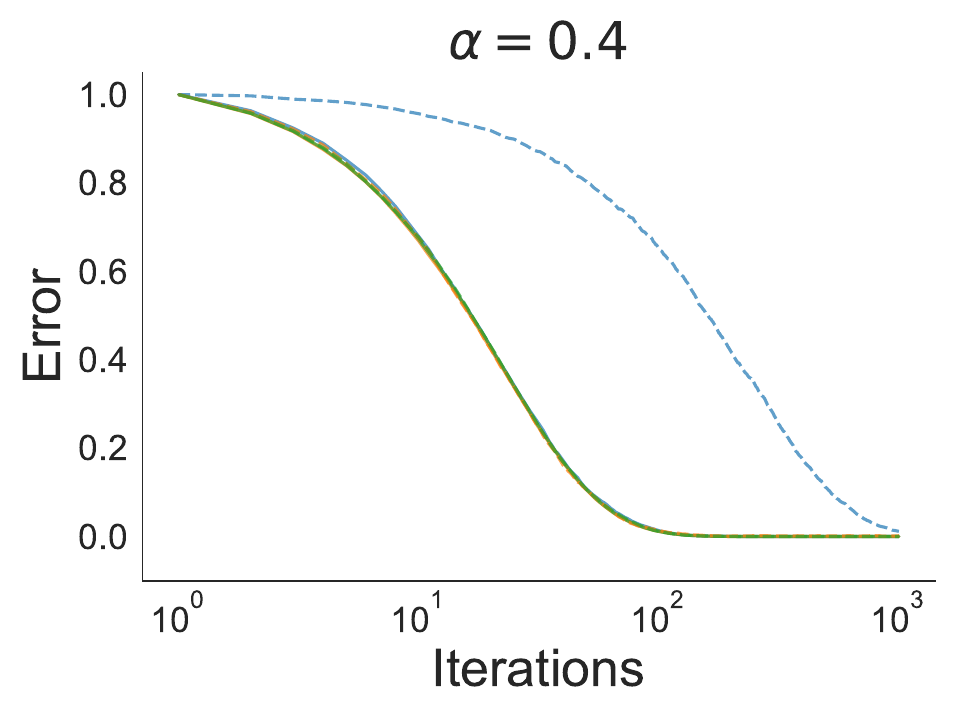}\hfill
  \includegraphics[scale=0.22, trim = {2.31cm 2.3cm 0 0}, clip]{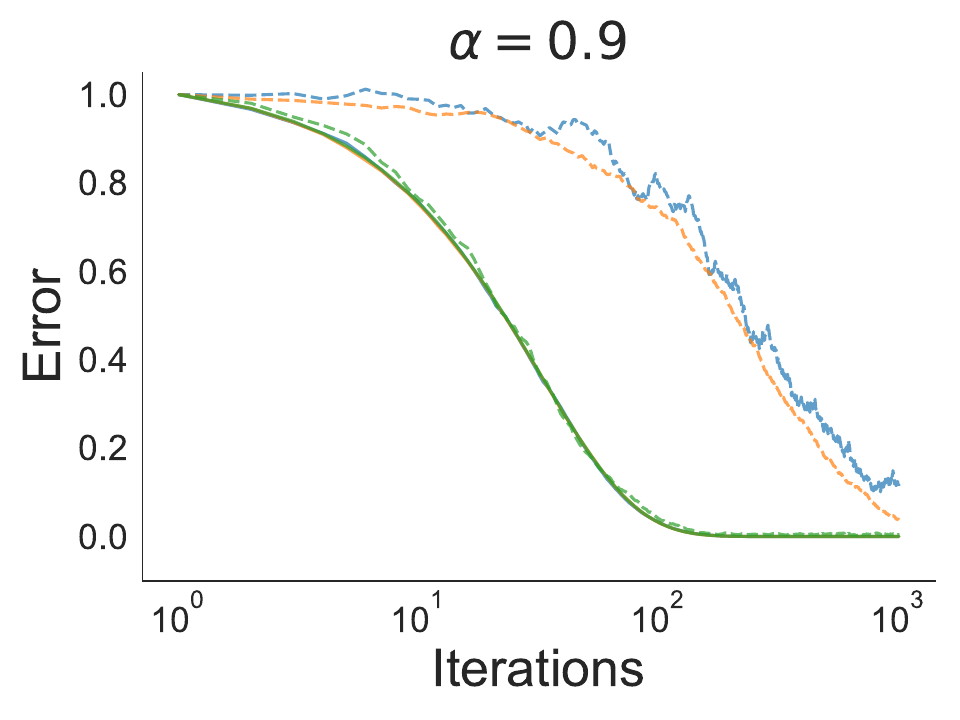}\hfill
  \includegraphics[scale=0.22, trim = {2.31cm 2.3cm 0 0}, clip]{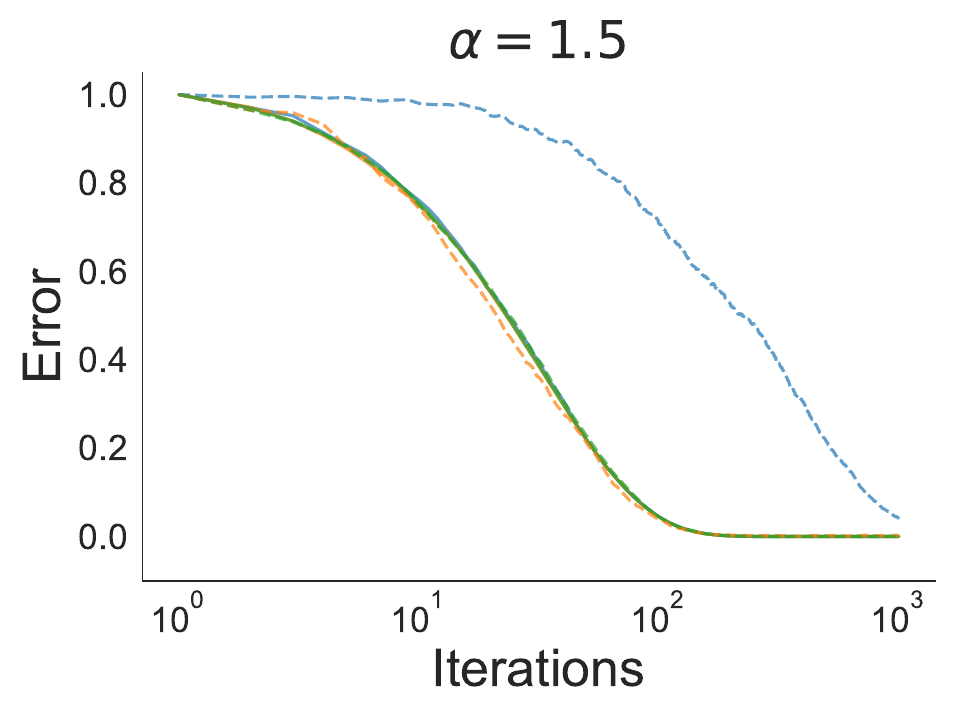}

  \vspace{0.2cm}

  \includegraphics[scale=0.22, trim = {0 2.3cm 0 1.05cm}, clip]{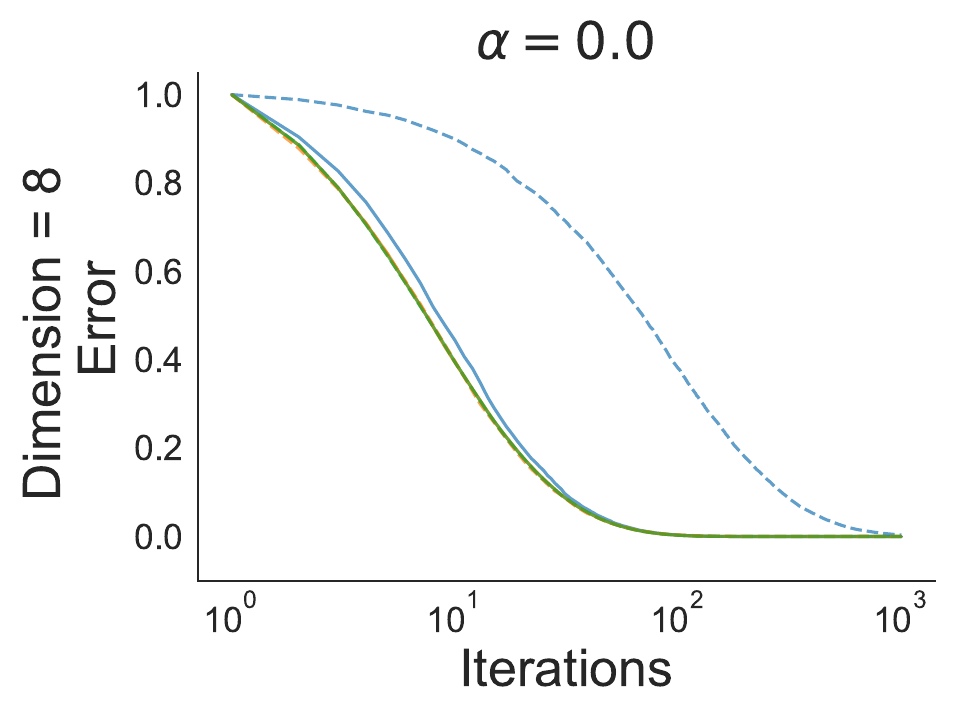}\hfill
  \includegraphics[scale=0.22, trim = {2.31cm 2.3cm 0 1.05cm}, clip]{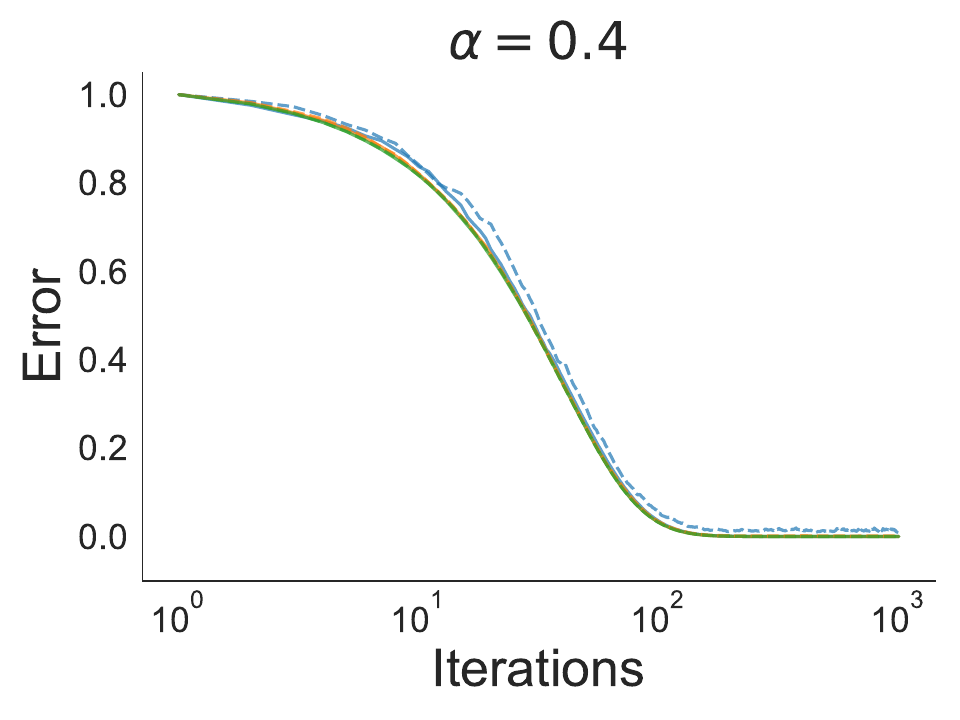}\hfill
  \includegraphics[scale=0.22, trim = {2.31cm 2.3cm 0 1.05cm}, clip]{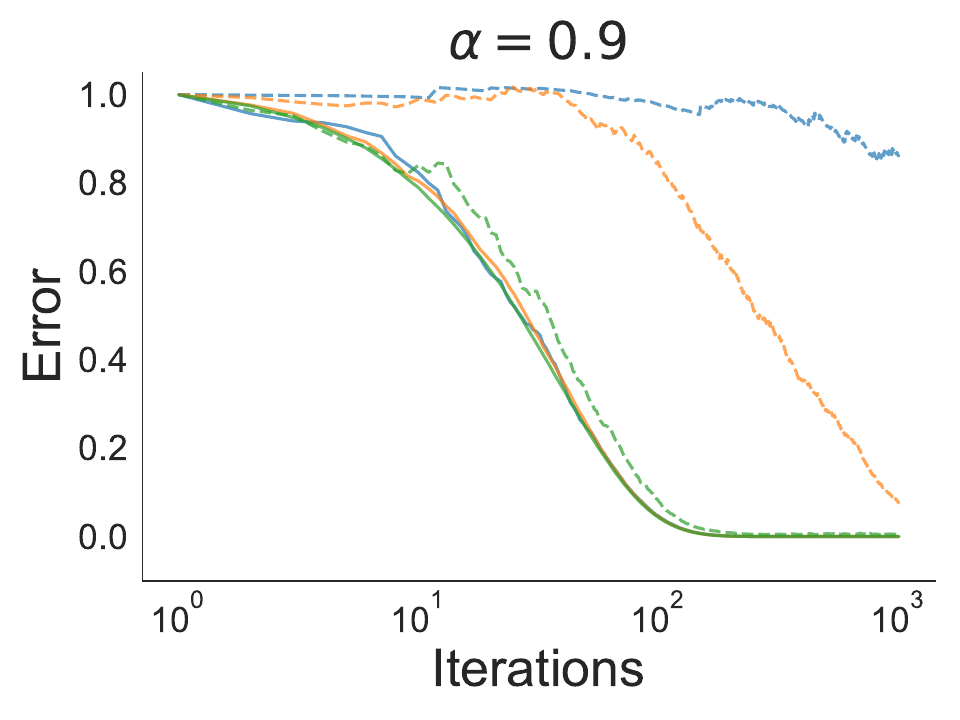}\hfill
  \includegraphics[scale=0.22, trim = {2.31cm 2.3cm 0 1.05cm}, clip]{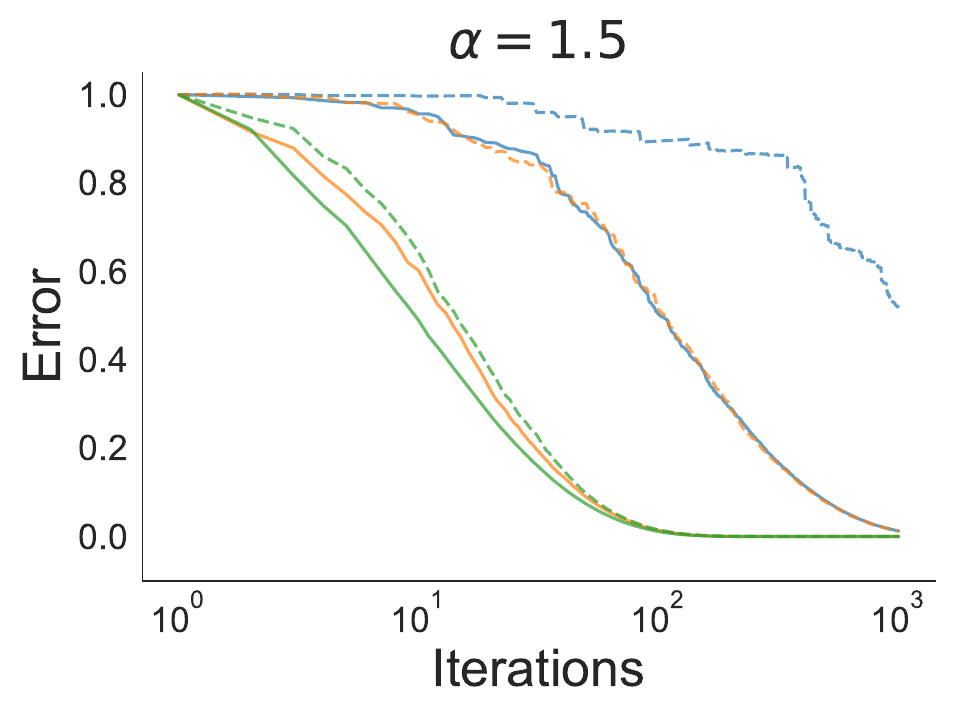}

  \vspace{0.2cm}

  \includegraphics[scale=0.22, trim = {0 2.3cm 0 1.05cm}, clip]{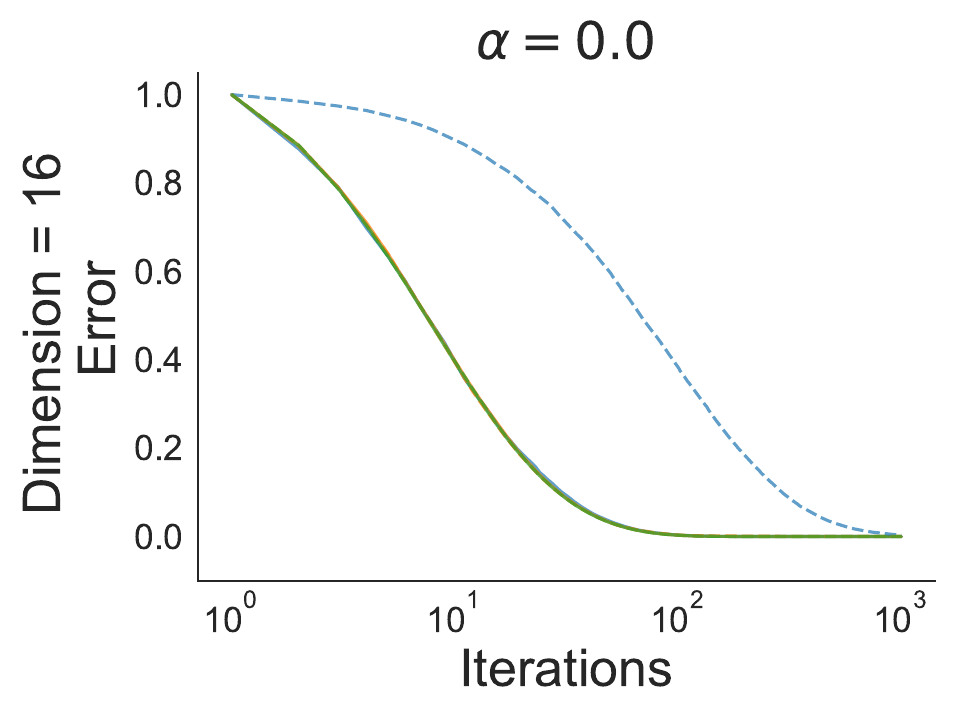}\hfill
  \includegraphics[scale=0.22, trim = {2.31cm 2.3cm 0 1.05cm}, clip]{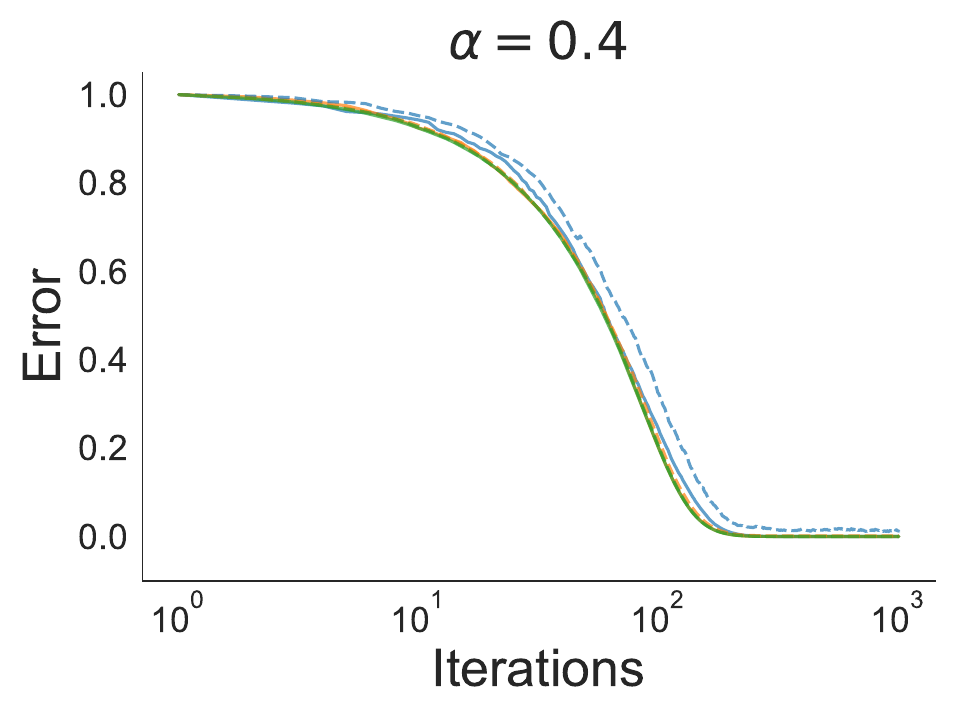}\hfill
  \includegraphics[scale=0.22, trim = {2.31cm 2.3cm 0 1.05cm}, clip]{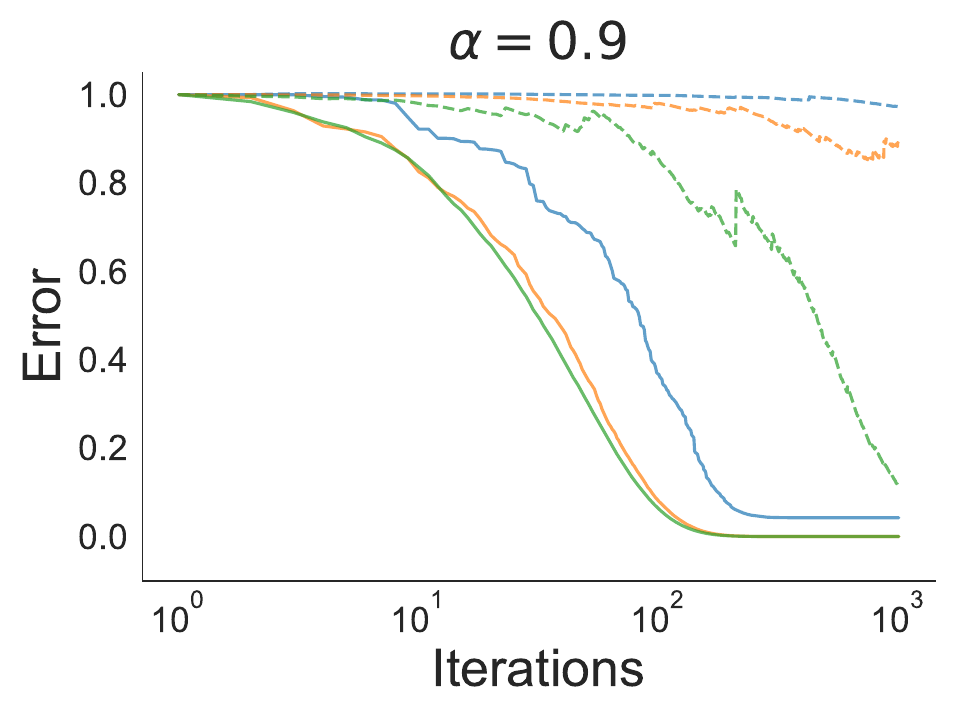}\hfill
  \includegraphics[scale=0.22, trim = {2.31cm 2.3cm 0 1.05cm}, clip]{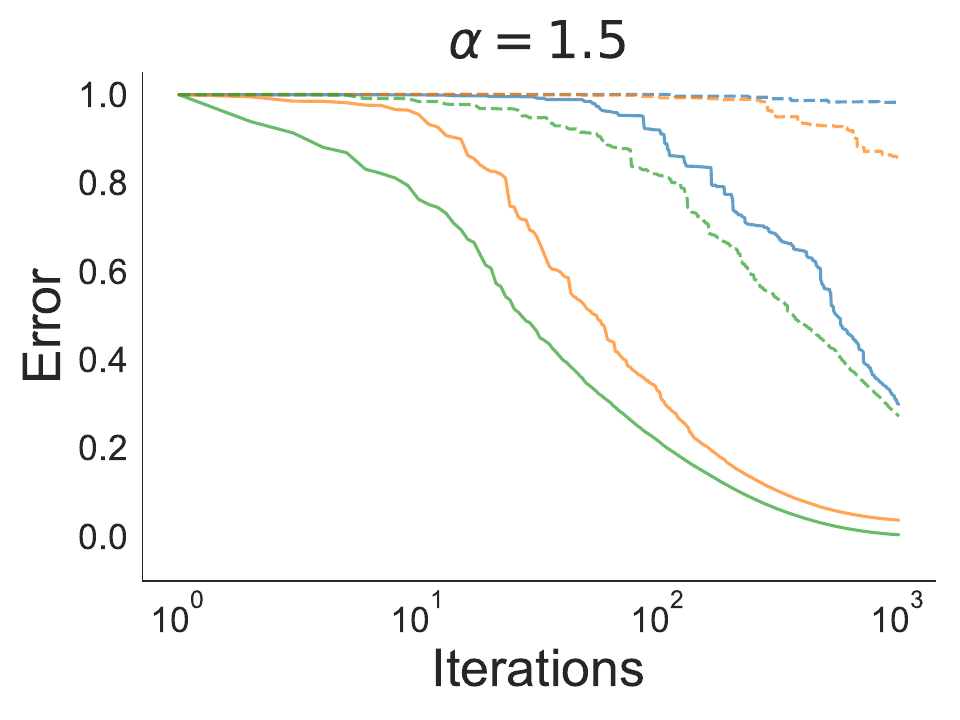}

  \vspace{0.2cm}

  \includegraphics[scale=0.22, trim = {0 0 0 1.1cm}, clip]{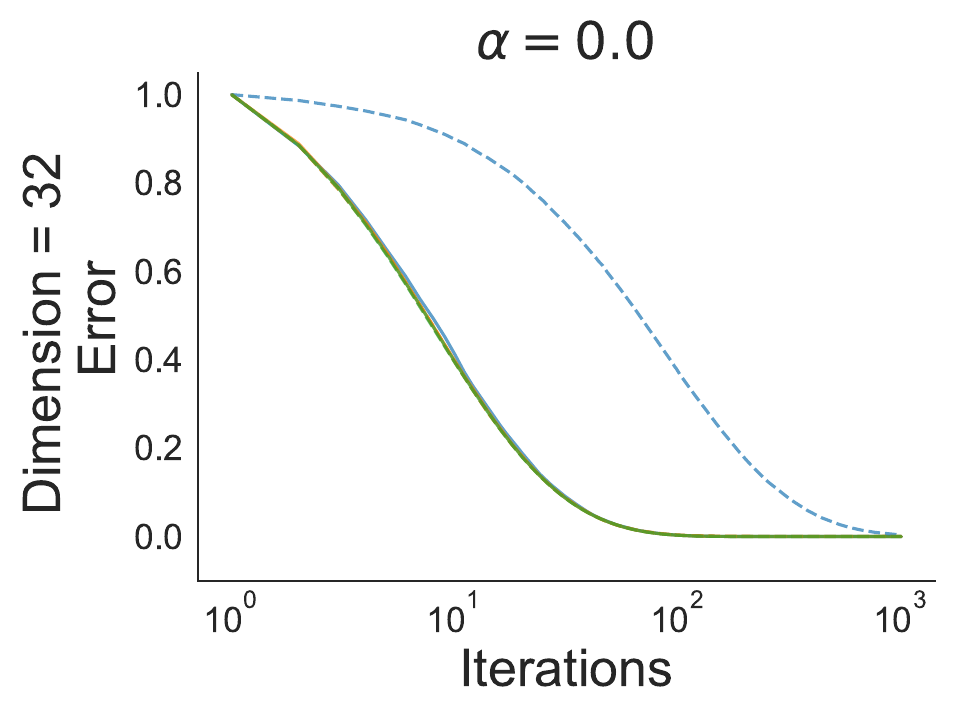}\hfill
  \includegraphics[scale=0.22, trim = {2.31cm 0 0 1.1cm}, clip]{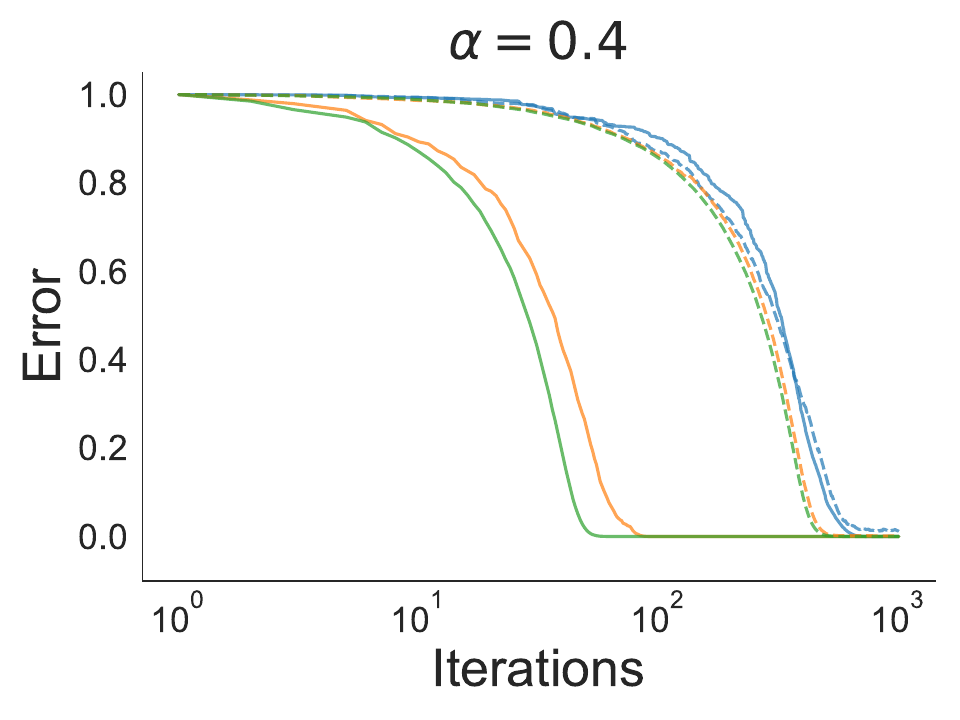}\hfill
  \includegraphics[scale=0.22, trim = {2.31cm 0 0 1.1cm}, clip]{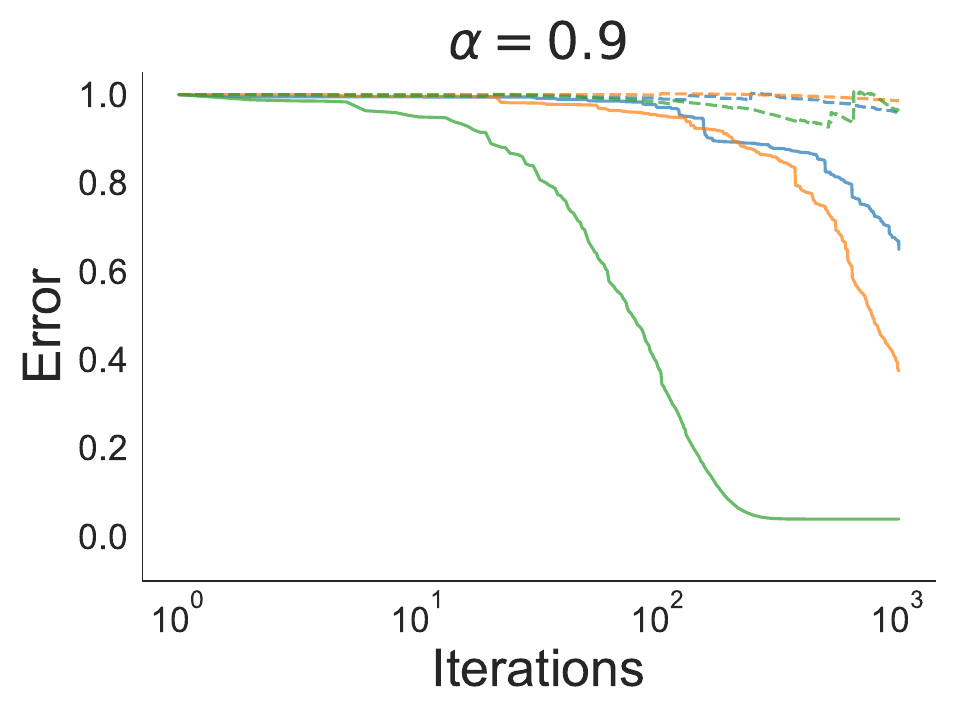}\hfill
  \includegraphics[scale=0.22, trim = {2.31cm 0 0 1.1cm}, clip]{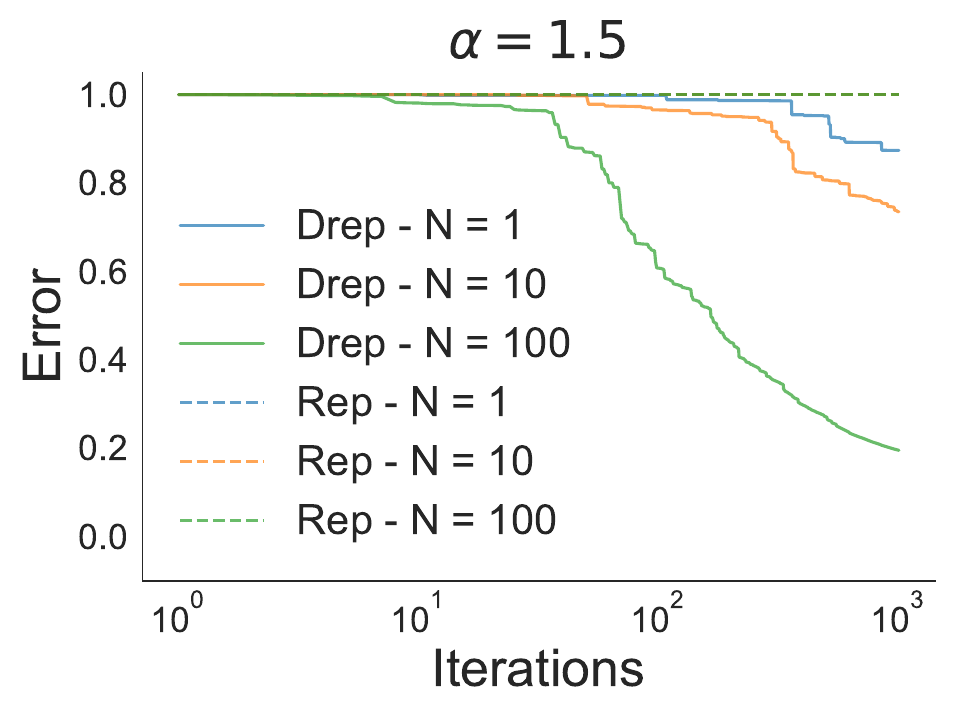}
  
  \caption{\textbf{$\mathbf{\gdrep}$ tends to work significantly better than $\mathbf{\grep}$.} Each plot shows the optimization results for each pair $(d, \alpha)$ for all values of $N$ considered. Plots show ``Error vs. Iteration'', where error is the normalized distance between the scale parameters of $p$ and $q_w$, computed as $(1/ d) \sum_i (\sigma_{qi} - 1)^2$.}
  \label{fig:compare_gauss}
\end{figure}

\clearpage

\section{Results using Adam} \label{app:adam}

In this section we show the same results as in the main paper, but using Adam \cite{adam} to optimize instead of plain SGD. Results for the Gaussian setting described in Section~\ref{sec:motiv} are shown in Fig.~\ref{fig:noopt_gauss_adam}, and results for the Logistic Regression model described in Section~\ref{sec:exps} are shown in Fig.~\ref{fig:exps_lr_opt_adam}. It can be observed that Adam also struggles when using gradient estimators with low SNR, and that for logistic regression with the \textit{australian} dataset ($d=14$) it also requires a huge number of samples to retain convergence (for a moderate value of $\alpha$). Overall, the results obtained using Adam are similar to those obtained using SGD.

\begin{figure*}[ht]
  \centering

  \includegraphics[scale=0.25, trim = {0 2.3cm 0 0}, clip]{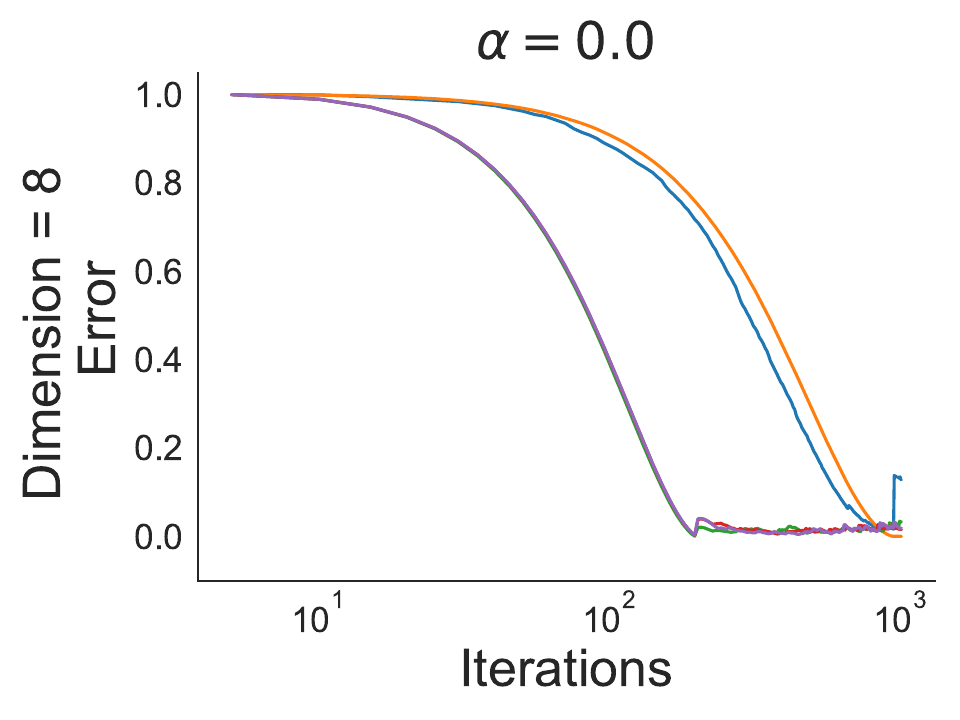}\hfill
  \includegraphics[scale=0.25, trim = {2.31cm 2.3cm 0 0}, clip]{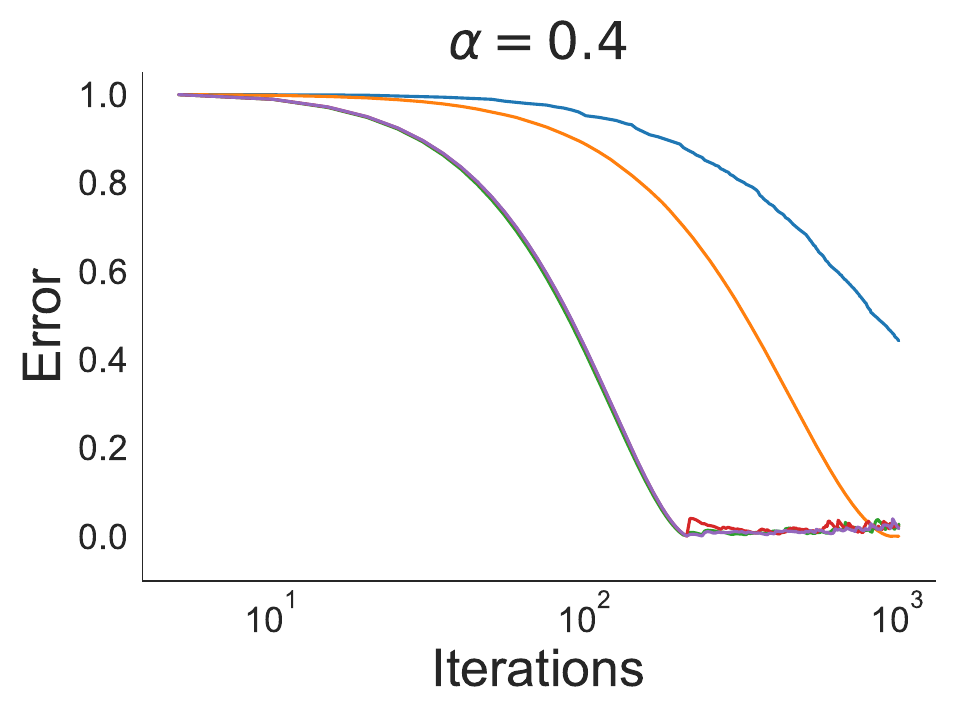}\hfill
  \includegraphics[scale=0.25, trim = {2.31cm 2.3cm 0 0}, clip]{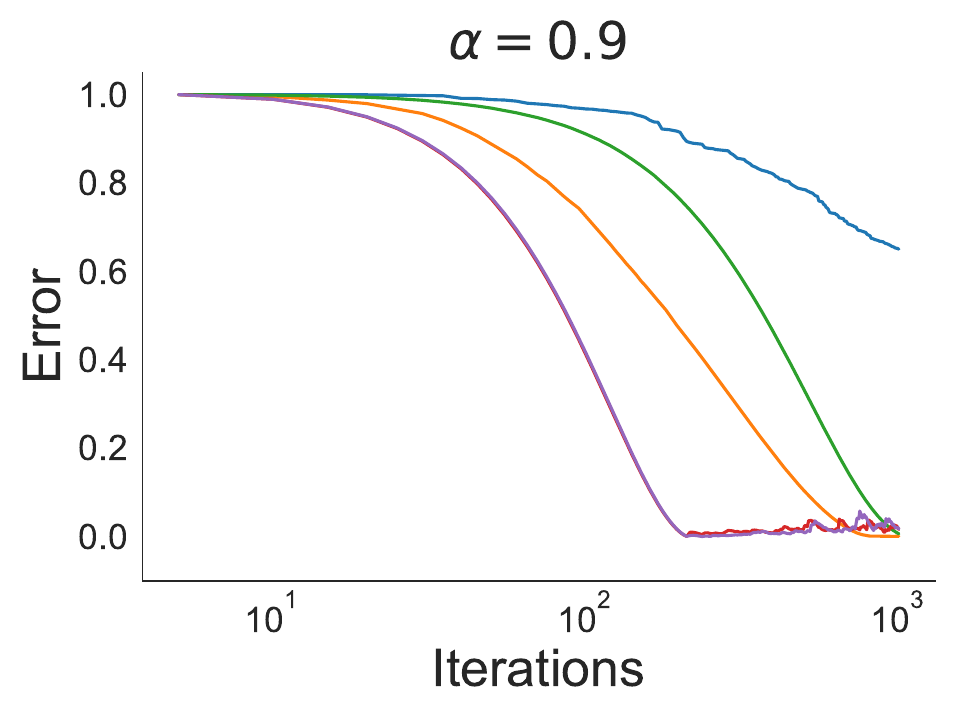}\hfill
  \includegraphics[scale=0.25, trim = {2.31cm 2.3cm 0 0}, clip]{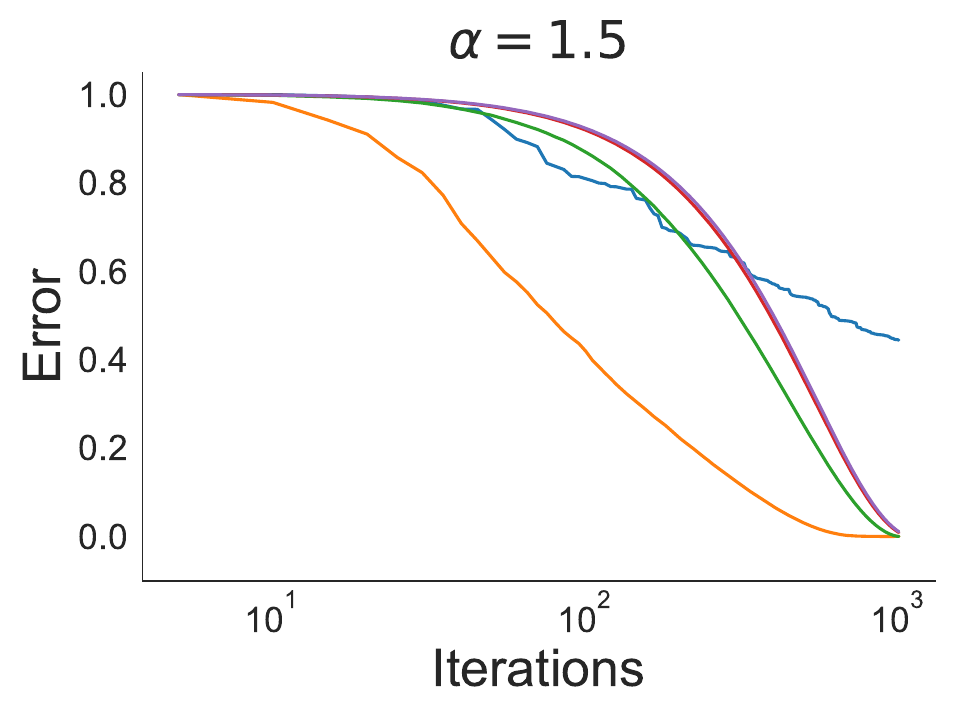}

  \vspace{0.2cm}

  \includegraphics[scale=0.25, trim = {0 2.3cm 0 1.05cm}, clip]{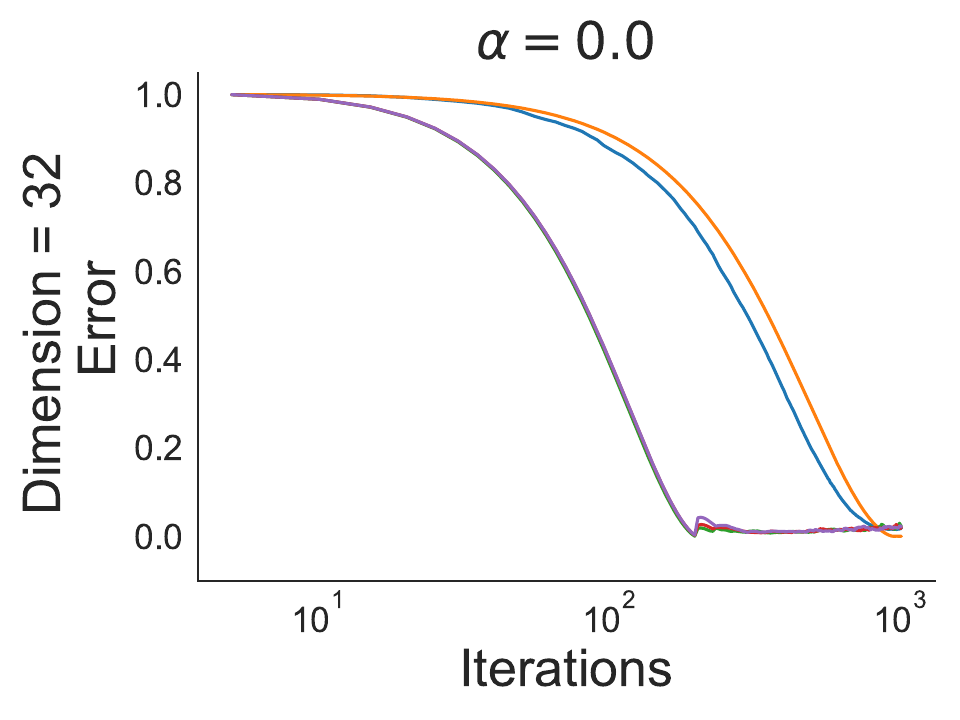}\hfill
  \includegraphics[scale=0.25, trim = {2.31cm 2.3cm 0 1.05cm}, clip]{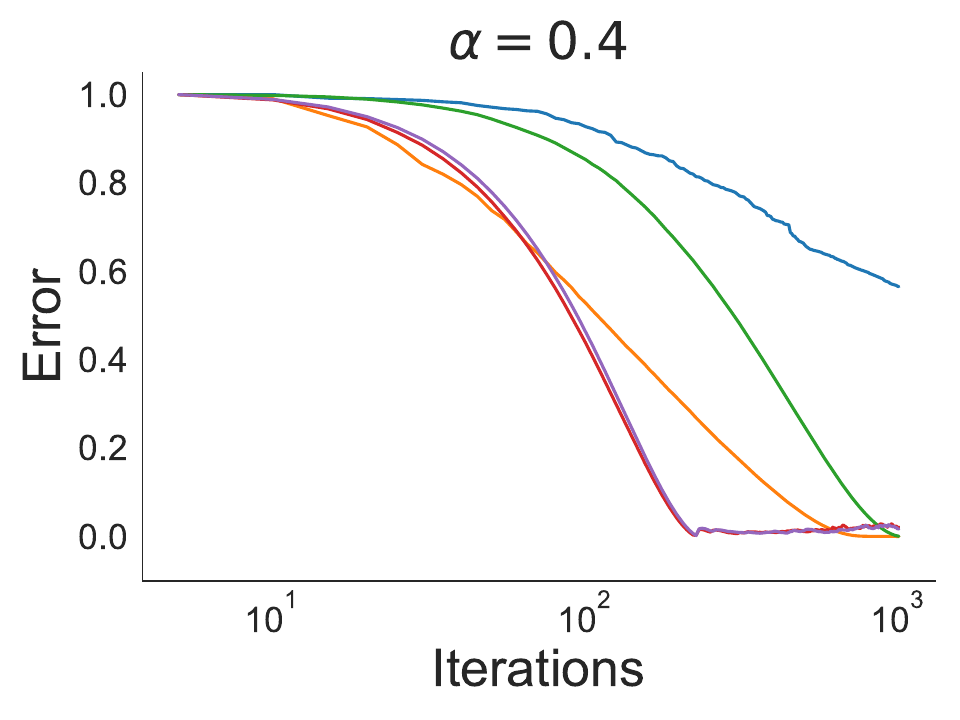}\hfill
  \includegraphics[scale=0.25, trim = {2.31cm 2.3cm 0 1.05cm}, clip]{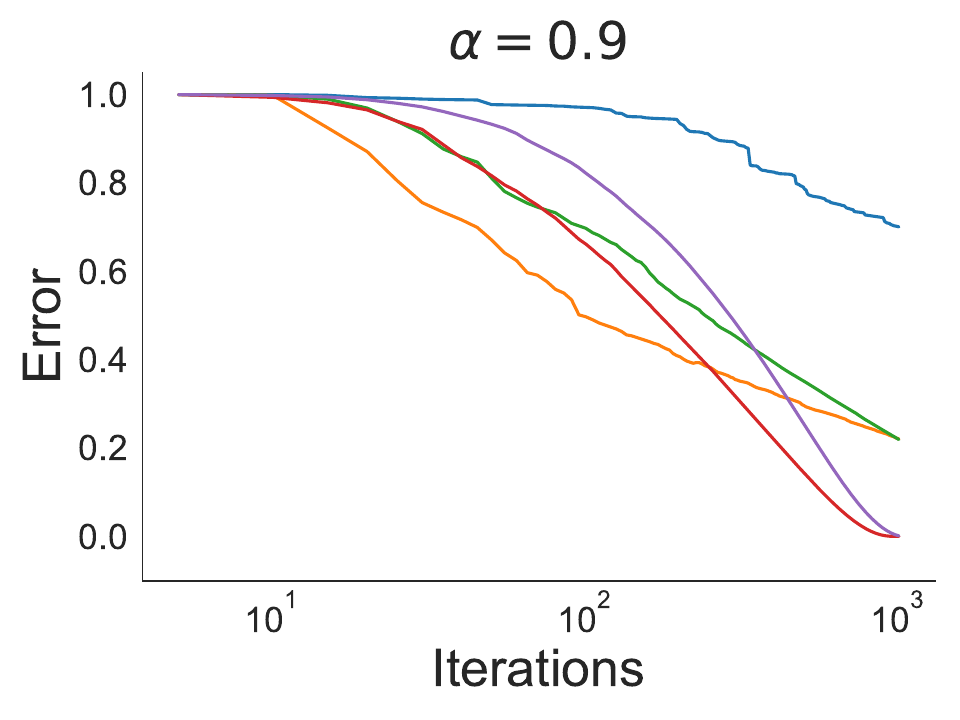}\hfill
  \includegraphics[scale=0.25, trim = {2.31cm 2.3cm 0 1.05cm}, clip]{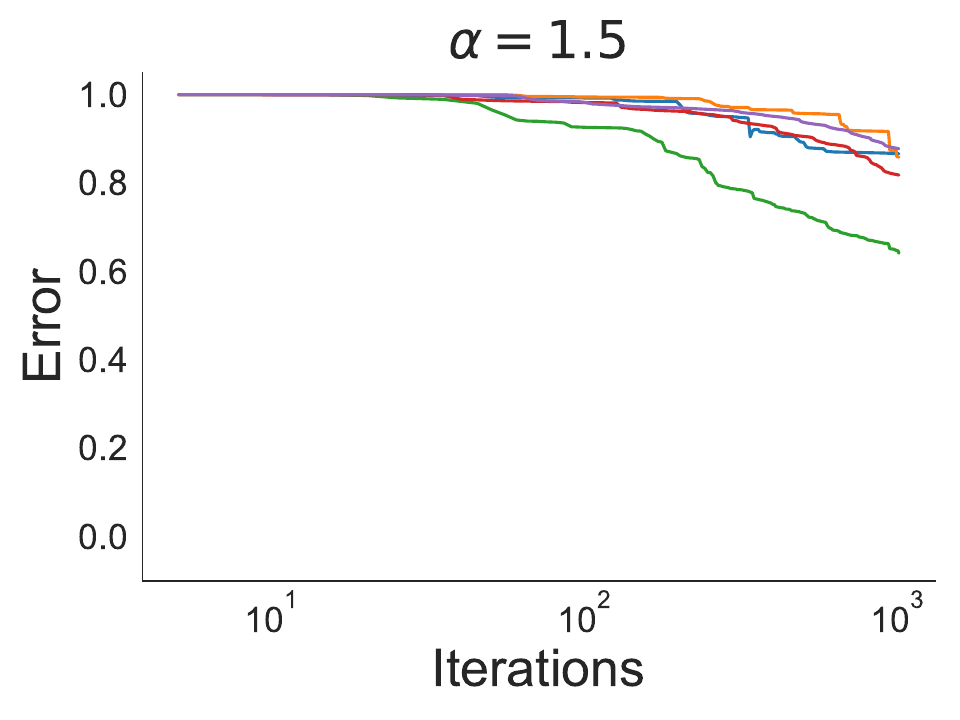}

  \vspace{0.2cm}

  \includegraphics[scale=0.25, trim = {0 0 0 1.1cm}, clip]{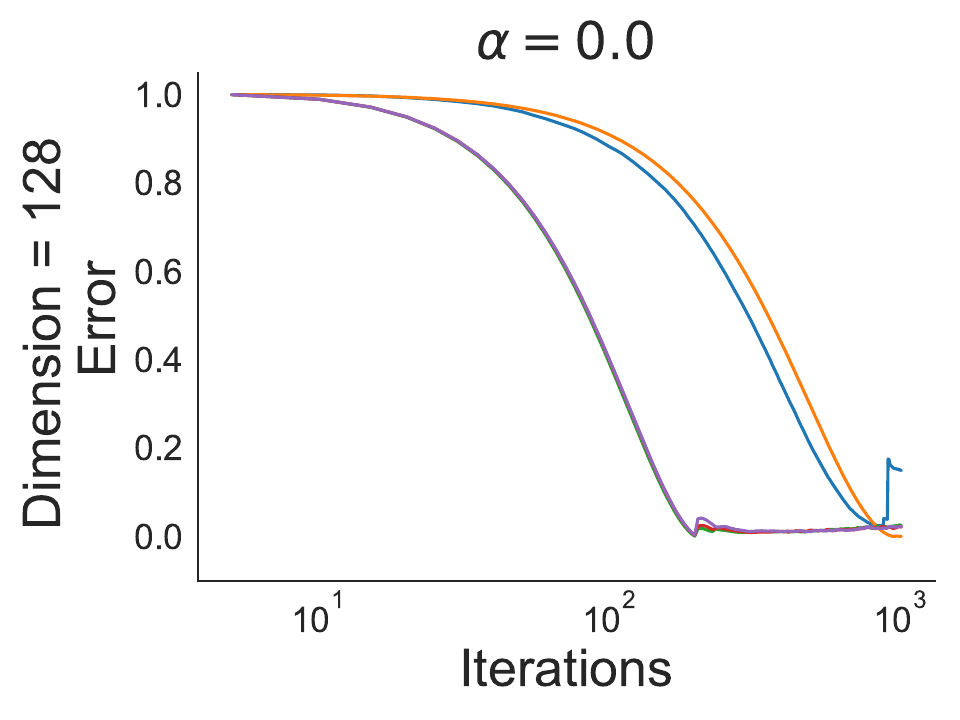}\hfill
  \includegraphics[scale=0.25, trim = {2.31cm 0 0 1.1cm}, clip]{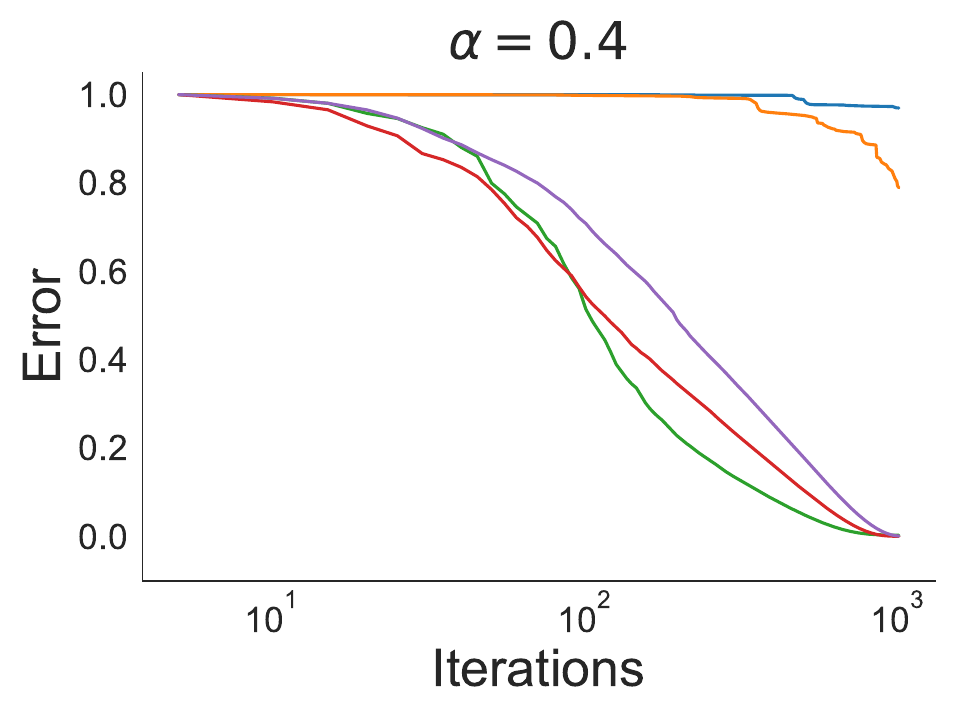}\hfill
  \includegraphics[scale=0.25, trim = {2.31cm 0 0 1.1cm}, clip]{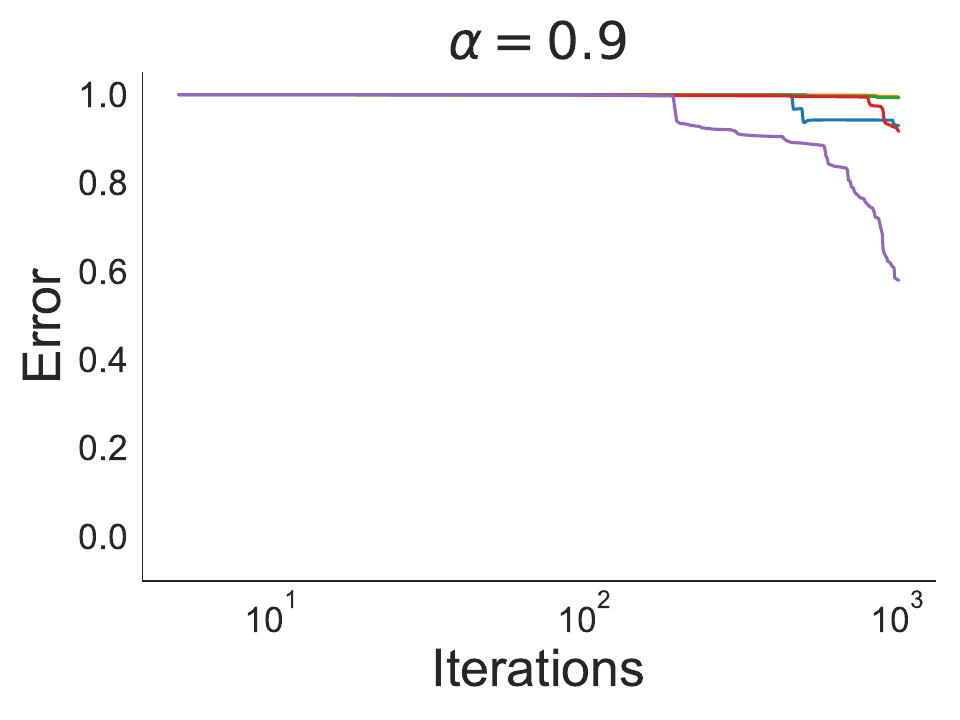}\hfill
  \includegraphics[scale=0.25, trim = {2.31cm 0 0 1.1cm}, clip]{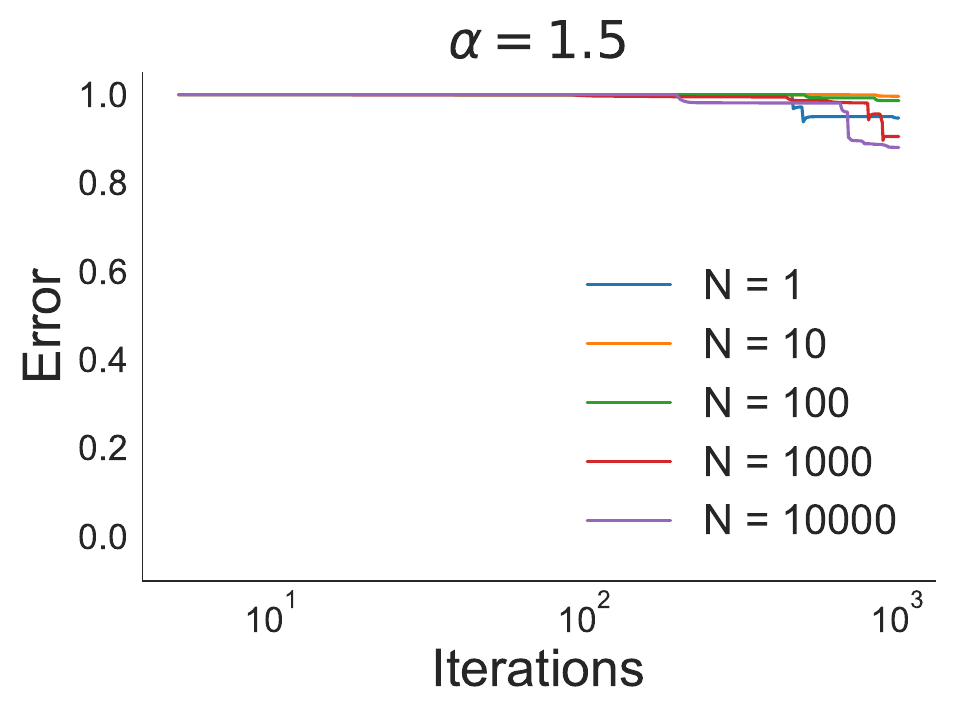}
  
  \caption{Optimization results for each pair $(d, \alpha)$ using Adam for all values of $N$ considered (exact setting described in Section~\ref{sec:motiv}). Plots show ``Error vs. Iteration'', where error is the normalized distance between the scale parameters of $p$ and $q_w$, computed as $(1/ d) \sum_i (\sigma_{qi} - 1)^2$.}
  \label{fig:noopt_gauss_adam}
\end{figure*}

\begin{figure*}[ht]
  \centering
  \includegraphics[scale=0.24, trim = {0 2cm 0 0}, clip]{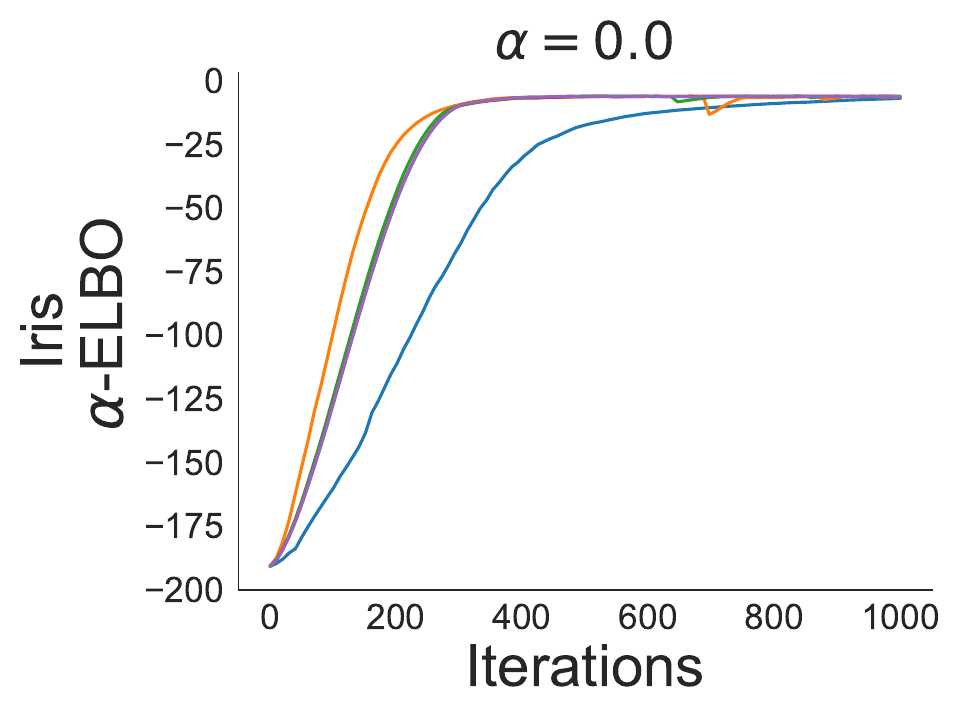}\hfill
  \includegraphics[scale=0.24, trim = {0 2cm 0 0}, clip]{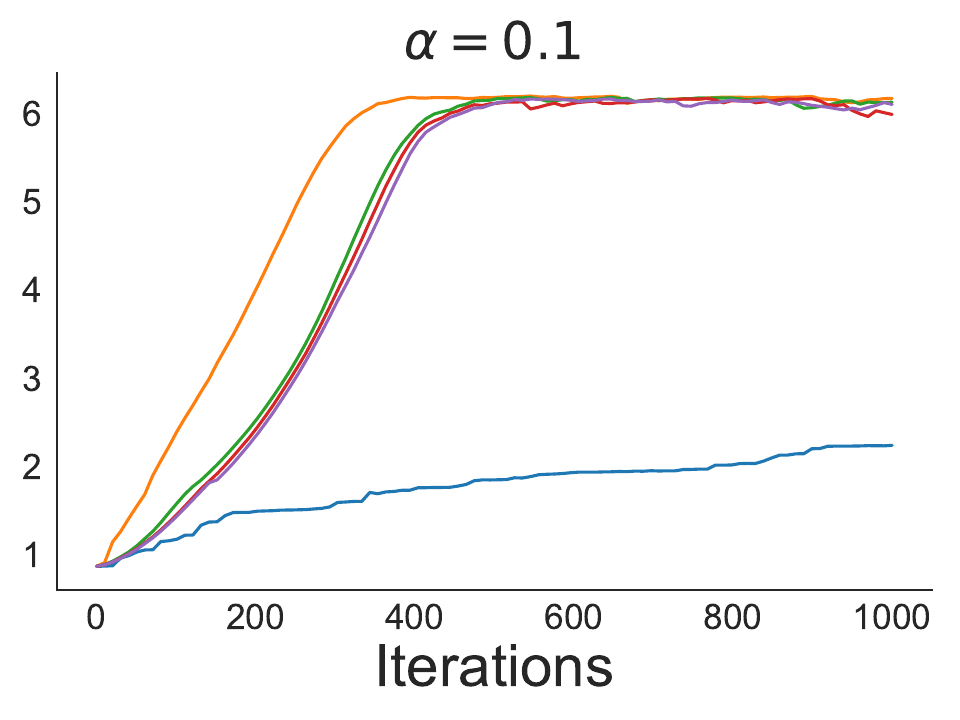}\hfill
  \includegraphics[scale=0.24, trim = {0 2cm 0 0}, clip]{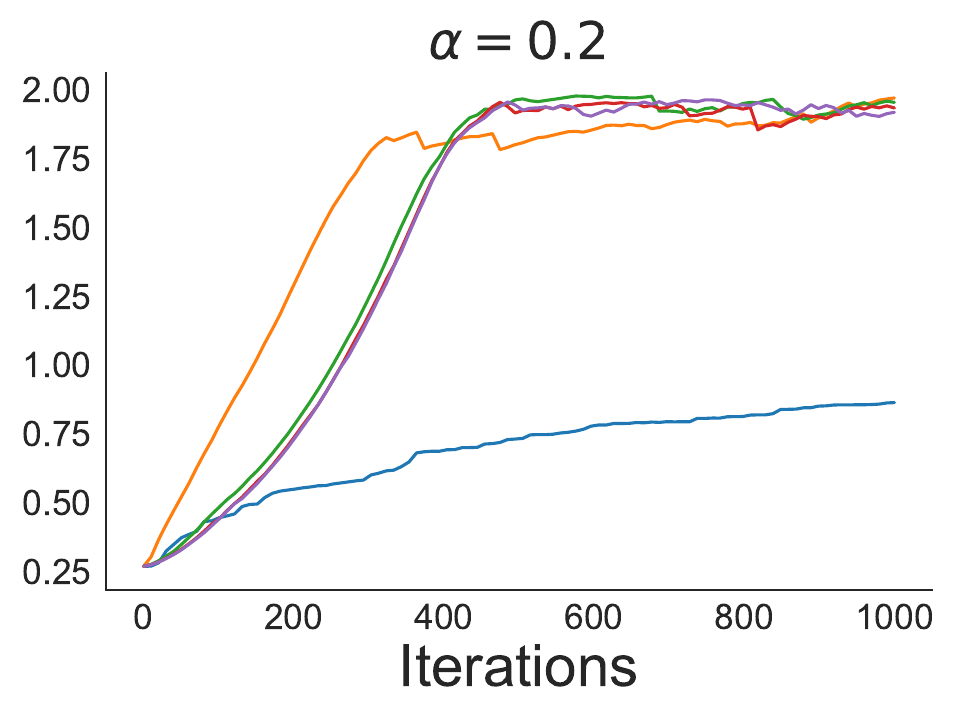}\hfill
  \includegraphics[scale=0.24, trim = {0 2cm 0 0}, clip]{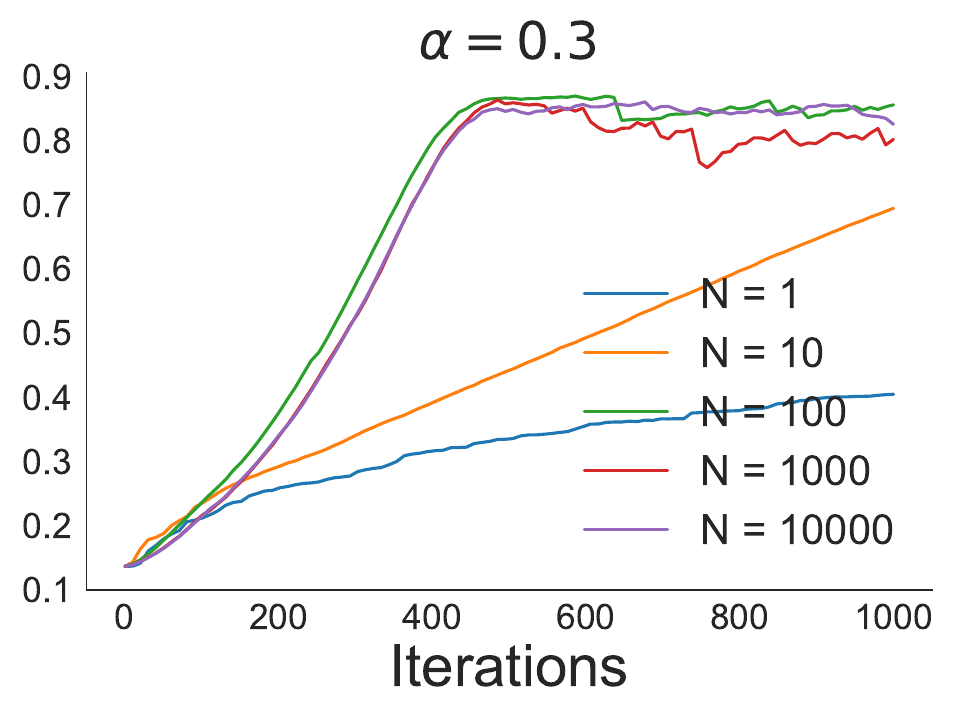}

  \vspace{0.1cm}

  \includegraphics[scale=0.24, trim = {0 0 0 0}, clip]{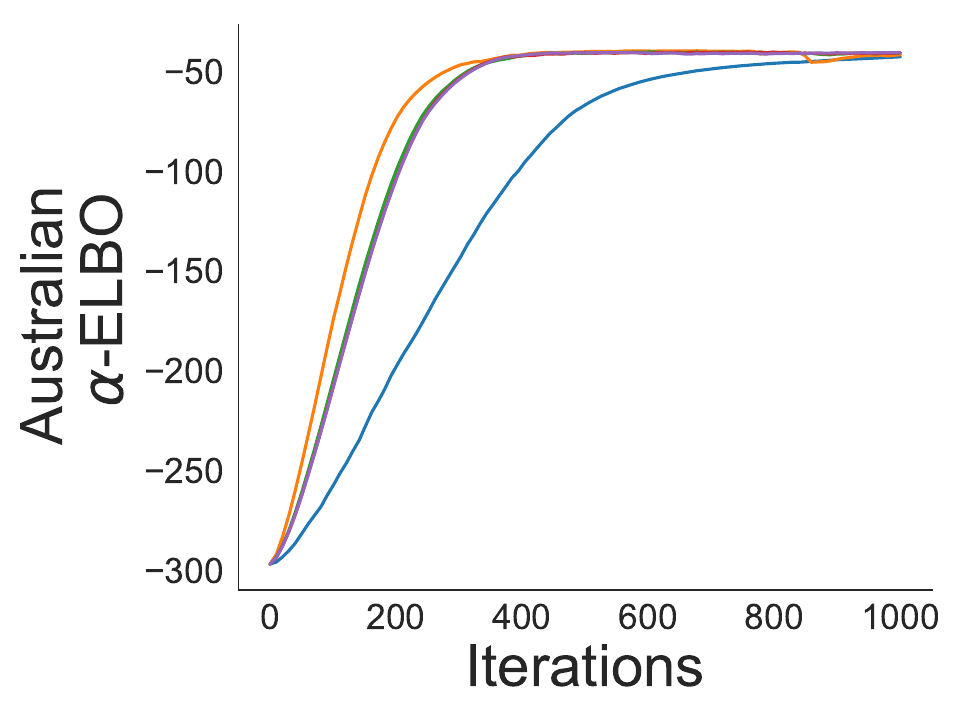}\hfill
  \includegraphics[scale=0.24, trim = {0 0 0 0}, clip]{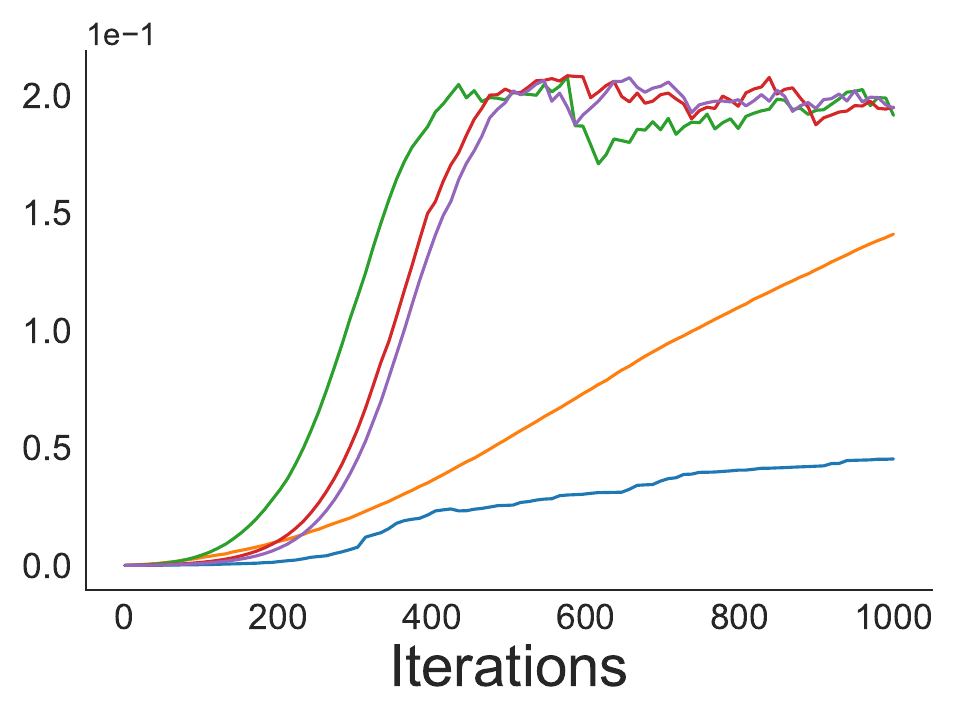}\hfill
  \includegraphics[scale=0.24, trim = {0 0 0 0}, clip]{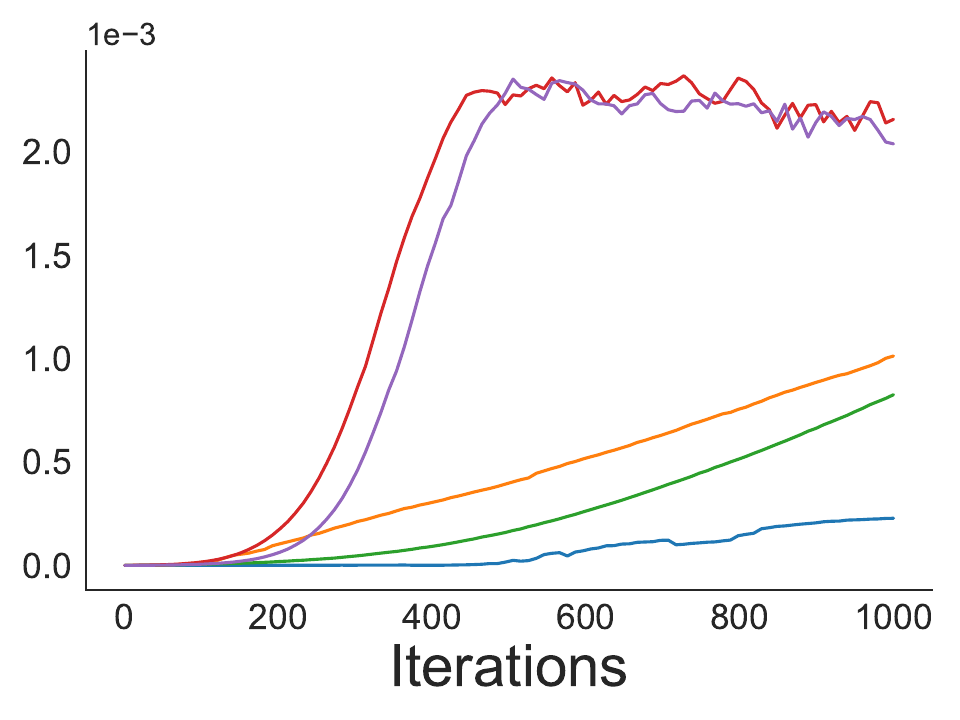}\hfill
  \includegraphics[scale=0.24, trim = {0 0 0 0}, clip]{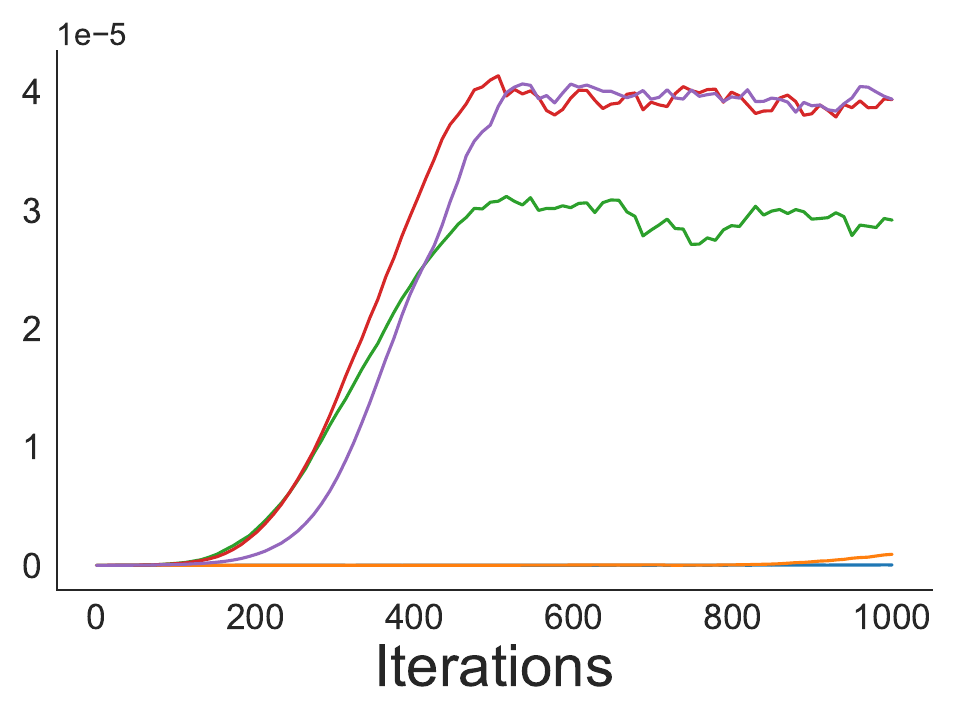}
  \caption{Optimization results using Adam for each $\alpha$ for all values of $N$ considered (exact setting described in Section~\ref{sec:exps}). The loss at each step (eq. \ref{eq:obj_L}) is estimated using $2.5 \times 10^5$ samples for both datasets.}
  \label{fig:exps_lr_opt_adam}
\end{figure*}

\clearpage

\section{Derivation of Estimators} \label{app:estimators}

We re-write the true gradient in different forms for each method. The corresponding estimator is obtained by estimating the resulting expectations with samples $\epsilon \sim q_0(\epsilon)$.

\textbf{Reparameterization.}

\begin{itemize}
\item For $\alpha \neq \{0, 1\}$ we have

\begin{align}
\nabla_w D_\alpha(p||q_w) & = \nabla_w \frac{1}{\alpha (\alpha - 1)}\E_{q_w(z)} \left[ \left( \frac{p(z)}{q_w(z)} \right)^\alpha - 1 \right]\\
& = \nabla_w \frac{1}{\alpha (\alpha - 1)}\E_{q_0(\epsilon)} \left[ \left( \frac{p(\T)}{q_w(\T)} \right)^\alpha - 1 \right]\\
& = \frac{1}{\alpha (\alpha - 1)}\E_{q_0(\epsilon)} \left[ \nabla_w \left( \frac{p(\T)}{q_w(\T)} \right)^\alpha \right].
\end{align}

\item For $\alpha \rightarrow 0$ we know that $D_\alpha(p||q_w) \rightarrow \KL(q_w||p)$, then the gradient can be expressed as

\begin{align}
\nabla_w D_\alpha(p||q_w) & = \nabla_w \E_{q_w(z)} \left[ \log \frac{q_w(z)}{p(z)} \right]\\
 & = \nabla_w \E_{q_0(\epsilon)} \left[ \log \frac{q_w(\T)}{p(\T)} \right]\\
 & = \E_{q_0(\epsilon)} \left[ \nabla_w \log \frac{q_w(\T)}{p(\T)} \right].
\end{align}

\item For $\alpha \rightarrow 1$ we know that $D_\alpha(p||q_w) \rightarrow \KL(p||q_w)$, then the gradient can be expressed as

\begin{align}
\nabla_w D_\alpha(p||q_w) & = \nabla_w \E_{p(z)} \left[ \log \frac{p(z)}{q_w(z)} \right]\\
 & = -\E_{p(z)} \left[ \nabla_w \log q_w(z) \right] \label{eqapp:KLpq}.
\end{align}

The expression from eq. \ref{eqapp:KLpq} is not useful, since we cannot draw samples from $p$. However, we can rewrite it as

\begin{align}
\nabla_w D_\alpha(p||q_w) & = -\E_{p(z)} \left[ \nabla_w \log q_w(z) \right]\\
& = -\E_{p(z)} \left[ \frac{1}{q_v(z)} \nabla_w q_w(z) \right]\\
& = -\nabla_w \E_{q_w(z)} \left[ \frac{p(z)}{q_v(z)} \right]_{v = w}\\
& = -\E_{q_0(\epsilon)} \left[\nabla_w \frac{p(\T)}{q_v(\T)} \right]_{v = w}\label{eqapp:KLpq2}.
\end{align}

\end{itemize}

\textbf{Double reparameterization.}

\begin{itemize}
\item For $\alpha \neq \{0, 1\}$ we have

\begin{align}
& \nabla_w D_\alpha(p||q_w)\\
& = \nabla_w \frac{1}{\alpha (\alpha - 1)} \E_{q_w(z)} \left[ \left( \frac{p(z)}{q_w(z)} \right)^\alpha - 1 \right]\\
& = \underbrace{\nabla_w \frac{1}{\alpha (\alpha - 1)} \E_{q_w(z)} \left[ \left( \frac{p(z)}{q_v(z)} \right)^\alpha\right]_{v = w}}_{A}
  + \underbrace{\nabla_w \frac{1}{\alpha (\alpha - 1)} \E_{q_v(z)} \left[ \left( \frac{p(z)}{q_w(z)} \right)^\alpha\right]_{v = w}}_{B}.
\end{align}

Term $A$ can be expressed as

\begin{equation}
A = \nabla_w \frac{1}{\alpha (\alpha - 1)} \E_{q_0(\epsilon)} \left[ \left( \frac{p(\T)}{q_v(\T)} \right)^\alpha\right]_{v = w}.
\end{equation}

For term $B$ we have

\begin{align}
B & = \nabla_w \frac{1}{\alpha (\alpha - 1)} \E_{q_v(z)} \left[ \left( \frac{p(z)}{q_w(z)} \right)^\alpha\right]_{v = w}\\
& = \frac{1}{\alpha (\alpha - 1)} \E_{q_v(z)} \left[ \nabla_w \left( \frac{p(z)}{q_w(z)} \right)^\alpha\right]_{v = w}\\
& = \frac{-1}{\alpha (\alpha - 1)} \E_{q_v(z)} \left[ \alpha \left( \frac{p(z)}{q_v(z)} \right)^{\alpha - 1} \frac{p(z)}{q_v(z)^2} \nabla_w q_w(z) \right]_{v = w}\\
& = \frac{-1}{\alpha (\alpha - 1)} \nabla_w \E_{q_w(z)} \left[ \alpha \left( \frac{p(z)}{q_v(z)} \right)^{\alpha - 1} \frac{p(z)}{q_v(z)} \right]_{v = w}\\
& = \frac{-1}{\alpha (\alpha - 1)} \nabla_w \E_{q_w(z)} \left[ \alpha \left( \frac{p(z)}{q_v(z)} \right)^{\alpha} \right]_{v = w}\\
& = \frac{-1}{\alpha - 1} \E_{q_0(\epsilon)} \left[ \nabla_w \left( \frac{p(\T)}{q_v(\T)} \right)^{\alpha} \right]_{v = w}.
\end{align}

Combining the results for $A$ and $B$ as $\nabla_w D_\alpha(p||q_w) = A + B$ yields

\begin{equation}
\nabla_w D_\alpha(p||q_w) = \frac{-1}{\alpha} \E_{q_0(\epsilon)} \left[ \nabla_w \left( \frac{p(\T)}{q_v(\T)} \right)^{\alpha} \right]_{v = w} \label{eqapp:drep}.
\end{equation}

\item For $\alpha \rightarrow 0$ we get the "Sticking the landing" estimator \citep{stickingthelanding}

\begin{align}
\nabla_w D_\alpha(p||q_w) & = \nabla_w \E_{q_w(z)} \left[ \log \frac{p(z)}{q_w(z)} \right]\\
& = \nabla_w \E_{q_v(z)} \left[ \log \frac{p(z)}{q_w(z)} \right]_{v = w} + \nabla_w \E_{q_w(z)} \left[ \log \frac{p(z)}{q_v(z)} \right]_{v = w}\\
& = -\E_{q_v(z)} \left[ \nabla_w \log q_w(z) \right]_{v = w} + \nabla_w \E_{q_0(\epsilon)} \left[ \log \frac{p(\T)}{q_v(\T)} \right]_{v = w}\\
& = \nabla_w \E_{q_0(\epsilon)} \left[ \log \frac{p(\T)}{q_v(\T)} \right]_{v = w}.
\end{align}

\item For $\alpha \rightarrow 1$ we just get eq. \ref{eqapp:KLpq2}.

\end{itemize}

\newpage

\section{Proofs}

We now present the proofs of the three SNR theorems in the paper. All theorems are proved separately for $\alpha \neq 0$ and $\alpha \rightarrow 0$, starting from the corresponding expression for the gradient estimator $\gdrep$ from eq. \ref{eq:gdrep}. However, all of these results could also be obtained by computing the SNR for $\alpha \neq 0$ and taking the limit when $\alpha \rightarrow 0$ of the resulting expression.

\subsection{Properties of the function $\f$} \label{app:fproof}

\begin{lemma} \label{lem:f-props}
Restricted to values of $\alpha$ and $\R$ that satisfy $1 + 2\alpha(\lambda - 1) > 0$, the function
\begin{equation}
\f = \frac{1}{\sqrt{1 + \alpha^2 \frac{(\R - 1)^2}{1 + 2\alpha \R - 2\alpha}}}
\end{equation}
has the following properties: (i) it achieves its maximum value of 1 if and only if $\R = 1$ or $\alpha = 0$; (ii) it is quasi-concave in $\R$ for a fixed $\alpha$; (iii) it is quasi-concave in $\alpha$ for a fixed $\R$.
\end{lemma}

\begin{proof}
We have \small$\f = \frac{1}{\sqrt{1 + \alpha^2 \frac{(\R - 1)^2}{1 + 2\alpha \R - 2\alpha}}}$\normalsize. Let C1 be the condition $1+2\alpha\R -2\alpha > 0$ and $\g = \alpha^2 \frac{(\R - 1)^2}{1 + 2\alpha \R - 2\alpha}$ (i.e. $\f = \sqrt{\frac{1}{1 + \g}}$). 

\begin{enumerate}
    \item C1 implies $1 + \alpha \R - \alpha > 0$. We will call this condition C2. This can be seen by noticing that C2 is equivalent to $\alpha(\R - 1) > -1$, while C1 is equivalent to $\alpha (\R - 1) > -1/2$. Thus, a pair $(\R, \alpha)$ that satisfies C1 must satisfy C2.
    \item $\f \leq 1$. This can be easily seen by noticing that C1 implies $\g \geq 0$.
    \item $\f = 1$ if $\R = 1$ or $\alpha = 0$. This is straightforward to check.
    \item $\f$ is quasi-concave in $\R$ for a fixed $\alpha$. We show this by proving that $\g$ is quasi-convex in $\R$ for a fixed $\alpha$. To do so we compute
    
    \begin{equation}
    \frac{\partial \g}{\partial \R} =  2 \alpha^2 \frac{(\R - 1) (1 + \alpha \R - \alpha)}{\left(1 + 2\alpha \R - 2\alpha \right)^2}.
    \end{equation}
    
    The denominator is strictly positive by C1, and $1 + \alpha \R - \alpha$ is also strictly positive by C2. Thus, because of the term $(\R - 1)$, this derivative is strictly negative for $\R < 1$ and strictly positive for $\R > 1$. Thus, $\g$ is decreasing for $\R < 1$ and increasing for $\R > 1$. Thus, $\f$ is increasing for $\R < 1$ and decreasing for $\R > 1$.
    
    \item $\f$ is quasi-concave in $\alpha$ for a fixed $\R$. We show this by proving that $\g$ is quasi-convex in $\alpha$ for a fixed $\R$. To do so we compute
    
    $$\frac{\partial \g}{\partial \alpha} =  2 \alpha \frac{(\R - 1)^2 (1 + \alpha \R - \alpha)}{\left(1 + 2\alpha \R - 2\alpha \right)^2}.$$
    
    The denominator is strictly positive by C1, and $1 + \alpha \R - \alpha$ is also strictly positive by C2. Thus, this derivative is strictly negative for $\alpha < 0$ and strictly positive for $\alpha > 0$. Thus, $\g$ is decreasing for $\alpha < 0$ and increasing for $\alpha > 0$. Thus, $\f$ is increasing for $\alpha < 0$ and decreasing for $\alpha > 0$.
    
    \item Finally, it can be observed that the derivatives are zero if and only if $\R = 1$ or $\alpha = 0$. Thus, we can conclude that the maximum value of 1 is achieved if and only if $\R = 1$ or $\alpha = 0$.
\end{enumerate}
\end{proof}

\subsection{General Fully-factorized Distributions}

\fullyfactgeneral*

\begin{proof}
We prove this results for $\ga = \gdrep$. The proof for $\ga = \grep$ follows the same steps.

We begin by proving the result for $\alpha \rightarrow 0$. In this case the gradient estimator is given by 

\begin{equation}
\gaj(p, q_w, \epsilon) = \nabla_{w_j} \log \frac{p(\mathcal{T}_w(\epsilon))}{q_v(\mathcal{T}_w(\epsilon))}.
\end{equation}

Since $p$ and $q$ are factorized

\begin{align}
\gaj(p, q_w, \epsilon) & = \left . \nabla_{w_j} \log \frac{p(\mathcal{T}_w(\epsilon))}{q_v(\mathcal{T}_w(\epsilon))} \right|_{v=w}\\
& = \left . \nabla_{w_j} \log \frac{p_j(\mathcal{T}_{w_j}(\epsilon_j))}{q_{v_j}(\mathcal{T}_{w_j}(\epsilon_j))}\right|_{v=w}\\
& = \ga(p_j, q_{w_j}, \epsilon_j).
\end{align}

Thus, for $\alpha \rightarrow 0$, 

\begin{equation}
\SNR[\gaj(p, q_w, \epsilon)] = \SNR[\ga(p_j, q_{w_j}, \epsilon_j)].
\end{equation}

For $\alpha \neq 0$, since $p$ and $q$ are factorized, the gradient estimator is given by

\begin{align}
\gaj(p, q_w, \epsilon) & = \left . -\frac{1}{\alpha} \nabla_{w_j} \left( \frac{p(\mathcal{T}_w(\epsilon))}{q_v(\mathcal{T}_w(\epsilon))} \right)^\alpha \right|_{v=w}\\
& = \left. -\frac{1}{\alpha} \nabla_{w_j} \left(\frac{p_j(\Tj)}{q_{v_j}(\Tj)}\right)^\alpha \prod_{i \in \{1, ..., d\} \backslash j} \left(\frac{p_i(\Ti)}{q_{v_i}(\Ti)}\right)^\alpha \right|_{v=w}.
\end{align}

The SNR of $\gaj(p, q_w, \epsilon)$ is given by $\nicefrac{\E[\gaj(p, q_w, \epsilon)]^2}{\E[\gaj(p, q_w, \epsilon)^2]}$. We compute these quantities as

\begin{align}
\E[\gaj]^2 & = \E \left[ \frac{-1}{\alpha} \nabla_{w_j} \left(\frac{p_j(\Tj)}{q_{v_j}(\Tj)}\right)^\alpha \prod_{i \in \{1, ..., d\} \backslash j} \left(\frac{p_i(\Ti)}{q_{v_i}(\Ti)}\right)^\alpha \right]^2_{v=w}\\
& = \E \left[ \frac{-1}{\alpha} \nabla_{w_j} \left(\frac{p_j(\Tj)}{q_{v_j}(\Tj)}\right)^\alpha\right]^2 \prod_{i \in \{1, ..., d\} \backslash j} \E \left[\left(\frac{p_i(\Ti)}{q_{v_i}(\Ti)}\right)^\alpha \right]^2_{v=w} \label{eq:lv1}\\
\E[\gaj^2] & = \E \left[ \left( \frac{-1}{\alpha} \nabla_{w_j} \left(\frac{p_j(\Tj)}{q_{v_j}(\Tj)}\right)^\alpha \prod_{i \in \{1, ..., d\} \backslash j} \left(\frac{p_i(\Ti)}{q_{v_i}(\Ti)}\right)^\alpha\right)^2 \right]_{v=w}\\
& = \E \left[ \left( \frac{-1}{\alpha} \nabla_{w_j} \left(\frac{p_j(\Tj)}{q_{v_j}(\Tj)}\right)^\alpha\right)^2\right] \prod_{i \in \{1, ..., d\} \backslash j} \E \left[\left(\frac{p_i(\Ti)}{q_{v_i}(\Ti)}\right)^{2\alpha} \right]_{v=w} \label{eq:lv2}.
\end{align}

Finally,

\begin{align}
\SNR[\gaj(p, q_w, \epsilon)] & = \frac{\E[\gaj(p, q_w, \epsilon)]^2}{\E[\gaj(p, q_w, \epsilon)^2]}\\
& = \frac{\E \left[ \frac{-1}{\alpha} \nabla_{w_j} \left(\frac{p_j(\Tj)}{q_{v_j}(\Tj)}\right)^\alpha\right]^2}{\E \left[ \left(\frac{-1}{\alpha}\nabla_{w_j} \left(\frac{p_j(\Tj)}{q_{v_j}(\Tj)}\right)^\alpha\right)^2\right]} \prod_{i \in \{1, ..., d\} \backslash j} \frac{\E \left[\left( \frac{p_i(\Ti)}{q_{v_i}(\Ti)}\right)^\alpha \right]^2}{\E \left[\left(\frac{p_i(\Ti)}{q_{v_i}(\Ti)}\right)^{2\alpha} \right]} \label{eq:here1}\\
& = \SNR[\gaj(p_j, q_{w_j}, \epsilon_j)] \prod_{i \in \{1, ..., d\} \backslash j} \SNR\left[\left(\frac{p_i(\Ti)}{q_{v_i}(\Ti)}\right)^\alpha\right].
\end{align}

Replacing the definition of $\tilde D_\alpha$ yields the desired result for $\alpha \neq 0$.

To prove the fact that the $j$-th component of the estimator is deterministically zero if $p_j = q_{w_j}$ it suffices to observe that, if $p_j = q_{w_j}$, then

\begin{align}
\left . \nabla_{w_j} \frac{p_j(\T)}{q_{v_j}(\T)} \right|_{v = w} & = \left . \nabla_{w_j} \frac{q_{v_j}(\T)}{q_{v_j}(\T)} \right|_{v = w} = \nabla_{w_j} 1 = 0.
\end{align}
\end{proof}

\newpage

\subsection{Fully-factorized Gaussians}

In this section we prove our main result for the case where both $p$ and $q_w$ are mean-zero fully-factorized Gaussians. We begin by proving some auxiliary lemmas for the one-dimensional case, which we then use to prove the main SNR result for $d$-dimensional factorized Gaussians. In this case the set of parameters $w$ is given by the scales of each dimension of $q_w$, $w = \{\sigma_{w1}, ..., \sigma_{wd}\}$.

\subsubsection{Factorized Gaussians - Auxiliary Lemmas}

\begin{lemma} \label{app:lemmaffg1}
Let $p$ and $q_w$ be one dimensional Gaussian distributions with mean zero and scales $\sigma_p$ and $\sw$, respectively. Then, for $\ga = \gdrep$ and $\alpha \rightarrow 0$,
\begin{itemize}[leftmargin=*] \denselist
\item If $\sigma_{p} = \sw$, the gradient estimator is deterministically zero.
\item Otherwise,
\begin{align}
\SNR[\ga(p, q_w, \epsilon)] = \frac{1}{3}.
\end{align}
\end{itemize}
\end{lemma}

\begin{proof}
Using the fact that $\T = \sw \epsilon$ and the pdf of Gaussian distributions we can compute the gradient estimator for $\alpha \rightarrow 0$ as

\begin{align}
\gaj(p, q_w, \epsilon) & = \left . \nabla_w \log \frac{q_v(\T)}{p(\T)} \right|_{v=w}\\
& = \left . \nabla_w \left(\log \frac{\sv}{\sigma_p} - \frac{1}{2} \sw^2 \epsilon^2 \left(\frac{1}{\sv^2} - \frac{1}{\sigma_p^2}\right)\right)\right|_{v=w}\\
& = -\sw \epsilon^2 \left(\frac{1}{\sw^2} - \frac{1}{\sigma_p^2}\right).
\end{align}

If $\sigma_p = \sw$ the gradient estimator is deterministically zero. Otherwise, using the fact that $\epsilon \sim \mathcal{N}(0, 1)$, we can compute
\begin{align}
\E[\gaj(p, q_w, \epsilon)]^2 & = \sw^2 \left(\frac{1}{\sw^2} - \frac{1}{\sigma_p^2}\right)^2\\
\E[\gaj(p, q_w, \epsilon)^2] & = 3\sw^2 \left(\frac{1}{\sw^2} - \frac{1}{\sigma_p^2}\right)^2.
\end{align}

Thus, if $\sw \neq \sigma_p$, we get $\SNR[\ga(p, q_w, \epsilon)] = \frac{1}{3}$.
\end{proof}

\begin{lemma} \label{app:lemmaf}
Let $p$ and $q_w$ be one dimensional Gaussian distributions with mean zero and scales $\sigma_p$ and $\sw$, respectively. Let $\R = \nicefrac{\sw^2}{\sigma_p^2}$. Then, for $\ga = \gdrep$ and $\alpha \neq 0$,

\begin{itemize}[leftmargin=*] \denselist
\item If $1 + 2\alpha \R - 2\alpha > 0$, then
\begin{align}
\SNR\left[\left(\frac{p(\T)}{q_w(\T)}\right)^\alpha\right] = f(\R, \alpha),
\end{align}

where \small{$f(\R, \alpha) = \nicefrac{1}{\sqrt{1 + \alpha^2 \frac{(\R - 1)^2}{1 + 2\alpha \R - 2\alpha}}}$}\normalsize.
\item If $1 + 2\alpha \R - 2\alpha \leq 0$, the SNR is not defined.
\end{itemize}
\end{lemma}

\begin{proof}
Using the pdf of a univariate Gaussian, we can compute

\begin{equation}
\left(\frac{p(\T)}{q_w(\T)}\right)^\alpha = \left(\frac{\sw}{\sigma_p}\right)^\alpha \exp\left[-\frac{\alpha}{2} \sw^2 \epsilon^2 \left(\frac{1}{\sigma_p^2} - \frac{1}{\sw^2} \right) \right].
\end{equation}

To compute the SNR we need to compute $\E\left[ \left(\frac{p(\T)}{q_w(\T)}\right)^\alpha \right]^2$ and $\E\left[ \left(\frac{p(\T)}{q_w(\T)}\right)^{2\alpha} \right]$. The former can be computed as 

\begin{align}
\E\left[ \left(\frac{p(\T)}{q_w(\T)}\right)^\alpha \right] & = \left(\frac{\sw}{\sigma_p}\right)^\alpha \int \frac{1}{\sqrt{2\pi}} \exp\left[-\frac{\epsilon^2}{2} \right] \exp\left[-\frac{\alpha}{2} \sw^2 \epsilon^2 \left(\frac{1}{\sigma_p^2} - \frac{1}{\sw^2} \right) \right] d\epsilon\\
& = \left(\frac{\sw}{\sigma_p}\right)^\alpha \int \frac{1}{\sqrt{2\pi}}\exp\left[-\frac{1}{2} \epsilon^2 \left(1 + \alpha\frac{\sw^2}{\sigma_p^2} - \alpha \right) \right] d\epsilon\\
& = \R^{\alpha / 2} \int \frac{1}{\sqrt{2\pi}} \exp\left[-\frac{1}{2} \epsilon^2 \left(1 + \alpha \R - \alpha \right) \right] d\epsilon.
\end{align}

This integral converges if and only if $1 + \alpha \R - \alpha > 0$. If this is not satisfied, the quantity does not have an expectation and the SNR is not defined. If the condition is satisfied, then

\begin{align}
\E\left[ \left(\frac{p(\T)}{q_w(\T)}\right)^\alpha \right] & = \frac{\R^{\alpha / 2}}{\sqrt{1 + \alpha \R - \alpha}} \label{eqapp:here2}.
\end{align}

With this result we get
\begin{align}
\E\left[ \left(\frac{p(\T)}{q_w(\T)}\right)^\alpha \right]^2 & = \frac{\R^{\alpha}}{1 + \alpha \R - \alpha}\\
\E\left[ \left(\frac{p(\T)}{q_w(\T)}\right)^{2\alpha} \right] & = \frac{\R^{\alpha}}{\sqrt{1 + 2\alpha \R - 2\alpha}} \label{eqapp:here3},
\end{align}

where eq. \ref{eqapp:here3} is obtained by replacing $\alpha$ by $2\alpha$ in eq. \ref{eqapp:here2} (notice that the condition for existence now becomes $1 + 2\alpha \R - 2\alpha > 0$). Finally, we can compute the SNR by taking the ratio between the above quantities, which yields

\begin{align}
\SNR\left[\left(\frac{p(\T)}{q_w(\T)}\right)^\alpha\right] & = \frac{\sqrt{1 - 2\alpha + 2\alpha \R}}{1 - \alpha + \alpha \R}\\
& = \sqrt{\frac{1}{ \frac{(1 + \alpha \R - \alpha)^2}{1 + 2\alpha \R - 2\alpha} }}\\
& = \sqrt{\frac{1}{ \frac{ 1 + \alpha^2 \R^2 + \alpha^2 + 2\alpha \R -2\alpha - 2\alpha^2 \R}{1 + 2\alpha \R - 2\alpha} }}\\
& = \sqrt{\frac{1}{ 1 + \alpha^2 \frac{(\R - 1)^2}{1 + 2\alpha \R - 2\alpha}}}\\
& = f(\R, \alpha).
\end{align}

Finally, the condition $1 + \alpha \R - \alpha > 0$ does not show up in the theorem statement because it is less restrictive than the condition $1 + 2\alpha \R - 2\alpha > 0$ (i.e. the latter is a sufficient but not necessary condition for the former). To see this, note that the former is equivalent to $\alpha (\R - 1) > -1$, while the latter is equivalent to $\alpha (\R - 1) > -1/2$. Thus, any pair $(\R, \alpha)$ that satisfies $1 + 2\alpha \R - 2\alpha > 0$ also satisfies $1 + \alpha \R - \alpha > 0$. The inverse is not true.
\end{proof}

\begin{lemma} \label{app:lemmaffg3}
Let $p$ and $q_w$ be one dimensional Gaussian distributions with mean zero and scales $\sigma_p$ and $\sw$, respectively. Let $\R = \nicefrac{\sw^2}{\sigma_p^2}$. Then, for $\ga = \gdrep$ and $\alpha \neq 0$,

\begin{itemize}[leftmargin=*] \denselist
\item If $\sigma_p = \sw$, the gradient estimator is deterministically zero.
\item If $1 + 2\alpha \R - 2\alpha \leq 0$, the gradient estimator's SNR is not defined.
\item Otherwise,
\begin{align}
\SNR[\ga(p, q_w, \epsilon)] = \frac{1 + 2\alpha \R - 2\alpha}{3} f(\R, \alpha)^3,
\end{align}
\end{itemize}

where \small{$f(\R, \alpha) = \nicefrac{1}{\sqrt{1 + \alpha^2 \frac{(\R - 1)^2}{1 + 2\alpha \R - 2\alpha}}}$}\normalsize.
\end{lemma}

\begin{proof}
For $\alpha \neq 0$ we have $\ga(p, q_w, \epsilon) = -\frac{1}{\alpha} \nabla_w \left( \frac{p(\mathcal{T}_w(\epsilon))}{q_v(\mathcal{T}_w(\epsilon))} \right)^\alpha$. We can get an expression for the gradient estimate by first computing 

\begin{equation}
\left(\frac{p(\T)}{q_v(\T)}\right)^\alpha = \left(\frac{\sv}{\sigma_p}\right)^\alpha \exp\left[-\frac{\alpha}{2} \sw^2 \epsilon^2 \left(\frac{1}{\sigma_p^2} - \frac{1}{\sv^2} \right) \right].
\end{equation}

Thus,

\begin{equation}
\ga(p, q_w, \epsilon) = \left(\frac{\sw}{\sigma_p}\right)^\alpha \exp\left[-\frac{\alpha}{2} \sw^2 \epsilon^2 \left(\frac{1}{\sigma_p^2} - \frac{1}{\sw^2} \right) \right] \epsilon^2 \sw \left(\frac{1}{\sigma_p^2} - \frac{1}{\sv^2} \right) \label{eqapp:here4}.
\end{equation}

It can be observed that, if $\sw = \sigma_p$, the gradient estimator is deterministically zero. Otherwise, to compute the SNR we need to compute $\E[\ga(p, q_w, \epsilon)]$ and $\E[\ga(p, q_w, \epsilon)^2]$. For the former, we can compute

\small
\begin{align}
\E[\ga(p, q_w, \epsilon)] & = \left(\frac{\sw}{\sigma_p}\right)^\alpha  \sw \left(\frac{1}{\sigma_p^2} - \frac{1}{\sw^2} \right)
\int \frac{\epsilon^2}{\sqrt{2\pi}} \exp\left[\frac{-\epsilon^2}{2}\right] \exp\left[-\frac{\alpha}{2} \sw^2 \epsilon^2 \left(\frac{1}{\sigma_p^2} - \frac{1}{\sw^2} \right) \right] d\epsilon\\
& = \left(\frac{\sw}{\sigma_p}\right)^\alpha  \sw \left(\frac{1}{\sigma_p^2} - \frac{1}{\sw^2} \right) 
\int \frac{\epsilon^2}{\sqrt{2\pi}} \exp\left[-\frac{\epsilon^2}{2} \left(1 + \frac{\alpha \sw^2}{\sigma_p^2} - \alpha \right) \right] d\epsilon\\
& = \R^{\alpha/2} \sw \left(\frac{1}{\sigma_p^2} - \frac{1}{\sw^2} \right) 
\int \frac{\epsilon^2}{\sqrt{2\pi}} \exp\left[-\frac{\epsilon^2}{2} \left(1 + \alpha \R - \alpha \right) \right] d\epsilon.
\end{align}
\normalsize

This integral converges if and only if $1 + \alpha R - \alpha > 0$. If so,

\begin{align}
\E[\ga(p, q_w, \epsilon)] & = \R^{\alpha/2} \sigma_q \left(\frac{1}{\sigma_p^2} - \frac{1}{\sw^2} \right) \frac{1}{(1 + \alpha \R - \alpha)^{3/2}}.
\end{align}

Thus, we get

\begin{equation}
\E[\ga(p, q_w, \epsilon)]^2 = \R^{\alpha} \sigma_q^2 \left(\frac{1}{\sigma_p^2} - \frac{1}{\sw^2} \right)^2 \frac{1}{(1 + \alpha \R - \alpha)^{3}}.
\end{equation}

We now need to compute $\E[\ga(p, q_w, \epsilon)^2]$. Using eq. \ref{eqapp:here4} we get

\begin{equation}
\ga(p, q_w, \epsilon)^2 = \left(\frac{\sw}{\sigma_p}\right)^{2\alpha} \exp\left[-\alpha \sw^2 \epsilon^2 \left(\frac{1}{\sigma_p^2} - \frac{1}{\sw^2} \right) \right] \epsilon^4 \sw^2 \left(\frac{1}{\sigma_p^2} - \frac{1}{\sw^2} \right)^2.
\end{equation}

The expectation of the term above can be computed as

\small
\begin{align}
\E[\ga(p, q_w, \epsilon)^2] & = \left(\frac{\sw}{\sigma_p}\right)^{2\alpha} \sw^2 \left(\frac{1}{\sigma_p^2} - \frac{1}{\sw^2} \right)^2
\int \frac{\epsilon^4}{\sqrt{2\pi}} \exp\left[\frac{-\epsilon^2}{2}\right] \exp\left[-\alpha \sw^2 \epsilon^2 \left(\frac{1}{\sigma_p^2} - \frac{1}{\sw^2} \right) \right] d\epsilon\\
& = \left(\frac{\sw}{\sigma_p}\right)^{2\alpha} \sw^2 \left(\frac{1}{\sigma_p^2} - \frac{1}{\sw^2} \right)^2
\int \frac{\epsilon^4}{\sqrt{2\pi}} \exp\left[-\frac{\epsilon^2}{2} \left(1 + 2\frac{\alpha \sw^2}{\sigma_p^2} - 2\alpha \right) \right] d\epsilon\\
& = \R^\alpha \sw^2 \left(\frac{1}{\sigma_p^2} - \frac{1}{\sw^2} \right)^2
\int \frac{\epsilon^4}{\sqrt{2\pi}} \exp\left[-\frac{\epsilon^2}{2} \left(1 + 2\alpha \R - 2\alpha \right) \right] d\epsilon.
\end{align}
\normalsize

This integral converges if and only if $1 + 2\alpha \R - 2\alpha > 0$. If so,

\begin{align}
\E[\ga(p, q_w, \epsilon)^2] & = \R^\alpha \sigma_q^2 \left(\frac{1}{\sigma_p^2} - \frac{1}{\sw^2} \right)^2 \frac{3}{(1 + 2\alpha \R - 2\alpha)^{5/2}} \label{eqapp:here10}.
\end{align}

Finally,

\begin{align}
\SNR[\ga(p, q_w, \epsilon)] & = \frac{1 + 2\alpha \R - 2\alpha}{3} \left(\frac{\sqrt{1 + 2\alpha \R - 2\alpha}}{1 - \alpha \R - \alpha}\right)^3\\
& = \frac{1 + 2\alpha \R - 2\alpha}{3} f(\R, \alpha)^3,
\end{align}

where we followed the same algebraic manipulations as in the proof of lemma \ref{app:lemmaf} to get the final expression in terms of $f(\R, \alpha)$.
\end{proof}

\subsubsection{Factorized Gaussians - Main Result}

\fullyfactgauss*



\begin{proof}
This result can be easily obtained by using the expression from Theorem \ref{thm:NSR-gff}, and the results from lemmas \ref{app:lemmaffg1}, \ref{app:lemmaf}, and \ref{app:lemmaffg3}. The limit for $\alpha \rightarrow 0$ is $1/3$, which is correct.
\end{proof}

\newpage

\subsection{Full-rank Gaussians} \label{app:allfull}

In this section we prove our main results for the case where both $p$ and $q_w$ are $d$-dimensional mean-zero Gaussians with arbitrary covariance matrices $\Sigma_p$ and $\Cw$, respectively. In this case the set of parameters $w$ is given by the some factorization of the covariance $\Sw$ such that $\Sw \Sw^\top = \Cw$. In this case the gradients $\nabla_w$ are computed with respect $\Sw$, and thus have shape $d \times d$. Thus, the SNR is defined in terms of the Frobenius norm $\Vert \cdot \Vert_F$, that is

\begin{equation}
\SNR[\ga(p, q_w, \epsilon)] = \frac{\Vert \E \ga(p, q_w, \epsilon) \Vert_F^2}{\E \Vert \ga(p, q_w, \epsilon) \Vert_F^2}.
\end{equation} 

We begin by showing a proof for the case $\alpha \rightarrow 0$.

\begin{lemma} \label{lem:frg-a0}
Let $p(z) = \mathcal{N}(z|0, \Sigma_p)$ and $q(z) = \mathcal{N}(z|0, \Cw)$. Also, let $\Sw$ be a matrix such that $\Sw \Sw^\top = \Cw$, which we use to reparameterize as $\T = \Sw \epsilon$. Then, for $\ga = \gdrep$ and $\alpha \rightarrow 0$,

\begin{itemize}[leftmargin=*] \denselist
\item If $\Sigma_p = \Cw$, the gradient estimator is deterministically zero.
\item Otherwise,
\begin{equation}
\SNR[\ga(p, q_w, \epsilon)] = \frac{1}{d + 2},
\end{equation}
where $d$ is the dimension of $z$.
\end{itemize}
\end{lemma}

\begin{proof}
We begin by computing the gradient estimator $\ga(p, q_w, \epsilon)$. Using the pdf of a mean-zero multivariate normal and the reparameterization transformation $\T = \Sw \epsilon$ we get

\begin{align}
\ga(p, q_w, \epsilon) & = \left . \nabla_w \log \frac{p(\T)}{q_v(\T)}\right|_{v=w}\\
& = \left . \nabla_w \left( \log \sqrt{\frac{\det \Sigma_p}{\det \Cv}} - \frac{1}{2} \epsilon^\top \Sw^\top (\Cv^{-1} - \Sigma_p^{-1}) \Sw \epsilon \right)\right|_{v=w}\\
& = \left. -\frac{1}{2} \nabla_w \epsilon^\top \Sw^\top (\Cv^{-1} - \Sigma_p^{-1}) \Sw \epsilon\right|_{v=w}\\
& = -(\Cw^{-1} - \Sigma_p^{-1}) \Sw \epsilon \epsilon^\top\\
& = -B \epsilon \epsilon^\top,
\end{align}

where in the last line we defined $B = (\Cw^{-1} - \Sigma_p^{-1}) S$. It can be observed that the estimator is deterministically zero if $\Sigma_p = \Cw$. Otherwise, to compute the SNR we have to compute $\Vert \E \ga(p, q_w, \epsilon) \Vert^2$ and $\E \Vert \ga(p, q_w, \epsilon) \Vert^2$. Using the fact that $\epsilon$ follows a $d$-dimensional standard Gaussian, the former can be computed as

\begin{equation}
\E[\ga(p, q_w, \epsilon)] = -B \,\,\,\,\mbox{ and }\,\,\,\, \Vert \E[\ga(p, q_w, \epsilon)] \Vert_F^2 = \Vert B \Vert_F^2.
\end{equation}

Computing $\E \Vert \ga(p, q_w, \epsilon) \Vert^2_F$ requires more work. We will first consider the expected squared norm of a single component $ij$ of $\ga(p, q_w, \epsilon)$, that is $\E \ga(p, q_w, \epsilon)_{ij}^2$, and then add all of the results. We have

\begin{align}
\ga(p, q_w, \epsilon)_{ij} & = (-B \epsilon \epsilon^\top)_{ij}\\
& = -\sum_{k=1}^d B_{ik} \epsilon_k \epsilon_j,\\
\ga(p, q_w, \epsilon)^2_{ij} & = \sum_{l = 1}^d \sum_{k=1}^d B_{ik} B_{il} \epsilon_k \epsilon_l \epsilon_j^2.
\end{align}

Using the fact that $\epsilon = [\epsilon_1, ..., \epsilon_d]$ follows a multivariate standard Gaussian, we can compute $\E \ga(p, q_w, \epsilon)_{ij}^2$ as

\begin{align}
\E \ga(p, q_w, \epsilon)_{ij}^2 & = \sum_{k, l} B_{ik} B_{il} \E [\epsilon_k \epsilon_l \epsilon_j^2]\\
& = \sum_{k,l} B_{ik} B_{il} \left( \mathbbm{1}_{k = l} + 2\mathbbm{1}_{j = k} \mathbbm{1}_{j = l} \right)\\
& = \sum_k B_{ik}^2 + 2B_{ij}^2.
\end{align}

Finally, adding the results for all components $ij$ yields

\begin{align}
\E \Vert \ga(p, q_w, \epsilon)\Vert_F^2 & = \sum_{ij} \E [\ga(p, q_w, \epsilon)_{ij}^2]\\
& = 2 \sum_{ij} B_{ij}^2 + \sum_{ijk} B_{ik}^2\\
& = 2 \sum_{ij} B_{ij}^2 + d \sum_{ik} B_{ik}^2\\
& = (2 + d) \sum_{ij} B_{ij}^2\\
& = (2 + d) \Vert B\Vert_F^2.
\end{align}

Taking the ratio between the expressions for $\E \Vert \ga(p, q_w, \epsilon)\Vert_F^2$ and $\Vert \E[\ga(p, q_w, \epsilon)] \Vert_F^2$ yields the desired result.
\end{proof}

\fullrankexact*




\begin{proof}
We split the proof into several steps:

\begin{enumerate}

\item We begin by finding an expression for the gradient estimator. Using the fact that both $p$ and $q$ are mean-zero multivariate Gaussians we get

\small
\begin{align}
\ga(p, q_w, \epsilon) & = \left. -\frac{1}{\alpha} \nabla_w \left(\frac{p(\T)}{q_v(\T)}\right)^\alpha\right|_{v=w}\\
& =\left. -\frac{1}{\alpha} \nabla_w \left(\frac{\det\Cv}{\det \Sigma_p}\right)^{\alpha / 2} \exp\left[ -\frac{\alpha}{2} \epsilon^\top \Sw^\top (\Sigma_p^{-1} - \Cw^{-1}) S \epsilon \right]\right|_{v=w}\\
& = \left(\frac{\det\Cw}{\det \Sigma_p}\right)^{\alpha/2} \exp\left[ -\frac{\alpha}{2} \epsilon^\top \Sw^\top (\Sigma_p^{-1} - \Cw^{-1}) \Sw \epsilon \right] (\Sigma_p^{-1} - \Cw^{-1}) \Sw \epsilon \epsilon^\top.
\end{align}\normalsize

It can be observed that if $\Sigma_p = \Cw$ the gradient estimator is deterministically zero.

\item We compute $\E \Vert \ga(p, q_w, \epsilon) \Vert_F^2$. Using the definitions for $B$ and $U$ we can compute $\E[\ga(p, q_w, \epsilon)]$ as 

\small
\begin{align}
\E[\ga(p, q_w, \epsilon)] & = \left(\frac{\det\Cw}{\det \Sigma_p}\right)^{\alpha/2} \int \frac{1}{(2\pi)^{d/2}}\exp\left[ -\frac{1}{2} \epsilon^\top \epsilon \right]\exp\left[ -\frac{1}{2} \epsilon^\top (U - I) \epsilon \right] B \epsilon \epsilon^\top d\epsilon\\
& = \left(\frac{\det\Cw}{\det \Sigma_p}\right)^{\alpha/2} B \int \frac{1}{(2\pi)^{d/2}} \exp\left[ -\frac{1}{2} \epsilon^\top U \epsilon \right] \epsilon \epsilon^\top d\epsilon.
\end{align}
\normalsize

This integral converges if and only if $U$ is positive definite. If this is satisfied, then 

\small
\begin{align}
\E[\ga(p, q_w, \epsilon)] & = \left(\frac{\det\Cw}{\det \Sigma_p}\right)^{\alpha/2} \sqrt{\det(U^{-1})} B \int \frac{1}{(2\pi)^{d/2} \sqrt{\det(U^{-1})}} \exp\left[ -\frac{1}{2} \epsilon^\top U \epsilon \right] \epsilon \epsilon^\top d\epsilon\\
& = \left(\frac{\det\Cw}{\det \Sigma_p}\right)^{\alpha/2} \frac{1}{\sqrt{\det U}} B U^{-1},
\end{align}
\normalsize

and 

\begin{equation}
\Vert \E[\ga(p, q_w, \epsilon)]\Vert^2_F = \left(\frac{\det\Cw}{\det \Sigma_p}\right)^{\alpha} \frac{1}{\det U} \Vert B U^{-1}\Vert^2_F.
\end{equation}

\item We compute $\E \Vert\ga(p, q_w, \epsilon)\Vert^2_F$. Computing $\E \Vert\ga(p, q_w, \epsilon)\Vert^2_F$ requires more work. We first compute $\E[\ga(p, q_w, \epsilon)_{ij}^2]$ for each component $ij$ of $\ga$ and then add up all the results. Recall, the gradient estimator is given by

\begin{align}
\ga(p, q_w, \epsilon) & = \left(\frac{\det\Cw}{\det \Sigma_p}\right)^{\alpha/2} \exp\left[ -\frac{\alpha}{2} \epsilon^\top S^\top (\Sigma_p^{-1} - \Cw^{-1}) S \epsilon \right] B \epsilon \epsilon^\top.
\end{align}

Thus, component $ij$ of the estimator, that is, $\ga(p, q_w, \epsilon)_{ij}$, is given by

\begin{align}
\ga(p, q_w, \epsilon)_{ij} & = \left(\frac{\det\Cw}{\det \Sigma_p}\right)^{\alpha/2} \exp\left[ -\frac{\alpha}{2} \epsilon^\top \Sw^\top (\Sigma_p^{-1} - \Cw^{-1}) \Sw \epsilon \right] \sum_{k= 1}^d B_{ik} \epsilon_k \epsilon_j,
\end{align}

and $\ga(p, q_w, \epsilon)_{ij}^2$ is given by

\begin{align}
\ga(p, q_w, \epsilon)_{ij}^2 & = \left(\frac{\det\Cw}{\det \Sigma_p}\right)^{\alpha} \exp\left[ -\alpha \epsilon^\top \Sw^\top (\Sigma_p^{-1} - \Cw^{-1}) \Sw \epsilon \right] \sum_{k,l = 1}^d B_{il} B_{ik} \epsilon_k \epsilon_l \epsilon_j^2.
\end{align}

Using the definition for $V$, and using the fact that $\epsilon$ follows a multivariate standard Gaussian distribution, we can compute $\E [\ga(p, q_w, \epsilon)_{ij}^2]$ as 

\small
\begin{align}
& \E[\ga(p, q_w, \epsilon)_{ij}^2] \\
& = \left(\frac{\det\Cw}{\det \Sigma_p}\right)^{\alpha} \int \frac{1}{(2\pi)^{d/2}} \exp\left[ -\frac{1}{2} \epsilon^\top \epsilon \right]\exp \left[ -\alpha \epsilon^\top \Sw^\top (\Sigma_p^{-1} - \Cw^{-1}) \Sw \epsilon \right] \sum_{k,l = 1}^d B_{il} B_{ik} \epsilon_k \epsilon_l \epsilon_j^2\\
& = \left(\frac{\det\Cw}{\det \Sigma_p}\right)^{\alpha} \int \frac{1}{(2\pi)^{d/2}} \exp \left[ -\frac{1}{2}\epsilon^\top V \epsilon \right] \sum_{k,l = 1}^d B_{il} B_{ik} \epsilon_k \epsilon_l \epsilon_j^2\\
& = \left(\frac{\det\Cw}{\det \Sigma_p}\right)^{\alpha} \sum_{k,l = 1}^d B_{il} B_{ik} \int \frac{1}{(2\pi)^{d/2}} \exp \left[ -\frac{1}{2}\epsilon^\top V \epsilon \right] \epsilon_k \epsilon_l \epsilon_j^2.
\end{align}
\normalsize

The integrals converge if and only if $V$ is positive definite. If this is satisfied, then

\begin{align}
& \E[\ga(p, q_w, \epsilon)_{ij}^2]\\
& = \left(\frac{\det\Cw}{\det \Sigma_p}\right)^{\alpha} \frac{1}{\sqrt{\det V}} \sum_{k,l = 1}^d B_{il} B_{ik} \int \frac{1}{(2\pi)^{d/2} \sqrt{\det (V^{-1})}} \exp \left[ -\frac{1}{2}\epsilon^\top V \epsilon \right] \epsilon_k \epsilon_l \epsilon_j^2\\
& = \left(\frac{\det\Cw}{\det \Sigma_p}\right)^{\alpha} \frac{1}{\sqrt{\det V}} \sum_{k,l = 1}^d B_{il} B_{ik} \left( V^{-1}_{jj} V^{-1}_{kl} + 2 V^{-1}_{kj} V^{-1}_{lj} \right),
\end{align}

where $V^{-1}_{ij}$ is component $ij$ of matrix $V^{-1}$. Finally, we compute $\E \Vert \ga(p, q_w, \epsilon) \Vert_F^2 = \sum_{ij} \E[\ga(p, q_w, \epsilon)_{ij}^2]$, which gives

\begin{equation}
\E \Vert \ga(p, q_w, \epsilon) \Vert_F^2 = \left(\frac{\det\Cw}{\det \Sigma_p}\right)^{\alpha} \frac{1}{\sqrt{\det V}} \sum_{i,j,k,l = 1}^d B_{il} B_{ik} \left( V^{-1}_{jj} V^{-1}_{kl} + 2 V^{-1}_{kj} V^{-1}_{lj} \right) \label{eqapp:here6}. 
\end{equation}

\item We work on eq. \ref{eqapp:here6}. For this we use the notation $B_{i:}$, which denotes the $i$-th row of $B$. First,  

\begin{align}
\sum_{ij} \sum_{kl} B_{il} B_{ik} V^{-1}_{jj} V^{-1}_{kl} & = \sum_{ij} V^{-1}_{jj} \sum_{kl} B_{ik} V^{-1}_{kl} B_{il}\\
& = \sum_{ij} V^{-1}_{jj} (B_{i:}^\top V^{-1} B_{i:})\\
& = \sum_{ij} V^{-1}_{jj} (B^\top V^{-1} B)_{ii}\\
& = \sum_{j} V^{-1}_{jj} \sum_i (B^\top V^{-1} B)_{ii}\\
& = \mathrm{tr}(V^{-1}) \mathrm{tr}(B^\top V^{-1} B).
\end{align}

Second,

\begin{align}
\sum_{ijkl} B_{il} B_{ik} V^{-1}_{kj} V^{-1}_{lj} & = \sum_{ij} \sum_{kl} B_{il} B_{ik} V^{-1}_{kj} V^{-1}_{lj}\\
& = \sum_{ij} \left(\sum_{k} B_{ik} V^{-1}_{kj}\right) \left(\sum_l B_{il} V^{-1}_{lj}\right)\\
& = \sum_{ij} \left(\sum_{k} B_{ik} V^{-1}_{kj}\right)^2\\
& = \sum_{ij} \left(B_{i:} V^{-1}_{:j}\right)^2\\
& = \sum_{ij} \left(B V^{-1}\right)_{ij}^2\\
& = \Vert B V^{-1} \Vert_F^2.
\end{align}

Using these two results we get

\begin{equation}
\E \Vert \ga(p, q_w, \epsilon) \Vert_F^2 = \left(\frac{\det\Cw}{\det \Sigma_p}\right)^{\alpha} \frac{1}{\sqrt{\det V}} \left(\mathrm{tr}(V^{-1}) \mathrm{tr}(B^\top V^{-1}B) + 2 \Vert B V^{-1} \Vert_F^2\right).
\end{equation}

\item Taking the ratio between $\Vert \E \ga(p, q_w, \epsilon) \Vert_F^2$ and $\E \Vert \ga(p, q_w, \epsilon) \Vert_F^2$ yields

\begin{equation}
\SNR[\ga(p, q_w, \epsilon)] = \frac{\sqrt{\det V}}{\det U} \frac{\Vert B U^{-1} \Vert^2_F}{\mathrm{tr}(V^{-1}) \mathrm{tr}(B V^{-1}B^\top) + 2 \Vert B V^{-1} \Vert_F^2}. \label{eqpp:here7}.
\end{equation}

\item We prove the eigenvalues $\lambda_1, ..., \lambda_d$ of $\Sigma_p^{-1} \Cw$ are all real and positive. This can be seen by noticing that $\Sigma_p^{-1} \Cw$ is similar to the matrix $S^\top \Sigma_p^{-1} S$.\footnote{$\Sw^{-\top}(\Sw^\top \Sigma_p^{-1} \Sw)\Sw^\top = \Sigma_p^{-1} \Sw \Sw^\top = \Sigma_p^{-1} \Cw$.} These latter matrix is symmetric and positive definite, and thus its eigenvalues are real and positive. Since similar matrices have the same eigenvalues, it can be concluded that $\lambda_1, ..., \lambda_d$ are real and positive.

\item The first condition for existence of the SNR that requires that $U = (1 - \alpha) I + \alpha \Sw^\top \Sigma_p^{-1} \Sw$ is positive definite can be rewritten as $1 - \alpha + \alpha \lambda_i > 0$ for $i = 1,...,d$. Since $\Sw^\top \Sigma_p^{-1} \Sw$ is similar to $\Sigma_p^{-1} \Cw$, $\Sw^\top \Sigma_p^{-1} \Sw$ eigenvalues are $\lambda_1,...,\lambda_d$. Therefore, $U$'s $i$-th eigenvalue is $\lambda_i(U) = 1 - \alpha + \alpha \lambda_i$. Since $U$ must be positive definite (i.e. all its eigenvalues are positive) we get the conditions $1 - \alpha + \alpha \lambda_i > 0$ for all $i$.

\item The second condition for existence of the SNR that requires that $V = (1 - 2\alpha) I + 2\alpha \Sw^\top \Sigma_p^{-1} \Sw$ is positive definite can be rewritten as $1 - 2\alpha + 2\alpha \lambda_i > 0$ for $i = 1,...,d$. Since $\Sw^\top \Sigma_p^{-1} \Sw$ is similar to $\Sigma_p^{-1} \Cw$, $\Sw^\top \Sigma_p^{-1} \Sw$ eigenvalues are $\lambda_1,...,\lambda_d$. Therefore, $V$'s $i$-th eigenvalue is $\lambda_i(V) = 1 - 2\alpha + 2\alpha \lambda_i$. Since $V$ must be positive definite (i.e. all its eigenvalues are positive) we get the conditions $1 - 2\alpha + 2\alpha \lambda_i > 0$ for all $i$. Notice that this condition is strictly more restrictive than $1 - \alpha + \alpha \lambda_i > 0$, and thus is the only one mentioned in the theorem.

\item We get an exact expression for the ratio $\frac{\sqrt{\det V}}{\det U}$ (see eq. \ref{eqpp:here7}). Using the fact that the determinant of a matrix is the product of its eigenvalues, we can get an exact expression for this ratio in terms of $\lambda_1, ..., \lambda_d$ as 

\begin{align}
\frac{\sqrt{\det V}}{\det U} & = \prod_{i=1}^d \frac{\sqrt{\lambda_{i}(V)}}{\lambda_{i}(U)}\\
& = \prod_{i=1}^d \frac{\sqrt{2\alpha \lambda_i + 1 - 2\alpha}}{\alpha \lambda_i + 1 - \alpha}\\
& = \prod_{i=1}^d \sqrt{\frac{1}{ \frac{(1 + \alpha \lambda_i - \alpha)^2}{1 + 2\alpha \lambda_i - 2\alpha} }}\\
& = \prod_{i=1}^d \sqrt{\frac{1}{ 1 + \alpha^2 \frac{(\lambda_i - 1)^2}{1 + 2\alpha \lambda_i - 2\alpha}}}\\
& = \prod_{i=1}^d f(\lambda_i, \alpha).
\end{align}

Plugging this result into eq. \ref{eqpp:here7} yields the desired result.
\end{enumerate}
\end{proof}


\fullrankbound*

\begin{proof} 

The case for $\alpha \rightarrow 0$ is given by Lemma \ref{lem:frg-a0}. 

For $\alpha \neq 0$, given eq.~\ref{eqpp:here7}, all we need to prove the result is an upper bound for $\frac{\Vert B U^{-1} \Vert^2_F}{\mathrm{tr}(V^{-1}) \mathrm{tr}(B V^{-1}B^\top) + 2 \Vert B V^{-1} \Vert_F^2}$. This requires finding an upper bound for the numerator, and a lower bound for the denominator. To do this we introduce the induced matrix norm $\Vert A \Vert_2 = \max_{x \neq 0} \frac{\Vert A x \Vert_2}{\Vert x \Vert_2}$, and use the following inequalities:
\begin{description}
  \item[M1:] For any $d \times d$ real matrix $A$ we have $\Vert A \Vert_2^2 \leq \Vert A \Vert_F^2 \leq d \Vert A \Vert_2^2$.
  \item[M2:] For any $d \times d$ real matrices $A$ and $B$ we have $\Vert A B \Vert_2 \leq \Vert A \Vert_2 \Vert B \Vert_2$.
  \item[M3:] For any symmetric $d \times d$ real matrix $A$ we have $\Vert A \Vert_2 = \max_i |\lambda_i(A)|$, where $\lambda_i(A)$ is the $i$-th eigenvalue of $A$.
  \item[M4:] For any $d \times d$ matrix $B$ and symmetric $d \times d$ real matrix $A$ we have $\Vert B A \Vert_2 \geq \Vert B \Vert_2 \min_i |\lambda_i(A)|$, where $\lambda_i(A)$ is the $i$-th eigenvalue of $A$.
\end{description}

An upper bound for $\Vert B U^{-1} \Vert^2_F$ is given by 

\begin{align}
\Vert B U^{-1} \Vert^2_F & \leq d \Vert B U^{-1} \Vert_2^2 & \mbox{(M1)}\\
& \leq d \Vert B \Vert_2^2 \Vert U^{-1} \Vert_2^2 & \mbox{(M2)}\\
& = d \Vert B \Vert_2^2 \left(\max_i \lambda_i(U^{-1}) \right)^2 & \mbox{(M3)}\\
& = d \Vert B \Vert_2^2 \, \lambda_{\min}(U)^2,
\end{align}

where in the last line we used the fact that all eigenvalues of $U$ are positive and that $\lambda_i(U) = \frac{1}{\lambda_i(U^{-1})}$.

A lower bound for $\mathrm{tr}(V^{-1})$ is given by

\begin{align}
\mathrm{tr}(V^{-1}) & = \sum_{i=1}^d \lambda_i(V^{-1})\\
& \geq d \lambda_{\min}(V^{-1}) &  \\
& = d \lambda_{\max}(V).
\end{align}

Using the eigen-decomposition $V = P D_v P^\top$, where $P$ is orthogonal and $D$ diagonal with positive entries, we can get a lower bound for $\mathrm{tr}(B^\top V^{-1} B)$ as

\begin{align}
\mathrm{tr}(B^\top V^{-1} B) & = \mathrm{tr}(B^\top P D_v^{-1} P^\top B)\\
& = \mathrm{tr}((B^\top P D_v^{-1/2})(B^\top P D_v^{-1/2})^\top)\\
& = \Vert B^\top P D_v^{-1/2} \Vert_F^2\\
& \geq \Vert B^\top P D_v^{-1/2} \Vert^2 & \mbox{(M1)}\\
& \geq \Vert B^\top P \Vert^2 \left(\min_i \lambda_i(D_v^{-1/2})\right)^2  & \mbox{(M4)}\\
& = \Vert B^\top \Vert^2 \min_i \lambda_i(V^{-1})\\
& = \Vert B \Vert^2 \lambda_{\max}(V).
\end{align}

Combining the last two results we can get a lower bound for the denominator $\mathrm{tr}(V^{-1}) \mathrm{tr}(B V^{-1}B^\top) + 2 \Vert B V^{-1} \Vert_F^2$ as 

\begin{align}
\mathrm{tr}(V^{-1}) \mathrm{tr}(B V^{-1}B^\top) + 2 \Vert B V^{-1} \Vert_F^2 & \geq \mathrm{tr}(V^{-1}) \mathrm{tr}(B V^{-1}B^\top)\\
& \geq d \Vert B \Vert_2^2 \, \lambda_{\max}(V)^2.
\end{align}

Finally, the upper bound for $\frac{\Vert B U^{-1} \Vert^2_F}{\mathrm{tr}(V^{-1}) \mathrm{tr}(B V^{-1}B^\top) + 2 \Vert B V^{-1} \Vert_F^2}$ is given by

\begin{align}
\frac{\Vert B U^{-1} \Vert^2_F}{\mathrm{tr}(V^{-1}) \mathrm{tr}(B V^{-1}B^\top) + 2 \Vert B V^{-1} \Vert_F^2} & \leq \left( \frac{\lambda_{\min}(U)}{\lambda_{\max}(V)} \right)^2\\
& = \left(\frac{\min_i (1 + \alpha \lambda_i - \alpha)}{\max_i (1 + 2\alpha \lambda_i - 2\alpha)}\right)^2.
\end{align}

Considering the cases $\alpha > 0$ and $\alpha < 0$ yields the desired result.
\end{proof}

\newpage
\section{Low Variance Proof (Section \ref{sec:motiv})} \label{app:lowvar}

Section \ref{sec:motiv} considers the setting where both $p$ and $q_w$ are fully-factorized Gaussians with mean zero. That section empirically showed that for $d = 128$ and $\alpha = 0.4$ optimization failed even though the gradient estimator had ``low'' variance. In this section we state this variance result formally. Theorem \ref{thm:small_var} presents an upper bound for the variance of the gradient estimator, and Preposition \ref{prep:small_var} states that, for $\alpha \in (0, 1/2)$, this upper bound ``gets small'' for high dimensional problems.

\begin{thm} \label{thm:small_var}
Let $p$ and $q$ be two fully-factorized $d$-dimensional Gaussian distributions with mean zero and variances {$\sigma_{pi}^2$} and {$\sigma_{qi}^2$} for $i \in \{ 1, \hdots, d \}$. Let {$\R_i = \nicefrac{\sigma_{qi}^2}{\sigma_{pi}^2}$} and $\ga = \gdrep$. Then,

\begin{itemize}[leftmargin=*] \denselist
\item If $\sigma_{pj} = \sigma_{qj}$, the $j$-th component of $\ga$ is deterministically zero.
\item If $\alpha \neq 0$ and $1 + 2\alpha \R_i - 2\alpha \leq 0$ for any $i$, the gradient estimator has infinite variance.
\item Otherwise,
\small
\begin{equation}
\Var[\gaj] \leq A(\sigma_{pj}, \sigma_{qj}, \alpha) \prod_{i=i}^d \sqrt{\frac{\R_i^{2\alpha}}{1 + 2\alpha \R_i - 2\alpha}}, \label{eq:var_bound}
\end{equation}
\normalsize
where $\gaj$ is the $j$-th component of the estimator $\ga$, and \small$A(\sigma_{pj}, \sigma_{qj}, \alpha) = \frac{3 \sigma_{qj}^2}{(1 + 2\alpha \R_j - 2\alpha)^2} (\frac{1}{\sigma_{pj}^{2}} - \frac{1}{\sigma_{qj}^{2}})^2$\normalsize.
\end{itemize}
\end{thm}

\begin{proposition} \label{prep:small_var}
If $\alpha \in (0, 1/2)$, then: (i) the gradient estimator $\gdrep$ has finite variance; (ii) all of the terms in the $d$-term product in eq. \ref{eq:var_bound} are less than or equal one, with equality holding if and only if $p_i = q_{w_i}$ for the corresponding $i$.
\end{proposition}

Theorem \ref{thm:small_var} gives an upper bound for the variance of the $j$-th component of the gradient estimator $\ga$. It basically states that this bound is some function $A$ that depends on $\alpha, \sigma_{pj}$ and $\sigma_{qj}$, times $d$ terms, one for each $i = 1, ..., d$. In addition, Preposition \ref{prep:small_var} states that, if $\alpha \in (0, 1/2)$, each of these $d$ terms is at most 1, with equality holding if and only if $p_i = q_{w_i}$ for the corresponding $i$. Therefore, if this does not hold for any $i$ (e.g. at initialization), the variance bound is just the factor $A(\sigma_{pj}, \sigma_{qj}, \alpha)$ times $d$ terms strictly less than one. Intuitively, if $\alpha \in (0, 1/2)$, larger dimensionalities $d$ lead to smaller upper bounds for the variance (if $p$ and $q_w$ are isotropic this bound goes to zero exponentially in $d$).

We now present a proof for Theorem \ref{thm:small_var}.

\begin{proof}
We know that the variance is just $\Var[\gaj] = \E[\gaj^2] - \E[\gaj]^2 \leq \E[\gaj^2]$. Thus, an upper bound for the variance is given by $\E[\gaj^2]$. We can compute this for the case where $p$ and $q_w$ are fully-factorized mean zero Gaussians. From eq. \ref{eq:lv2} we know 

\small
\begin{equation}
\E[\gaj^2] = \E \left[ \left( \frac{-1}{\alpha} \nabla_{w_j} \left(\frac{p_j(\Tj)}{q_{v_j}(\Tj)}\right)^\alpha\right)^2\right] \prod_{i \in \{1, ..., d\} \backslash j} \E \left[\left(\frac{p_i(\Ti)}{q_{v_i}(\Ti)}\right)^{2\alpha} \right]_{v=w}.
\end{equation}
\normalsize

Replacing the results from eqs. \ref{eqapp:here3} and \ref{eqapp:here10} into the above expression yields the result from the lemma.
\end{proof}

We now present a proof for Preposition \ref{prep:small_var}.

\begin{proof}
We begin by proving item (i) of the preposition. The condition for finite variance is $1 + 2\alpha(\R_i - 1) > 0$ for all $i$. For $\alpha > 0$ this is equivalent to $\sigma_{qi}^2 > (1 - \frac{1}{2\alpha}) \sigma_{pi}^2$. If $\alpha \in (0, 1/2)$ the right hand side of the inequality is negative. Thus, since $\sigma_q^2 > 0$, the condition is always satisfied.

We now prove item (ii). $\R^{2\alpha}$ is strictly concave for $\alpha \in (0, 1/2)$ and $R \geq 0$. Thus, a first order Taylor expansion around any $\R_0 > 0$ yields an upper bound of the function. For $\R_0 = 1$ the Taylor expansion is given by $1 + 2\alpha(\R - 1) = 1 + 2\alpha \R - 2\alpha$. Since the Taylor expansion is an upper bound, we must have $\frac{\R^{2\alpha}}{1 + 2\alpha \R - 2\alpha} \leq 1$, with equality holding if and only if $\R = 1$.
\end{proof}

\end{document}